%% file: main.tex
\pdfoutput=1
\documentclass[11pt]{article}
\usepackage[OT1]{fontenc}
\usepackage[utf8]{inputenc}
\usepackage{algorithm}
\usepackage{amsfonts}
\usepackage{amsmath}
\usepackage{amssymb}
\usepackage{amsthm}
\usepackage{bm}
\usepackage{bbm}
\usepackage{booktabs}
\usepackage{caption}
\usepackage{extarrows}
\usepackage{float} 
\usepackage{graphicx}
\usepackage{inconsolata}
\usepackage{listings}
\usepackage{makecell}
\usepackage{mathrsfs}
\usepackage{multicol}
\usepackage{multirow}
\usepackage{newclude}
\usepackage{subfigure}
\usepackage{setspace}
\usepackage[nolabel]{showlabels}
\usepackage{threeparttable}
\usepackage{smile}
\usepackage{subcaption}
\usepackage{tikz}
\usetikzlibrary{positioning,arrows.meta}

\tikzset{
    dagvar/.style={
        circle,
        draw,
        minimum size=8mm,
        inner sep=0pt
    },
    dagarrow/.style={
        -{Latex[length=2mm]},
        semithick
    },
    daglabel/.style={
        font=\itshape
    }
}
\tikzset{
    dagsel/.style={
        rectangle,
        draw,
        semithick,
        minimum size=7mm,
        inner sep=1pt
    }
}
\tikzset{
    daglatent/.style={
        circle,
        draw,
        densely dotted,
        semithick,
        minimum size=9mm,
        inner sep=0pt
    }
}

\graphicspath{{figs/}}

\setitemize[1]{itemsep=0pt,partopsep=0pt,parsep=\parskip,topsep=0pt}
\DeclareMathAlphabet{\mathpzc}{OT1}{pzc}{m}{it}
\newcommand{\E}{\mathbb{E}}
\renewcommand{\P}{\mathbb{P}}
\renewcommand{\d}{\mathrm{d}}

\renewcommand{\proofname}{\textit{Proof}}

\title{\huge\bf Transportable Offline Reinforcement Learning in Presence of Unmeasured Confounding}
\author{
Junyi Liao$^1$,\ \ 
Hangqi Ren$^2$,\ \ 
Zhuoran Yang$^3$\ \ 
~\\
$^{1}$\small\it Department of Electrical and Computer Engineering, Duke University \\
$^{2}$\small\it Department of Electrical and Computer Engineering, Vanderbilt University \\
$^{3}$\small\it Department of Statistics and Data Science, Yale University
\\
\small\sf junyi.liao@duke.edu, hangqi.ren@vanderbilt.edu, zhuoran.yang@yale.edu
}
\date{}
\begin{document}
\maketitle
\begin{abstract}
Offline reinforcement learning seeks to learn decision policies from previously collected data, but deployment environments may differ from the domains in which the data were collected, and logged variables may omit confounders that affected both historical actions and outcomes. We study offline policy learning under the joint challenges of domain discrepancy and unmeasured confounding, where multiple heterogeneous source datasets provide complementary information about a target domain. We develop a causally adjusted transportability framework combining transportability theory with instrumental variables, negative controls, and proxies to identify target-domain causal quantities. For confounded contextual bandits, identification reduces to systems of bridge-function integral equations, enabling minimax estimation from finite samples. Confidence sets for the identified quantities are then used for pessimistic policy optimization, yielding CAT-CCB. We extend the framework to episodic confounded Markov decision processes, where complementary sources identify target reward and transition mechanisms under structural and completeness conditions. Under Gaussian dynamics and linear function approximation, this yields CAT-CMDP. We establish high-probability finite-sample suboptimality guarantees that quantify the effects of approximation, estimation, source-to-target shift, offline coverage, and sample size. For bounded finite-dimensional linear classes, the contextual-bandit guarantee recovers the standard $\widetilde{\mathcal O}(N^{-1/2})$ rate. Numerical experiments support the effectiveness of the framework for learning target-domain policies from heterogeneous confounded data. Together, the results provide a principled approach when no single source dataset suffices for causal identification and policy optimization.
\end{abstract}

\include*{chap1}
\include*{chap2}
\include*{chap3}
\include*{chap4}
\include*{chap5}
\include*{chap6}
\include*{chap7}
\include*{chap8}

\include*{chap9}

\newpage
\bibliographystyle{ims}
\bibliography{reference}

\newpage 
\appendix
\include*{append1}
\include*{append2}
\include*{append3}
\include*{append4}
\include*{append5}
\input{appendix_scm.tex}
\end{document}

%% file: chap1.tex
\section{Introduction}
Offline reinforcement learning (RL) seeks to learn decision policies from previously collected data, without requiring further interaction with an environment that may be costly, slow, or unsafe to explore (\citealp{lange2012batch}; \citealp{levine2020offline}; \citealp{prudencio2023survey}). Its potential is especially compelling in applications such as healthcare (\citealp{raghu2018behaviour}; \citealp{gottesman2019guidelines}) and robotics (\citealp{kalashnikov2018scalable}; \citealp{fu2021benchmarks}), where large observational datasets may be available even though prospective experimentation is limited. However, the environment in which an offline policy will ultimately be deployed often differs from those in which the available data were collected, potentially limiting the direct transfer of learned policies (\citealp{levine2020offline,dulac2021challenges}). Moreover, the recorded variables may omit factors that affected both historical actions and their outcomes, inducing spurious associations between actions and outcomes. Together, these issues motivate our study of offline policy learning under the dual challenges of cross-domain heterogeneity and unmeasured confounding (\citealp{bareinboim2016causal,kallus2018confounding,wang2021provably}). 

\paragraph{Key Challenges.} The transition from theoretical potential to practical deployment is currently hindered by three interconnected challenges: domain discrepancy between source and target environments, unmeasured confounding, and insufficient data coverage.

First, the data-collection environments (the \emph{source domains}) may differ structurally from the environment in which the learned policy will be applied (the \emph{target domain}). Such domain discrepancy can arise from simplified experimental or simulated environments and intrinsic differences among populations, institutions, or physical systems (\citealp{taylor2009transfer, levine2020offline, dulac2021challenges}). For example, a simulator or controlled experiment may omit mechanisms that become relevant during deployment (\citealp{tobin2017domain, peng2018sim}). Because these differences can alter the mechanisms governing rewards or state transitions, directly pooling heterogeneous data or deploying a source-optimal policy need not produce reliable target-domain decisions (\citealp{schulam2017reliable, subbaswamy2019preventing}).

Second, recording every variable relevant to historical decisions and outcomes (\citealp{dulac2021challenges}) in an offline dataset can be expensive or even impossible. An unmeasured confounder is a latent factor that affects both the action selected by the logging policy and the subsequent reward or transition (\citealp{pearl_2009, vdw13confounder}). In autonomous driving, for instance, unrecorded road or weather conditions may affect both a driver's braking decision and the probability of an accident, while imperfect sensing further obscures the operative state (\citealp{amodei2016concrete, michelmore2018evaluating}). Consequently, observational conditional distributions need not coincide with the interventional distributions required for policy evaluation. Methods that do not adjust for this distinction may attribute outcome differences to actions when they are partly induced by latent common causes (\citealp{bareinboim2015bandits, lattimore2016causal}). 

Third, because an offline dataset is fixed, it may provide limited coverage of the state--action pairs relevant to a candidate target policy (\citealp{lange2012batch, levine2020offline, wang2020statistical}). Data generated by particular experimental designs or behavior policies typically concentrate on only a subset of possible decisions, creating a mismatch between the observed distribution and the visitation distribution of a desirable policy (\citealp{fujimoto2019off, kumar2020conservative}). Estimation error can then be amplified when policy optimization favors poorly supported actions. Pessimistic offline RL addresses this statistical difficulty by optimizing a lower-confidence assessment of policy value rather than an unsupported point estimate (\citealp{pmlr-v139-jin21e}).

The three challenges commonly occur together in applications:
\begin{itemize}
\item In \emph{autonomous control and robotics}, policies are frequently trained in simulators or laboratories that omit rare disturbances, hardware imperfections, or complex human--environment interactions (\citealp{jakobi1995noise, peng2018sim}). Logged actions and outcomes may therefore be affected by latent factors, while deployment dynamics differ from those represented in the training data (\citealp{de2019causal}).

\item In \emph{operations research and resource allocation}, historical pricing, inventory, or routing decisions can reflect unobserved market conditions or managerial judgment (\citealp{besbes2012blind, athey2017state}). Data collected across regions or time periods may also encode different mechanisms, so naive pooling can distort estimated intervention effects.

\item In \emph{biomedicine}, trial populations may differ from target patients in socioeconomic, genetic, or clinical characteristics (\citealp{rothwell2005external,ackerman2021generalizing}). Observational records may also omit relevant clinical context, so policies learned without adequate adjustment may reproduce treatment-withdrawal patterns associated with comfort-care transitions and yield unsafe recommen dations (\citealp{ji2021trajectory,ren2026unmasking}). Reliable policy learning may therefore require both confounding adjustment and transport across populations.

\item In \emph{social science and economics}, policy evaluation often draws on multiple studies or jurisdictions with different institutions and latent socioeconomic factors (\citealt{deaton2018understanding, vivalt2020much}). Reliable policy design consequently depends on understanding which causal mechanisms can be transported to the setting of interest.
\end{itemize}

Addressing these challenges requires separating causal mechanisms that remain invariant across domains from those that change, while also accounting for finite-sample uncertainty. We use structural causal models (SCMs) and causal directed acyclic graphs (DAGs) to describe the relevant data-generating mechanisms and latent confounding structure (\citealp{pearl_2009}). We then draw on transportability theory, which gives graphical and algebraic criteria for determining whether a target causal quantity can be expressed using information available from heterogeneous source domains (\citealp{Pearl_Bareinboim_2011, Lee_Correa_Bareinboim_2020}). Instrumental and proxy-variable techniques provide additional identifying information when observational actions are confounded. Finally, we combine these identification arguments with pessimistic policy optimization (\citealp{pmlr-v139-jin21e}) so that uncertainty caused by finite samples and limited coverage is propagated into the learned policy.

\paragraph{Our Contributions.}
We develop a framework for transportable offline policy learning from heterogeneous source datasets in the presence of unmeasured confounding. Its central contribution is to identify and estimate target-domain causal quantities by combining complementary information across source domains. We then propagate finite-sample uncertainty in these quantities into policy learning through data-dependent confidence sets and pessimistic optimization. Our main contributions are:

\begin{itemize}
    \item \textit{Causal identification across heterogeneous domains:}
    We combine transportability theory with instrumental-variable, negative-control, and proxy-variable arguments to identify target-domain causal quantities. For confounded contextual bandits, we characterize target causal rewards through integral equation systems of bridge functions; for confounded MDPs, we extend the framework to identify target reward and transition mechanisms.

    \item \textit{Minimax estimation and uncertainty-aware offline policy optimization:}
    We estimate the bridge functions from finite data using minimax objectives, which transform conditional moment constraints into empirical unconditional saddle-point problems. Building on the standard pessimism principle in offline RL, we integrate pessimistic optimization into the causal transportability framework, which optimize worst-case policy values over the corresponding data-dependent uncertainty sets.

    \item \textit{Finite-sample guarantees:}
    We establish high-probability suboptimality bounds for the resulting policies, characterizing the effects of function approximation, statistical estimation, source-to-target distribution shift, offline coverage, and sample size.

    \item \textit{Numerical validation:}
    We conduct numerical experiments to evaluate causal estimation and target-domain policy learning from heterogeneous and confounded offline data. The results support the effectiveness of the overall framework.
\end{itemize}

Overall, the main methodological advance lies in combining causal transportability and confounding adjustment across complementary source domains, with pessimism serving as a downstream mechanism for incorporating statistical uncertainty into policy selection. 

\subsection{Related Works}
{\noindent\bf Structural causal models.} The application of structural equations and graph models traces back to some early work of econometrics (\citealp{wright1928tariff}; \citealp{strolz1960}; \citealp{goldberger1972structural}) and socilogy (\citealp{duncan1966}; \citealp{duane1974}), which traditionally assumes parametric or linear functional forms. To express the interrelation and control of variables in a economic system, a symbolic engine called \textit{do-calculus} (\citealp{judea1995}) for inference on causal directed acyclic graphs, which gives rise to a series of graph algorithms (\citealp{tian2002}; \citealp{shipster2006}; \citealp{Pearl_Bareinboim_2011}; \citealp{textor2012adjustment}). In further research, the structural causal models (\citealp{pearl_2009,Angrist2009,BPSM21}; \citealp{hm2023causal}) are developed, which form the basis for our discussion in this paper.

For adjustment of unmeasured confounding (\citealp{pearl_2009}; \citealp{vdw13confounder}), a large body of works in causal inference literature leverage side observations like instrumental variables (\citealp{angrist1996iv,nonparamivreg,Baiocchi2014,singh2020kernel}) and proxy variables (\citealp{miao2018identifying}; \citealp{miao2020confounding}; \citealp{tchetgen2020introduction,kallus2022causal,cui2023semiparametric}), which play a crucial role in identification of causal quantities associated with interventional distributions in structural causal models.

\vspace{0.2cm}
{\noindent\bf Transportability.} A central goal of causal inference is to extrapolate causal knowledge across domains, a problem known as transportability (\citealp{Pearl_Bareinboim_2011}) or external validity (\citealp{pearl2022external}). Our framework builds on work that develops graphical and algebraic criteria for determining whether and how target causal quantities can be identified from data collected in one or more source domains (\citealp{Pearl_Bareinboim_2011}; \citealp{lee_ztrans}; \citealp{Bareinboim_Pearl_2013}; \citealp{pmlr-v31-bareinboim13a}; \citealp{bepj2014}; \citealp{Lee_Correa_Bareinboim_2020}; \citealp{softtrans2020}). Among these works, most relevant to our setting are meta-transportability (\citealp{pmlr-v31-bareinboim13a}) and general transportability (\citealp{Lee_Correa_Bareinboim_2020}), both of which address identification from multiple source domains.

Recent work connects transportability with adjustment for unmeasured confounding and causal data fusion. \citet{li2025identification} identify and multiply robustly estimate target-population causal effects using instruments available only in an auxiliary population, whereas \citet{iglesias2025transferring} establish multi-domain identification using proxies for a latent confounder. In bandit settings, \citet{deng2025transfer} study latent contextual bandits under covariate shift, while \citet{park2025transportability} use heterogeneous observational and experimental data to accelerate online structural causal bandits. These works focus on static effect estimation or online learning, whereas our framework instead combines multi-source causal identification with pessimistic offline policy optimization.
\vspace{0.2cm}

{\noindent\bf Offline reinforcement learning and pessimism.} Offline RL seeks to learn policies from static historical data without further interaction with the environment (\citealp{levine2020offline}). Our work adopts the pessimism principle (\citealp{liu2020provably,pmlr-v139-jin21e}; \citealp{xie2023bellmanconsistent}; \citealp{uehara2023pessimistic}), which addresses distribution shift between the data-generating policy and the policies being evaluated or optimized (\citealp{wang2020statistical}; \citealp{levine2020offline}). 

A growing literature studies policy learning and evaluation under unmeasured confounding (\citealp{lu22prior}; \citealp{fu2022offline}; \citealp{chen2022instrumental}; \citealp{chen2023unified}; \citealp{lu2023pessimism,liao2021instrumental}) and across heterogeneous offline datasets (\citealp{mao2024stylized,gong2025provably,carranza2025robust}). Within the first strand, \citet{shao2024decision} develop an IV-based double-machine-learning method for single-stage policy learning, \citet{wang2024pessimistic} combine mediator-based identification with pessimistic offline RL, and \citet{chen2025quantile} study quantile-optimal policy learning using instrumental variables, negative controls, and pessimism. Among approaches to heterogeneous data, \citet{gong2025provably} derive causal bounds from ambiguity sets and use them to accelerate online bandit learning, while \citet{carranza2025robust} minimize worst-case regret over mixtures of observational source distributions without identifying the causal mechanisms of a specified target domain.

\subsection{Preliminaries}
\paragraph{Structural causal models}(\citealp{pearl_2009}).  A structural causal model (SCM) is a quadruple $\mathcal{M}=(\mathbf{U},\mathbf{V},\mathcal{F},p)$, where
\begin{itemize}
\item $\mathbf{U}$ is a set of exogenous variables that are determined by factors outside the model. 
\item $\mathbf{V} = \lbrace V_1, \cdots , V_n\rbrace$ is a set of endogenous variables determined by factors in the model.
\item $\mathcal{F}$ is a set of functions $\lbrace f_1, \cdots, f_n\rbrace$ such that each $f_i$ is a mapping from  (the respective domains of) $U_i\cup Pa_i$ to $V_i,$ where $U_i\subseteq\mathbf{U}$ is some subset of $U$ whose elements are related to the variation of $V_i,$ $Pa_i\subseteq V\backslash V_i$ includes the parents of $V_i$ in the underlying causal DAG, and entire set of $F$ forms a mapping from $U$ to $V.$ In other words, $f_i$ assigns a value to the corresponding $V_i\in\mathbf{V},\ v_i\gets f(u_i,pa_i).$
\item $p$ is a probability distribution defined over the domain of $\mathbf{U}$, which creates variation in the endogenous variables $\mathbf{V}$.
\end{itemize}

Every SCM $\mathcal{M}$ gives rise to a directed acyclic graph (DAG) $\mathcal{G}(\mathcal{M}) = (\mathbf{V},\mathbf{E})$ (denoted as $\mathcal{G}$ for simplicity). Nodes in $\mathcal{G}$ correspond to endogenous variables in $\mathbf{V}$, and directed edges $\mathbf{E}$ emit from the set of parent nodes $Pa_i$ towards $V_i.$ This causal DAG specifies the causal mechanism according to which the data are generated.

Let $\mathcal{G}=(\mathbf{V},\mathbf{E})$ be a causal graph (DAG), and $\mathbf{A}$ be a subset of $\mathbf{V}$. Then we use $\mathcal{G}_{\overline{\mathbf{A}}}$ to denote the causal diagram induced by intervening on variables $\mathbf{A}$, i.e., the edges pointing to nodes in $\mathbf{A}$ are removed from $\mathcal{G}$. Similarly, let $\mathcal{G}_{\underline{\mathbf{A}}}$ denote the causal diagram resulting from removing edges that are emitted by nodes in $\mathbf{A}$.

\paragraph{Symbolic engine: do-calculus} (\citealp{judea1995}). We denote by $p(y|x,\mathrm{do}(z))$ the interventional distribution of $Y$ conditional $Z$ after intervening $Z=z$. Let $X, Y, Z$ and $W$ be arbitrary disjoint sets of nodes in $\mathcal{G}$. We use $p$ to stand for probability distribution. Do-calculus states the following three rules. To establish identifiability of a query $Q$, one needs to repeatedly apply the rules of do-calculus to $Q$ until the final expression no longer contains a do-operator.
\begin{itemize}
    \item Rule 1 (Insertion/deletion of observations).
  \begin{equation*}
      p(y\vert \mathrm{do}(x), z, w) = p(y\vert\mathrm{do}(x),w),\ \ \text{if}\ Y\perp Z\vert X,W\ \text{in}\ {\mathcal{G}_{\overline{X}}};
  \end{equation*}
  \item Rule 2 (Action/observation exchange).
  \begin{equation*}
  p(y\vert \mathrm{do}(x),\mathrm{do}(z), w) = p(y\vert\mathrm{do}(x),z, w),\ \ \text{if}\ Y\perp Z\vert X,W\ \text{in}\ {\mathcal{G}_{\overline{X},\underline{Z}}};\end{equation*}

  \item Rule 3 (Insertion/deletion of actions).
  \begin{equation*}
  p(y\vert \mathrm{do}(x),\mathrm{do}(z), w) = p(y\vert\mathrm{do}(x), w),\ \ \text{if}\ Y\perp Z\vert X,W\ \text{in}\ {\mathcal{G}_{\overline{X},\overline{Z(W)}}},\end{equation*}
  
where $Z(W)$ denotes the set of $Z$-nodes that are not ancestors of any $W$-nodes in $\mathcal{G}_{\overline{X}}$.
\end{itemize}

\paragraph{Selection diagram} (\citealp{Lee_Correa_Bareinboim_2020}). Suppose that both $\mathcal{M}^*$ and $\{\mathcal{M}^{(i)}\}_{i=1}^n$ conform to a causal graph $\mathcal{G}_0=(\mathbf{V},\mathbf{E})$. Then for any $i\in[n]$, we denote by $\Delta^{(i)}$ a set of variables such that for any $V\in\Delta^{(i)}$, there might exist a discrepancy: either the structural equations $f^{(i)}_V\neq f^*_V$ or $p^{(i)}(\mathbf{u}_V)\neq p^*(\mathbf{u}_V)$. Moreover, let $\mathbf{C}=\{C_V\,|\,V\in\bigcup_{i=1}^n\Delta^{(i)}\}$ be selection variables. Then $\mathcal{G}_0$ can be augmented. A selection diagram is defined as a graph $\mathcal{G}=(\mathbf{V}\cup\mathbf{C},\mathbf{E}\cup\{C_V\to V\}_{C_V\in \mathbf{C}})$.

We denote $\mathcal{M}^{(i)}$-specific variables by $\mathbf{C}^{(i)}=\{C_V\}_{V\in\Delta^{(i)}}$, and the remainder by $\mathbf{C}^{(-i)}=\mathbf{C}\backslash\mathbf{C}^{(i)}$. Then the range of a selection variable $C_V$ is $\{*\}\cup\{i|V\in\Delta^{(i)}\}$. Consequently, the selection diagram $\mathcal{G}$ can be viewed as a unifying SCM where the selection variables works like switches specifying the SCM of interest:
\begin{equation}
	p(\mathbf{y}|\mathrm{do}(\mathbf{a}),\mathbf{x},\mathbf{C}^{(i)}=i,\mathbf{C}^{(-i)}=*) = p^{(i)}(\mathbf{y}|\mathrm{do}(\mathbf{a}),\mathbf{x}).
\end{equation}

\paragraph{S-admissibility} (\citealp{Pearl_Bareinboim_2011}). A set of variables $\mathbf{Z}$ is said to be an S-admissible set between $\mathcal{M}^*$ and $\mathcal{M}^{(i)}$ with respect to the causal effect of $\mathbf{A}$ on $\mathbf{Y},$ if
\begin{equation*}
\left(\mathbf{C}^{(i)}\indep\mathbf{Y}|\mathbf{Z}\right)_{\mathcal{G}_{\overline{\mathbf{A}}}},
\end{equation*}
where $\mathcal{G}_{\overline{\mathbf{A}}}$ is induced by removing all arrows towards nodes in $\mathbf{A}$ from the selection diagram $\mathcal{G}$. Moreover, we have the direct transportability:
\begin{equation}
p^*(\mathbf{y}|\mathrm{do}(\mathbf{a}),\mathbf{w}) = p^{(i)}(\mathbf{y}|\mathrm{do}(\mathbf{a}),\mathbf{w}).\label{dteq}
\end{equation}
\begin{proof}
By equation (\ref{dteq}) and do-calculus rule 1, we have
\begin{align*}
	p^*(\mathbf{y}|\mathrm{do}(\mathbf{a}),\mathbf{w}) 
	&= p(\mathbf{y}|\mathrm{do}(\mathbf{a}),\mathbf{x},\mathbf{C}^{(i)}=*,\mathbf{C}^{(-i)}=*)\\
	&= p(\mathbf{y}|\mathrm{do}(\mathbf{a}),\mathbf{x},\mathbf{C}^{(i)}=i,\mathbf{C}^{(-i)}=*)
	= p^{(i)}(\mathbf{y}|\mathrm{do}(\mathbf{a}),\mathbf{w}),
\end{align*}
which concludes the proof.
\end{proof}

\paragraph{Rademacher complexity and critical radius.} Suppose we have a function class $\mathcal{F}:\mathcal{X}\to\mathbb{R}$ that is uniformly bounded by $b$, i.e. $\forall f\in\mathcal{F},\Vert f\Vert_\infty \leq b$. The localized empirical Rademacher complexity of a function class $\mathcal{F}$ based on a sample $\{x_i\}_{i=1}^n$ is defined as
\begin{equation}
	\mathcal{R}(\delta;\mathcal{F},x_{1:n}) := \E\left[\sup_{f\in\mathcal{F}:\Vert f\Vert_{P_n}\leq\delta}\frac{1}{n}\sum_{i=1}^n\epsilon_if(x_i)\right],\label{radcomp}
\end{equation}
where $\{\epsilon_i\}_{i=1}^n$ are independent Rademacher variables, i.e. $\mathbb{P}(\epsilon_i=1) = \mathbb{P}(\epsilon_i=-1) = 1/2$, and $\Vert\cdot\Vert_{P_n}$ is the $L_2$ norm with respect to empirical distribution $P_n:=\frac{1}{n}\sum_{i=1}^n \delta_{x_i}$. The empirical critical radius of $\mathcal{F}$ on dataset $\{x_i\}_{i=1}^n$ is defined as the minimum positive solution to $\mathcal{R}(\delta;\mathcal{F},x_{1:n})\leq\delta^2/b$:
\begin{equation}
	\widehat{\delta}_n := \min\left\{\delta > 0:\mathcal{R}(\delta;\mathcal{F},x_{1:n})\leq\frac{\delta^2}{b}\right\}.\label{criticalradius}
\end{equation}
\paragraph{Notations.} We use capital letters like $X,A,Y$ to denote random variables, and use the corresponding calligraphic letters like $\mathcal{A},\mathcal{X},\mathcal{Y}$ to denote their range.  Suppose $\mathcal{A}$ is a measurable space, we let $\Delta(\mathcal{A})$ denote the family of probability distributions over $\mathcal{A}$. We write $X\indep Y$ if random variables $X$ and $Y$ are independent, and we write $X\indep Y\,|\,Z$ if they are conditionally independent given variable $Z$. For any $H\in\mathbb{N}$, we use $[H]$ to denote consecutive integers from $1$ to $H$, i.e. $[H] = \{1,2,\cdots,H\}$. We denote by $L^2(\mathcal{X},\mu)$ by the $L^2$-space (i.e. the square-integrable functions) defined on a measure space $(\mathcal{X},\mu)$, and we omit $\mu$ if there is no ambiguity. For any function $f$, we let $\mathcal{O}(f)$ denote the functions that increase as fast as $f$, i.e. there exists constants $C_1,C_2>0$ such that $C_1f\leq \mathcal{O}(f)\leq C_2f$. We let $\widetilde{O}(f)$ denote $\mathcal{O}(f\cdot\mathrm{poly}(\log f))$.

\subsection{Roadmap}
In \S\hyperref[sec:2]{2}, we discuss the impact of domain discrepancy in unconfounded contextual bandits and the possibility of leverage it with transportability, and we also formalize the problem setting in presence of unmeasured confounders. In \S\hyperref[sec:3]{3}, we discuss how to use transportability to identify causal quantities introduced in \S\hyperref[sec:2]{2} in some typical contextual bandits. In \S\hyperref[sec:4]{4}, we discuss the estimation of conditional moments, and propose our algorithm CAT-CCB to solve the offline policy learning problem in confounded contextual bandits. In \S\hyperref[sec:5]{5}, we expand the setting of domain discrepancy to Markov decision processes (MDPs) in presence of unobserved confounders, and discuss the causal identification in confounded MDPs using instrumental variables. In \S\hyperref[sec:6]{6}, we further propose our CAT-CMDP algorithm to solve the offline policy learning problem discussed in \S\hyperref[sec:5]{5} under Gaussian dynamics. In \S\hyperref[sec:7]{7}, we provide convergence analysis of our CAT algorithms proposed in \S\hyperref[sec:4]{4} and \S\hyperref[sec:6]{6} under regular conditions. Finally, in \S\ref{sec:experiments}, we conduct numerical experiments to demonstrate the effecitveness of our framework for offline policy learning from heterogeneous and confounded data.

%% file: chap2.tex
\section{Contextual Bandits}
\label{sec:2}

We first study the problem of offline policy learning in contextual bandits (\citealp{langford2007epoch}). We clarify the problem setting and the performance metric of contextual bandits in $\S$\ref{sec:2.1}. We discuss the impact of domain discrepancy in contextual bandits in $\S$\ref{sec:2.2}, and we also give some primary approaches to deal with it. In $\S$\ref{sec:2.3}, we further discuss the challenge of policy learning in contextual bandits in presence of both domain discrepancy and unmeasured confounding, and formalize our setting of confounded contextual bandits.

\subsection{Problem Setting}
\label{sec:2.1}
A \textit{contextual bandit} can be viewed as a one-step sequential decision problem in which the agent observes information about the current environment before choosing an action. It naturally extends the classical multi-armed bandit problem by incorporating environmental information, referred to as a \textit{context}, into each round of interaction, so that the reward associated with an action may depend on the observed context.

Concretely, let $\mathcal{X}$ and $\mathcal{A}$ denote the context space and the action space, respectively, and let $\mathcal{Y}\subseteq\mathbb{R}$ denote the reward space. We write $X\in\mathcal{X}$ for the random context, $A\in\mathcal{A}$ for the action, and $Y\in\mathcal{Y}$ for the resulting reward. We assume that $\mathcal{X}$ and $\mathcal{A}$ are measurable spaces, so that probability distributions and policies on these spaces are well defined. A single round of interaction proceeds as follows: 
\begin{itemize} 
\item The environment generates a context $X\in\mathcal{X}$ according to a distribution $p_X(x)$. 
\item After observing $X$, the agent chooses an action $A\in\mathcal{A}$ according to its \textit{policy}. 
\item Conditional on the context $X$ and the chosen action $A$, the environment generates a reward $Y\in\mathcal{Y}$ according to a conditional distribution $p_Y(y\,|\,x,a)$. 
\end{itemize} 
The statistical dependence structure is shown by the directed acyclic graph in Figure \ref{fig:contextual-bandit-dag}. 

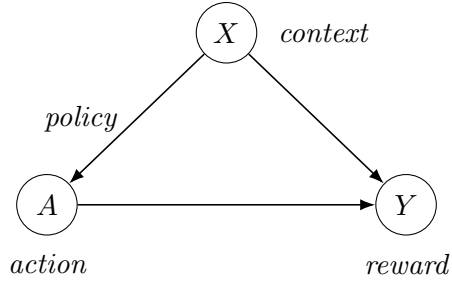
\begin{figure}[t]
    \centering
    \begin{tikzpicture}[node distance=1.7cm and 1.8cm]
        \node[dagvar] (X) {$X$};
        \node[dagvar, below left=of X]  (A) {$A$};
        \node[dagvar, below right=of X] (Y) {$Y$};

        \draw[dagarrow]
            (X) -- node[left, daglabel] {policy\,\, } (A);
        \draw[dagarrow] (X) -- (Y);
        \draw[dagarrow] (A) -- (Y);

        \node[daglabel, right=4pt of X] {context};
        \node[daglabel, below=3pt of A] {action};
        \node[daglabel, below=3pt of Y] {reward};
    \end{tikzpicture}

    \caption{The statistical DAG for a contextual bandit.}
    \label{fig:contextual-bandit-dag}
\end{figure}

In the interaction above, the ``policy'' describes an agent's preference over actions conditional on the observed context. Formally, a policy $\pi:\mathcal{X}\to\Delta(\mathcal{A})$ assigns to every context $x\in\mathcal{X}$ a probability distribution $\pi(\cdot\,|\,x)$ over the action space $\mathcal{A}$. Thus, conditional on observing $X=x$, the action is sampled according to $A\sim\pi(\cdot\,|\,x)$. A deterministic policy is a special case in which $\pi(\cdot\,|\,x)$ is supported on a single action for every $x\in\mathcal{X}$. 

We use $p_X(x)$ to denote the probability density or probability mass function of the context $X$, and write $p(x)$ when there is no ambiguity. Similarly, $p(y\,|\,x,a)$ denotes the conditional distribution of the reward given $X=x$ and $A=a$. Throughout the paper, we use densities or probability mass functions with respect to appropriate reference measures and write integrals for notational convenience; when a variable is discrete, the corresponding integral is understood as a sum. Under a policy $\pi$, the interaction above induces the joint distribution 
\begin{equation*} 
p^\pi(x,a,y) = p(x)\,\pi(a\,|\,x)\,p(y\,|\,x,a), \qquad (x,a,y)\in\mathcal{X}\times\mathcal{A}\times\mathcal{Y}. \end{equation*} 
In the offline setting, the agent cannot interact further with the environment. Instead, it is provided with a fixed dataset \begin{equation*} 
\mathcal{D} = \{(x_t,a_t,y_t)\}_{t=1}^{T}, 
\end{equation*} 
where each tuple $(x_t,a_t,y_t)$ records one round of interaction generated before policy learning. Based only on this offline dataset, our objective is to learn a policy that achieves a large expected reward. 

For any policy $\pi$, we define its \textit{value} as the expected reward under the induced distribution $p^\pi$: 
\begin{equation*} 
v^\pi = \E_{p^\pi}[Y] = \int_{\mathcal{X}\times\mathcal{A}\times\mathcal{Y}} y\,p^\pi(x,a,y) \,\mathrm{d}y\,\mathrm{d}a\,\mathrm{d}x. \end{equation*} 
Denote by $\Pi=[\Delta(\mathcal{A})]^{\mathcal{X}}$ the class of context-dependent policies under consideration. An optimal policy is any policy
\begin{equation*} 
\pi^* \in \underset{\pi\in\Pi}{\mathrm{argmax}}\;v^\pi. \end{equation*} 
Because the policy must be learned from a finite offline dataset, the learned policy need not coincide with $\pi^*$. We therefore evaluate a policy $\pi$ through its suboptimality, 
\begin{equation*} 
\mathrm{SubOpt}(\pi) := v^{\pi^*}-v^\pi. 
\end{equation*} 
A statistically consistent offline policy-learning procedure should produce policies whose suboptimality approaches zero as the amount of informative offline data increases.

\subsection{Domain Discrepancy}
\label{sec:2.2}

In offline policy learning, the data available to the learner are typically collected in advance from one or more environments, which we refer to as \textit{source domains}. The policy, however, may ultimately be deployed in a different environment, referred to as the \textit{target domain}. When the data-generating mechanisms differ across these domains, a policy learned directly from a source dataset need not remain optimal in the target domain. We refer to such differences between the source and target environments as \textit{domain discrepancy}.

To first isolate the effect of domain discrepancy from that of unmeasured confounding, throughout this subsection we consider the standard unconfounded contextual-bandit setting introduced in \S\ref{sec:2.1}. All domains share the same context, action, and reward spaces $\mathcal{X}$, $\mathcal{A}$, and $\mathcal{Y}$, while their underlying probability distributions may differ.

\paragraph{Contextual bandits under domain discrepancy.}
Recall that a contextual bandit $\mathscr{C}$ is characterized by the context distribution $p(x)$ and the conditional reward distribution $p(y\,|\,x,a)$. For a given context $x\in\mathcal{X}$, an optimal policy $\widetilde{\pi}$ places probability only on actions attaining the largest conditional expected reward:
\begin{equation}
    \supp\bigl(\widetilde{\pi}(\cdot|x)\bigr)
    \subseteq
    \argmax_{a\in\mathcal{A}}
    \int_{\mathcal{Y}} y\,p(y|x,a)\,\mathrm{d}y.
    \label{eq:optimal_cb}
\end{equation}
Now suppose that $\mathscr{C}$ is the source domain, while the target domain $\mathscr{C}^*$ is characterized by $p^*(x)$ and $p^*(y\,|\,x,a)$. Equation \eqref{eq:optimal_cb} reveals two qualitatively different types of domain discrepancy:
\begin{itemize}
    \item If only the context distribution changes,
    \begin{equation*}
        p^*(x)\neq p(x),
        \qquad
        p^*(y\,|\,x,a)=p(y\,|\,x,a),
    \end{equation*}
    then under the unrestricted class of context-dependent policies considered in \S\ref{sec:2.1}, the optimal policy remains unchanged, since the optimal action for each context $x$ depends only on $p(y\,|\,x,a)$. However, such a shift may still create a statistical difficulty in the offline setting: if the source dataset provides insufficient coverage of contexts that occur frequently in the target domain, the optimal action at those contexts may be difficult to estimate reliably.

    \item In contrast, suppose that the conditional reward mechanism changes,
    \begin{equation*}
        p^*(y\,|\,x,a)\neq p(y\,|\,x,a).
    \end{equation*}
    In this case, even for the same context $x$, the expected rewards associated with different actions may change substantially across domains, and consequently the optimal policy in the target domain can be very different from that in the source domain. Hence an optimal policy learned in the source domain $\mathscr{C}$ is generally not \textit{transportable} to the target domain $\mathscr{C}^*$ without further structural assumptions.
\end{itemize}

\subsubsection{Contextual Bandits with Side Observations}
The preceding discussion shows that transporting a policy becomes nontrivial when the conditional reward mechanism changes across domains. A useful possibility arises when the reward-generating process admits additional observed variables whose mechanisms differ across domains in a structured manner. In particular, by introducing a side observation $W\in\mathcal{W}$, the conditional reward distribution can be decomposed into multiple components, and different components may be recovered from different source domains.

We consider two representative causal structures, illustrated in Figure~\ref{fig:contextual-bandit-side-observations}. In the first structure, $W$ is not affected by the action and serves as a \textit{negative outcome}. In the second structure, $W$ lies on a causal path from the action to the reward and serves as a \textit{mediator}. These two examples provide simple illustrations of how information from heterogeneous source domains can be combined to recover the target-domain decision rule.

\begin{figure}[t]
    \centering
    \begin{tikzpicture}
        \begin{scope}
            \node[dagvar] (X1) at (0,2.5) {$X$};
            \node[dagvar] (W1) at (3.0,2.5) {$W$};
            \node[dagvar] (A1) at (-1.5,0) {$A$};
            \node[dagvar] (Y1) at (1.5,0) {$Y$};

            \draw[dagarrow] (X1) -- (W1);
            \draw[dagarrow] (X1) -- (A1);
            \draw[dagarrow] (X1) -- (Y1);
            \draw[dagarrow] (W1) -- (Y1);
            \draw[dagarrow] (A1) -- (Y1);

            \node at (0.6,-0.75) {$\mathcal{G}_1$};
        \end{scope}

        \begin{scope}[xshift=7.2cm]
            \node[dagvar] (X2) at (0,2.5) {$X$};
            \node[dagvar] (W2) at (0,0.80) {$W$};
            \node[dagvar] (A2) at (-1.8,0) {$A$};
            \node[dagvar] (Y2) at (1.8,0) {$Y$};

            \draw[dagarrow] (X2) -- (W2);
            \draw[dagarrow] (X2) -- (A2);
            \draw[dagarrow] (X2) -- (Y2);
            \draw[dagarrow] (A2) -- (W2);
            \draw[dagarrow] (W2) -- (Y2);
            \draw[dagarrow] (A2) -- (Y2);

            \node at (0,-0.65) {$\mathcal{G}_2$};
        \end{scope}
    \end{tikzpicture}

    \caption{The statistical DAGs for contextual bandits with side
    observations. The left diagram describes a contextual bandit with a negative outcome $W$ that is not influenced by the action $A$. The right diagram describes a contextual bandit with a mediator $W$ lying on a causal path from $A$ to the reward $Y$.}
    \label{fig:contextual-bandit-side-observations}
\end{figure}
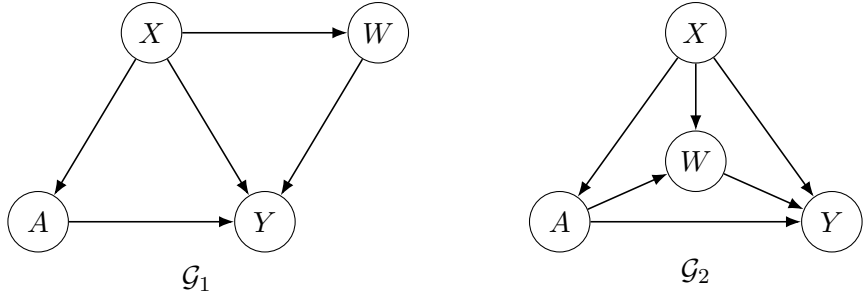

\paragraph{Negative outcome.}
We first consider the left diagram in Figure~\ref{fig:contextual-bandit-side-observations}. Here $W$ may depend on the context $X$, but it is not affected by the action $A$; equivalently, $A\perp\!\!\!\perp W\,|\, X$. Under a policy $\pi:\mathcal{X}\to\Delta(\mathcal{A})$, the joint distribution of $(X,A,W,Y)$ factorizes as
\begin{equation}
    p^\pi(x,a,w,y)
    =
    p(x)\pi(a|x)p(w|x)p(y|x,a,w).
    \label{eq:cb_neg_factorization}
\end{equation}
Therefore, for each $x\in\mathcal{X}$, an optimal policy $\widetilde{\pi}$ satisfies
\begin{equation}
    \supp\bigl(\widetilde{\pi}(\cdot|x)\bigr)
    \subseteq
    \argmax_{a\in\mathcal{A}}
    \int_{\mathcal{W}\times\mathcal{Y}}
    y\,p(y|x,a,w)p(w|x)\,\mathrm{d}w\,\mathrm{d}y.
    \label{eq:cb_neg_optimal}
\end{equation}

Now suppose that we have two source domains $\mathscr{C}^{(1)}$ and $\mathscr{C}^{(2)}$, with corresponding distributions $p^{(1)}$ and $p^{(2)}$, and a target domain $\mathscr{C}^*$ with distribution $p^*$. Consider the following complementary discrepancies:
\begin{equation}
\begin{cases}
p^{(1)}(w|x)\neq p^*(w|x),
&
p^{(1)}(y|x,a,w)=p^*(y|x,a,w),
\\
p^{(2)}(w|x)=p^*(w|x),
&
p^{(2)}(y|x,a,w)\neq p^*(y|x,a,w).
\end{cases}
\label{eq:cb_neg_discrepancy}
\end{equation}
Thus, source domain $\mathscr{C}^{(1)}$ provides the target conditional reward mechanism given $(X,A,W)$, while source domain $\mathscr{C}^{(2)}$ provides the target conditional distribution of the side observation $W$ given $X$. Neither source domain, by itself, is required to reproduce the full target-domain mechanism.

Combining the invariant components from the two sources, the optimal policy in the target domain satisfies
\begin{align}
    \supp\bigl(\pi^*(\cdot|x)\bigr)
    &\subseteq
    \argmax_{a\in\mathcal{A}}
    \int_{\mathcal{W}\times\mathcal{Y}}
    y\,p^{(1)}(y|x,a,w)p^{(2)}(w|x)
    \,\mathrm{d}w\,\mathrm{d}y
    \nonumber\\
    &=
    \argmax_{a\in\mathcal{A}}
    \E^{(2)}
    \left[
        \E^{(1)}[Y|X=x,A=a,W]
        \,\,|\,dle|\,X=x
    \right],
    \label{eq:cb_neg_transport}
\end{align}
where $\E^{(i)}$ denotes expectation with respect to $p^{(i)}$. This expression illustrates the central idea behind our use of multiple heterogeneous datasets: the target quantity may be recovered by combining different mechanisms that are invariant in different source domains, even when no single source reproduces the entire target domain.

To prepare for the confounded setting studied later and for the minimax estimation procedure in \S\ref{sec:4}, we express the above identification through \textit{bridge functions}. Let
\begin{equation*}
    h_1:\mathcal{W}\times\mathcal{A}\times\mathcal{X}\to\mathbb{R},
    \qquad
    g:\mathcal{A}\times\mathcal{X}\to\mathbb{R},
\end{equation*}
where $h_1(w,a,x):=\E^{(1)}[Y\,|\, W=w,A=a,X=x]$ represents the conditional expected reward in the first source domain and $g:=\E^{(2)}[h_1(W,a,x)\,|\, X=x]$ represents the conditional expected reward in the target domain. They satisfy the integral equation system
\begin{equation}
\begin{cases}
\E^{(1)}
    [h_1(W,A,X)-Y\,|\, W,A,X]=0,
\\
\E^{(2)}
    [g(a,X)-h_1(W,a,X)\,|\, X]=0,
\qquad \forall a\in\mathcal{A}.
\end{cases}
\label{eq:cb_neg_ies}
\end{equation}

We summarize this setting as follows.

\begin{assumption}[Contextual bandits with negative outcome]
\label{assump:cb_neg}
There exists a negative outcome $W\in\mathcal{W}$ such that $A\perp\!\!\!\perp W\,|\, X$ and the joint distribution under any policy factorizes as in \eqref{eq:cb_neg_factorization}. We have two source domains $(\mathscr{C}^{(1)},p^{(1)})$ and $(\mathscr{C}^{(2)},p^{(2)})$ and a target domain $(\mathscr{C}^*,p^*)$ satisfying the discrepancies in \eqref{eq:cb_neg_discrepancy}. Under these conditions, the target conditional expected reward
\begin{equation*}
    g(a,x)=\E^*[Y|A=a,X=x]
\end{equation*}
can be recovered through the integral equation system \eqref{eq:cb_neg_ies}.
\end{assumption}

\paragraph{Mediator.}
We next consider the right diagram in Figure~\ref{fig:contextual-bandit-side-observations}, where the side observation $W$ lies on a causal path from the action $A$ to the reward $Y$. In this case, $W$ is a \textit{mediator}, and its distribution generally depends on both the context and the chosen action. Under a policy $\pi$, the joint distribution factorizes as
\begin{equation}
    p^\pi(x,a,w,y)
    =
    p(x)\pi(a|x)p(w|x,a)p(y|x,a,w).
    \label{eq:cb_med_factorization}
\end{equation}
Accordingly, an optimal policy $\widetilde{\pi}$ satisfies
\begin{equation}
    \supp\bigl(\widetilde{\pi}(\cdot|x)\bigr)
    \subseteq
    \argmax_{a\in\mathcal{A}}
    \int_{\mathcal{W}\times\mathcal{Y}}
    y\,p(y|x,a,w)p(w|x,a)
    \,\mathrm{d}w\,\mathrm{d}y.
    \label{eq:cb_med_optimal}
\end{equation}

Again consider two source domains $\mathscr{C}^{(1)}$ and $\mathscr{C}^{(2)}$ and the target domain $\mathscr{C}^*$. We assume the complementary discrepancies
\begin{equation}
\begin{cases}
p^{(1)}(w\,|\,x,a)\neq p^*(w\,|\,x,a),
&
p^{(1)}(y\,|\,x,a,w)=p^*(y\,|\,x,a,w),
\\
p^{(2)}(w\,|\,x,a)=p^*(w\,|\,x,a),
&
p^{(2)}(y\,|\,x,a,w)\neq p^*(y\,|\,x,a,w).
\end{cases}
\label{eq:cb_med_discrepancy}
\end{equation}
Then the target optimal policy satisfies
\begin{align}
    \supp\bigl(\pi^*(\cdot|x)\bigr)
    &\subseteq
    \argmax_{a\in\mathcal{A}}
    \int_{\mathcal{W}\times\mathcal{Y}}
    y\,p^{(1)}(y|x,a,w)p^{(2)}(w|x,a)
    \,\mathrm{d}w\,\mathrm{d}y
    \nonumber\\
    &=
    \argmax_{a\in\mathcal{A}}
    \E^{(2)}
    \left[
        \E^{(1)}[Y|X=x,A=a,W]
        \,\middle|\,X=x,A=a
    \right].
    \label{eq:cb_med_transport}
\end{align}
Using the same bridge function $h_1$ for the conditional reward and $g$ for the target conditional expected reward, we obtain
\begin{equation}
\begin{cases}
\E^{(1)}
    [h_1(W,A,X)-Y\,|\, W,A,X]=0,
\\
\E^{(2)}
    [g(A,X)-h_1(W,A,X)\,|\, X,A]=0.
\end{cases}
\label{eq:cb_med_ies}
\end{equation}
The second equation differs from that in the negative-outcome setting \eqref{eq:cb_neg_ies} because the distribution of the mediator $W$ is action-dependent. We summarize the mediator setting as follows.

\begin{assumption}[Contextual bandits with mediator]
\label{assump:cb_med}
There exists a mediator $W\in\mathcal{W}$ such that the joint distribution under any policy factorizes as in \eqref{eq:cb_med_factorization}. We have two source domains $(\mathscr{C}^{(1)},p^{(1)})$ and $(\mathscr{C}^{(2)},p^{(2)})$ and a target domain $(\mathscr{C}^*,p^*)$ satisfying the discrepancies in \eqref{eq:cb_med_discrepancy}. Under these conditions, the target conditional expected reward
\begin{equation*}
    g(a,x)=\E^*[Y|A=a,X=x]
\end{equation*}
can be recovered through the integral equation system \eqref{eq:cb_med_ies}.
\end{assumption}

\subsection{Confounding Bias}
\label{sec:2.3}

The discussion in \S\ref{sec:2.2} assumes that the variables relevant to action selection and reward generation are fully observed. In many real-world decision-making problems, however, it can be expensive or even impossible to record every relevant factor. As a result, some variables that influence both the action taken by the agent and the resulting reward may be absent from the offline dataset. We refer to such variables as \textit{unobserved} or \textit{unmeasured confounders} \citep{VanderWeele_Hernan_Robins_2008}. Their presence can induce spurious associations between actions and rewards, so that the observational relationship learned from historical data no longer represents the causal effect of taking an action.

Figure~\ref{fig:confounded-contextual-bandit} gives a simple illustration. In the source domain shown on the left, $U$ is a common cause of the action $A$ and the reward $Y$. The historical agent have may have access to information represented by $U$ when choosing its action, so that its behavior depends on both the observed context $X$ and the confounder $U$. From the perspective of the offline learner, however, $U$ is not recorded, and the available data contain only $(X,A,Y)$. Consequently, conditioning on the observed action $A=a$ also changes the distribution of the latent variable $U$ among the observed samples.

In the target domain, our goal is instead to deploy a policy $\pi:\mathcal{X}\to\Delta(\mathcal{A})$ that chooses actions based only on the observed context $X$. The effect of assigning action $A=a$ is therefore described by the interventional distribution $p(y\,|\, x,\operatorname{do}(a))$, rather than the observational conditional distribution $p(y\,|\, x,a)$. To see the distinction, note that
\begin{align*}
p(y\,|\,x,\mathrm{do}(a)) &= \int_\mathcal{U} p(y\,|\,x,\mathrm{do}(a),u)\,p(u\,|\,x,\mathrm{do}(a))\mathrm{d}u = \int_\mathcal{U} p(y\,|\,x,a,u)\,p(u)\mathrm{d}u,\\
p(y\,|\,x,a) &= \int_\mathcal{U} p(y\,|\,x,a,u)\,p(u\,|\,x,a)\mathrm{d}u = \int_\mathcal{U} p(y\,|\,x,a,u)\,p(u\,|\,a)\mathrm{d}u.
\end{align*}
Because the historical action $A$ is influenced by $U$, we generally have
$p(u\,|\, x,a)\neq p(u\,|\, x)$. Hence $p(y\,|\, x,a)$ can differ substantially from $p(y\,|\, x,\operatorname{do}(a))$, even when both are evaluated at the same context and action.

This distinction is the source of \textit{confounding bias}. In particular, estimating the expected reward of an intervention from the observational conditional mean $\E[Y| X=x,A=a]$ can be misleading. The relevant quantity for policy learning is instead the \textit{causal effect}
\begin{equation*}
    \E[Y| X=x,\operatorname{do}(A=a)],
\end{equation*}
which is generally not identifiable from the observed variables alone without additional causal assumptions or auxiliary information. In the following two subsections, we formalize the observational and interventional processes of our confounded contextual-bandit model and then define the causal quantity that will serve as the target of identification and policy optimization.

\begin{figure}[t]
    \centering
    \begin{tikzpicture}
        \begin{scope}
            \node[dagvar] (X1) at (0,1.5) {$X$};
            \node[dagvar] (A1) at (-2.0,0) {$A$};
            \node[dagvar] (Y1) at (2.0,0) {$Y$};
            \node[dagvar] (U1) at (0,-1.5) {$U$};

            \draw[dagarrow] (X1) -- (A1);
            \draw[dagarrow] (X1) -- (Y1);
            \draw[dagarrow] (A1) -- (Y1);

            \draw[dagarrow,dashed] (U1) -- (A1);
            \draw[dagarrow,dashed] (U1) -- (Y1);

            \node[daglabel, right=4pt of X1] {context};
            \node[daglabel, right=5pt of U1] {common cause};
        \end{scope}

        \begin{scope}[xshift=7cm]
            \node[dagvar] (X2) at (0.5,1.5) {$X$};
            \node[dagvar] (A2) at (-2.0,0) {$A$};
            \node[dagvar] (Y2) at (2.0,0) {$Y$};
            \node[dagvar,dashed] (U2) at (0,-1.5) {$U$};

            \draw[dagarrow] (X2) -- (Y2);
            \draw[dagarrow] (A2) -- (Y2);
            \draw[dagarrow,dashed] (U2) -- (Y2);

            \node[daglabel] (policy) at (-2.0,1.5)
                {policy\,/\,intervention};
            \draw[dagarrow] (policy.south) -- (A2.north);
            \node[daglabel, right=5pt of U2] {(unobserved)};
        \end{scope}
    \end{tikzpicture}

    \caption{A toy model for illustration of confounded contextual bandits.}
    \label{fig:confounded-contextual-bandit}
\end{figure}
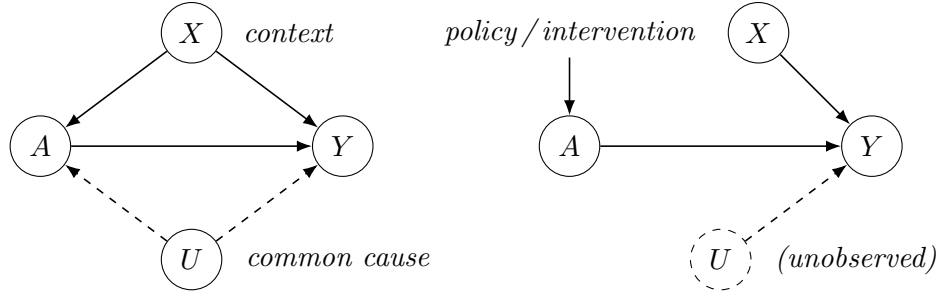

\subsubsection{Observational and Interventional Processes}

We formalize the confounded contextual bandit (CCB) framework, which is inspired by \cite{chen2023unified}, with several modifications tailored to our multi-domain and identification setting. 

In a confounded contextual bandit, each round is described by a collection of random variables $(U,O,X,A,M,Y)$,
where $U\in\mathcal{U}$ collects the unobserved confounders, $O\in\mathcal{O}$ denotes the side observations, $X\in\mathcal{X}$ is the observed context, $A\in\mathcal{A}$ is the action or treatment, $M\in\mathcal{M}$ is a causal mediator between $A$ and $Y$ when such a mediator is present, and $Y\in\mathcal{Y}\subseteq\mathbb{R}$ is the resulting reward.  Among these variables, $X$ is the information available to the policy we ultimately wish to deploy, while $O$ contains auxiliary observations that may be available in the source datasets and will later be used for causal identification.

A key distinction in our setting is between the \textit{observational process}, which generates the offline datasets in the source domains, and the \textit{interventional process}, which describes deployment of the learned policy in the target domain. Figure~\ref{fig:ccb-observational-interventional} illustrates these two processes. In the observational process, the historical action may depend on unobserved confounders and auxiliary observations. In the interventional process, this historical action mechanism is replaced by a policy that depends only on the observed context $X$.

\tikzset{
    dagconf/.style={
        <->,
        dashed,
        semithick
    }
}
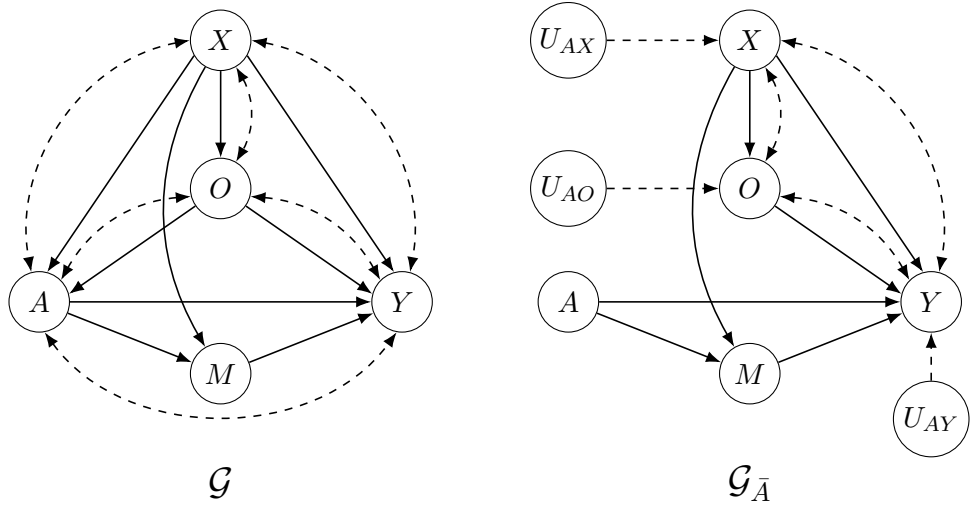
\begin{figure}[t]
    \centering
    \begin{tikzpicture}[>={Latex[length=2mm]}]

        \begin{scope}
            \node[dagvar] (X1) at (0,3.46) {$X$};
            \node[dagvar] (O1) at (0,1.5) {$O$};
            \node[dagvar] (A1) at (-2.4,0) {$A$};
            \node[dagvar] (Y1) at (2.4,0) {$Y$};
            \node[dagvar] (M1) at (0,-0.95) {$M$};

\draw[dagconf]
    (X1.west) to[bend right=45] (A1.105);

\draw[dagconf]
    (X1.east) to[bend left=45] (Y1.75);

\draw[dagconf]
    (X1.300) to[bend left=30] (O1.60);

\draw[dagconf]
    (O1.195) to[bend right=28] (A1.45);

\draw[dagconf]
    (O1.345) to[bend left=28] (Y1.135);

\draw[dagconf]
    (A1.280) to[bend right=54] (Y1.260);

            \draw[dagarrow] (X1) -- (O1);
            \draw[dagarrow] (X1.210) -- (A1.75);
            \draw[dagarrow] (X1.330) -- (Y1.105);

            \draw[dagarrow]
                (X1.240)
                to[bend right=30]
                (M1.120);

            \draw[dagarrow] (O1) -- (A1.15);
            \draw[dagarrow] (O1) -- (Y1.165);
            \draw[dagarrow] (A1) -- (Y1);
            \draw[dagarrow] (A1) -- (M1);
            \draw[dagarrow] (M1) -- (Y1);

            \node at (0,-2.4) {\Large $\mathcal{G}$};
        \end{scope}

        \begin{scope}[xshift=7cm]
            \node[dagvar] (X2) at (0,3.46) {$X$};
            \node[dagvar] (O2) at (0,1.5) {$O$};
            \node[dagvar] (A2) at (-2.4,0) {$A$};
            \node[dagvar] (Y2) at (2.4,0) {$Y$};
            \node[dagvar] (M2) at (0,-0.95) {$M$};

            \node[
                dagvar,
                minimum size=10mm,
                inner sep=1pt
            ] (UAX) at (-2.4,3.46) {$U_{AX}$};

            \node[
                dagvar,
                minimum size=10mm,
                inner sep=1pt
            ] (UAO) at (-2.4,1.5) {$U_{AO}$};

            \node[
                dagvar,
                minimum size=10mm,
                inner sep=1pt
            ] (UAY) at (2.4,-1.55) {$U_{AY}$};

            \draw[dagarrow,dashed] (UAX) -- (X2);
            \draw[dagarrow,dashed] (UAO) -- (O2);
            \draw[dagarrow,dashed] (UAY) -- (Y2);

\draw[dagconf]
    (X2.300) to[bend left=30] (O2.60);

\draw[dagconf]
    (X2.east) to[bend left=45] (Y2.75);

\draw[dagconf]
    (O2.345) to[bend left=28] (Y2.135);

            \draw[dagarrow] (X2) -- (O2);
            \draw[dagarrow] (X2.330) -- (Y2.105);

            \draw[dagarrow]
                (X2.240)
                to[bend right=30]
                (M2.120);

            \draw[dagarrow] (O2) -- (Y2.165);
            \draw[dagarrow] (A2) -- (Y2);
            \draw[dagarrow] (A2) -- (M2);
            \draw[dagarrow] (M2) -- (Y2);

            \node at (0,-2.4) {{\Large $\mathcal{G}_{\bar A}$}};
        \end{scope}
    \end{tikzpicture}

    \caption{A causal diagram describing the observational and interventional processes in a CCB. The left diagram depicts the observational process, where a dashed bidirectional arc represents an unobserved common cause or confounder between two variables. The right diagram is induced by intervening on $A$ in the left diagram.}
    \label{fig:ccb-observational-interventional}
\end{figure}

\paragraph{Observational process.} We use the bracketed superscript $(i)$ to index a source domain $\mathscr{C}^{(i)}$. The observational process in $\mathscr{C}^{(i)}$ describes how the corresponding offline dataset is generated. In each round:
\begin{itemize}
    \item The environment first generates $(U,O,X)$ according to the joint distribution
    $p^{(i)}(u,o,x)$.

    \item Given $(U,O,X)$, the historical agent chooses an action according to the observational policy
    \[
        A\sim p^{(i)}(\cdot\,|\,U,O,X).
    \]
    In particular, the action may depend on the unobserved confounder $U$. After $U$ is omitted from the recorded data, the observational relationship between $A$ and $Y$ can be confounded.

    \item After the action is taken, the mediator $M$, when present, is generated from $p^{(i)}(m|u,o,x,a)$, followed by the reward $Y\sim p^{(i)}(\cdot|u,o,x,a,m)$.
\end{itemize}
Accordingly, the joint distribution of one observational round factorizes as
\begin{equation}
    p^{(i)}(u,o,x,a,m,y)
    =
    p^{(i)}(u,o,x)\,
    p^{(i)}(a|u,o,x)\,
    p^{(i)}(m|u,o,x,a)\,
    p^{(i)}(y|u,o,x,a,m).
    \label{eq:ccb_observational}
\end{equation}
After $T$ rounds, the learner observes an offline dataset
\begin{equation*}
    \mathcal{D}^{(i)}
    =
    \{(o_t,x_t,a_t,m_t,y_t)\}_{t=1}^{T},
\end{equation*}
whereas the corresponding confounders $\{u_t\}_{t=1}^{T}$ are missing. 

\paragraph{Interventional process.}
The interventional process describes deployment in the target domain $\mathscr{C}^*$. Here the historical action mechanism is replaced by an \textit{interventional} policy $\pi:\mathcal{X}\to\Delta(\mathcal{A})$, so that the action is generated according to $A\sim\pi(\cdot\,|\,X)$ and therefore depends only on the observed context $X$. Importantly, the confounders $U$ and side observations $O$ are not inputs to the deployed policy, and the auxiliary variables $O$ and $M$ need not be observed during deployment.

The remaining variables evolve according to the target-domain causal mechanisms. In particular, conditional on $X=x$ and the assigned action $A=a$,
\begin{equation}
    p^*(u,o,m,y\,|\,x,a)
    =
    p^*(u,o\,|\,x)\,
    p^*(m\,|\,u,o,x,a)\,
    p^*(y\,|\,u,o,x,a,m).\label{eq:ccb_interventional_act}
\end{equation}
Hence a policy $\pi$ induces the joint distribution
\begin{equation}
    p^\pi(u,o,x,a,m,y)
    =
    p^*(x)\,
    p^*(u,o\,|\,x)\,
    \pi(a\,|\,x)\,
    p^*(m\,|\,u,o,x,a)\,
    p^*(y\,|\,u,o,x,a,m).
    \label{eq:ccb_interventional}
\end{equation}
As is discussed in \S\ref{sec:2.1}, our policy-learning objective is defined with respect to this interventional process in the target domain. Specifically, the value of $\pi$ is
\begin{equation}
    v^\pi
    :=
    \E_{p^\pi}[Y]
    =
    \int_{\mathcal{U}\times\mathcal{O}\times\mathcal{X}
    \times\mathcal{A}\times\mathcal{M}\times\mathcal{Y}}
    y\,p^\pi(u,o,x,a,m,y)
    \,\mathrm{d}u\,\mathrm{d}o\,\mathrm{d}x\,\mathrm{d}a\,\mathrm{d}m\,\mathrm{d}y.
    \label{eq:ccb_policy_value}
\end{equation}
The distinction between \eqref{eq:ccb_observational} and \eqref{eq:ccb_interventional} is essential: the offline datasets are generated under confounded observational action mechanisms $p^{(i)}(a|u,o,x)$, whereas the policy to be evaluated and optimized acts according to $\pi(a|x)$ in the target domain.

Here we provide two representative instances of this general CCB formulation, which are studied in detail in \S\ref{sec:3}.

\begin{example}
In the \textit{confounded contextual bandit with negative outcome} (CCB-Neg), shown in Figure~\ref{fig:ccb-negative-outcome}, the side observations $O$ include the instrumental variables $Z_1,Z_2$ and a negative outcome $W$, while no mediator is present. On each dashed bidirected arcs, there exists a confounder variable as part of $U$. The observational action mechanism takes the form $p(a\,|\,u_a,x,z_1)$, where $U_A$ denotes the components of $U$ associated with the action.  
\end{example}

\begin{example}
In the \textit{confounded contextual bandit with mediator} (CCB-Med), shown in Figure~\ref{fig:ccb-negative-outcome-iv}, the side observations include $Z_1,Z_2$, while the mediator $M$ is instantiated by a variable $W$ lying causally between $A$ and $Y$. On each dashed bidirected arcs, there exists a confounder variable as part of $U$. The observational policy is specified by a conditional distribution $p(a\,|\,u_a, x, z_1)$, where $U_A$ denotes the confounders associated with $A$.
\end{example}

\subsubsection{Conditional Average Treatment Effect}

The value $v^\pi$ in \eqref{eq:ccb_policy_value} is defined through the full interventional distribution involving the latent confounders and auxiliary variables. For policy learning, however, it is sufficient to identify a lower-dimensional causal quantity: the \textit{conditional average treatment effect} (CATE) of the action on the reward in the target domain. For each $(a,x)\in\mathcal{A}\times\mathcal{X}$, define
\begin{equation}
    g^*(a,x)
    :=
    \E^*\!\left[
        Y
        \,\middle|\,
        X=x,\operatorname{do}(A=a)
    \right]
    =
    \int_{\mathcal{Y}}
    y\,p^*\!\left(
        y
        \mid
        x,\operatorname{do}(a)
    \right)\mathrm{d}y.
    \label{eq:target_cate}
\end{equation}
Unlike the observational conditional mean $\E^*[Y|X=x,A=a]$, the quantity $g^*(a,x)$ measures the expected reward when the action is externally assigned to $a$, thereby removing the dependence of the action on latent confounders. Under the target-domain causal model in \eqref{eq:ccb_interventional_act}, the same quantity can be written explicitly as
\begin{equation*}
\begin{aligned}
g^*(a,x&)=
\int_{\mathcal{U}\times\mathcal{O}\times\mathcal{M}\times\mathcal{Y}}
y\,p^*(u,o\,|\,\operatorname{do}(a),x)\,p^*(m\,|\,u,o,x,\operatorname{do}(a))p^*(y\,|\,u,o,x,\operatorname{do}(a),m)\,\mathrm{d}u\,\mathrm{d}o\,\mathrm{d}m\,\mathrm{d}y\\
&=\int_{\mathcal{U}\times\mathcal{O}\times\mathcal{M}\times\mathcal{Y}}y\,p^*(u,o\,|\,x)\,p^*(m\,|\,u,o,x,a)p^*(y\,|\,u,o,x,a,m)\,\mathrm{d}u\,\mathrm{d}o\,\mathrm{d}m\,\mathrm{d}y.
\end{aligned}
\end{equation*}
where the second equality follows from $(U,O\indep A\,|\,X)$ in $\mathcal{G}_{\overline{A}}$, $(M\indep A\,|\,U,O,X)$ in $\mathcal{G}_{\underline{A}}$, and $(Y\indep A\,|\,U,O,X,M)$ in $\mathcal{G}_{\underline{A}}$. Thus, $g^*(a,x)$ averages the reward over the latent and auxiliary variables according to their target-domain distribution after action $a$ is assigned.

Most importantly, the value of any policy $\pi$ can be expressed entirely through CATE $g^*$. Indeed,
\begin{equation}
\begin{aligned}
    \E_{p^\pi}[g^*(A,X)]
    &=
    \int_{\mathcal{X}\times\mathcal{A}}
    g^*(a,x)\,
    p^*(x)\pi(a|x)
    \,\mathrm{d}x\,\mathrm{d}a
    \\
    &=
    \int_{\mathcal{U}\times\mathcal{O}\times\mathcal{X}
    \times\mathcal{A}\times\mathcal{M}\times\mathcal{Y}}
    y\,
    p^*(x)\,
    p^*(u,o|x)\,
    \pi(a|x)\,
    p^*(m|u,o,x,a)\\
    &\hspace{5cm}
    \times p^*(y|u,o,x,a,m)\,\mathrm{d}u\,\mathrm{d}o\,\mathrm{d}x\,\mathrm{d}a\,\mathrm{d}m\,\mathrm{d}y
    \\
    &=
    \E_{p^\pi}[Y]
    =
    v^\pi.
\end{aligned}\label{eq:cate_policy_value}
\end{equation}
Therefore,
\begin{equation*}
    v^\pi=\E_{p^\pi}[g^*(A,X)].
\end{equation*}
This identity reduces the policy-learning problem to identifying the target-domain CATE $g^*$. Once $g^*(a,x)$ is known, the value of any context-dependent policy can be evaluated without explicitly recovering the unobserved confounders.

The central difficulty is that $g^*$ is an interventional quantity, whereas the available datasets are observational and are collected from heterogeneous source domains. In the presence of both unmeasured confounding \citep{VanderWeele_Hernan_Robins_2008} and domain discrepancy \citep{Lee_Correa_Bareinboim_2020}, $g^*$ therefore cannot in general be obtained by ordinary regression on any single source dataset. In the next section, we develop identification strategies that combine transportability with instrumental- and proxy-variable methods to recover $g^*$ from multiple heterogeneous observational datasets.

%% file: chap3.tex
\section{Causal Identification of CATE}
\label{sec:3}
In \S\ref{sec:2}, we showed that offline policy learning in a confounded contextual bandit can be reduced to identifying the target-domain CATE $g^*(a,x)$. We now study how this causal quantity can be recovered from multiple observational datasets collected from heterogeneous source domains. The identification problem has two intertwined challenges: 
\begin{itemize}
    \item First, due to domain discrepancy, the target causal effect may not be directly transportable from any single source domain; instead, different components of the target causal mechanism may need to be assembled from different sources. 
    \item Second, even after such a transportability decomposition is obtained, the resulting causal quantities may still involve interventions and therefore cannot generally be recovered from observational conditional distributions in the presence of unmeasured confounding.
\end{itemize}
Our approach addresses these two challenges in sequence: we first use transportability and do-calculus to decompose the target CATE into causal quantities associated with different source domains, and then employ instrumental- and proxy-variable techniques to identify the remaining interventional quantities through bridge functions and conditional moment equations. We develop this procedure for CCBs with negative outcomes in \S\ref{sec:3.1} and with mediators in \S\ref{sec:3.2}, before presenting a more general identification strategy in \S\ref{sec:3.3}.

\subsection{Identification in CCB with Negative Outcome}
\label{sec:3.1}

We first consider the CCB with a negative outcome (CCB-Neg) introduced in \S\ref{sec:2.3}. Recall that the negative outcome $W$ is not causally affected by the action $A$, although $A$ and $W$ may still be statistically associated through unobserved common causes \citep{miao2018identifying, miao2020confounding}. This structure allows us to decompose the target CATE into two components that can be transported from different source domains. Following the general strategy outlined above, we first derive this transportability decomposition and then identify the remaining source-domain interventional quantity using instrumental variables.

\begin{figure}
    \centering
    \begin{tikzpicture}[>={Latex[length=2mm]}]
    \path[use as bounding box]
        (-2.45,-2.45) rectangle (10.6,4.0);
        \begin{scope}
            \node[dagvar] (X1) at (0,2.2) {$X$};
            \node[dagvar] (W1) at (3.2,2.2) {$W$};
            \node[dagvar] (A1) at (-1.6,0) {$A$};
            \node[dagvar] (Y1) at (1.6,0) {$Y$};

            \node[dagsel] (SX1) at (0,3.5) {$S_X$};
            \node[dagsel] (SW1) at (3.2,3.5) {$S_W$};
            \node[dagsel] (SA1) at (-1.6,-1.3) {$S_A$};
            \node[dagsel] (SY1) at (1.6,-1.3) {$S_Y$};

            \draw[dagarrow] (SX1) -- (X1);
            \draw[dagarrow] (SW1) -- (W1);
            \draw[dagarrow] (SA1) -- (A1);
            \draw[dagarrow] (SY1) -- (Y1);

            \draw[dagconf]
                (X1.210)
                to[bend right=36]
                (A1.90);

            \draw[dagconf]
                (X1.330)
                to[bend left=36]
                (Y1.90);

            \draw[dagconf]
                (A1.300)
                to[bend right=48]
                (Y1.240);

            \draw[dagconf]
    (A1.120)
    to[out=113,in=113,looseness=1.80]
    (W1.125);

            \draw[dagarrow] (X1) -- (W1);
            \draw[dagarrow] (X1.240) -- (A1.60);
            \draw[dagarrow] (X1.300) -- (Y1.120);
            \draw[dagarrow] (W1.240) -- (Y1.60);
            \draw[dagarrow] (A1) -- (Y1);

            \node at (0.55,-2.25) {\large $\mathcal{G}$};
        \end{scope}

        \begin{scope}[xshift=7.4cm]
            \node[dagvar] (X2) at (0,2.2) {$X$};
            \node[dagvar] (W2) at (3.2,2.2) {$W$};
            \node[dagvar] (A2) at (-1.6,0) {$A$};
            \node[dagvar] (Y2) at (1.6,0) {$Y$};

            \node[dagsel] (SX2) at (0,3.5) {$S_X$};
            \node[dagsel] (SW2) at (3.2,3.5) {$S_W$};
            \node[dagsel] (SA2) at (-1.6,-1.3) {$S_A$};
            \node[dagsel] (SY2) at (1.6,-1.3) {$S_Y$};

            \draw[dagarrow] (SX2) -- (X2);
            \draw[dagarrow] (SW2) -- (W2);
            \draw[dagarrow] (SY2) -- (Y2);


            \draw[dagconf]
                (X2.330)
                to[bend left=36]
                (Y2.90);

            \draw[dagarrow] (X2) -- (W2);
            \draw[dagarrow] (X2.300) -- (Y2.120);
            \draw[dagarrow] (W2.240) -- (Y2.60);
            \draw[dagarrow] (A2) -- (Y2);

            \node at (0.55,-2.25) {\large $\mathcal{G}_{\bar A}$};
        \end{scope}
    \end{tikzpicture}
    \caption{Causal diagrams of a CCB with a negative outcome. Selection nodes are represented by square vertices and indicate possible discrepancies between the source and target domains. The right diagram is induced by intervening on $A$ in the left diagram.}
    \label{fig:ccb-negative-outcome}
\end{figure}
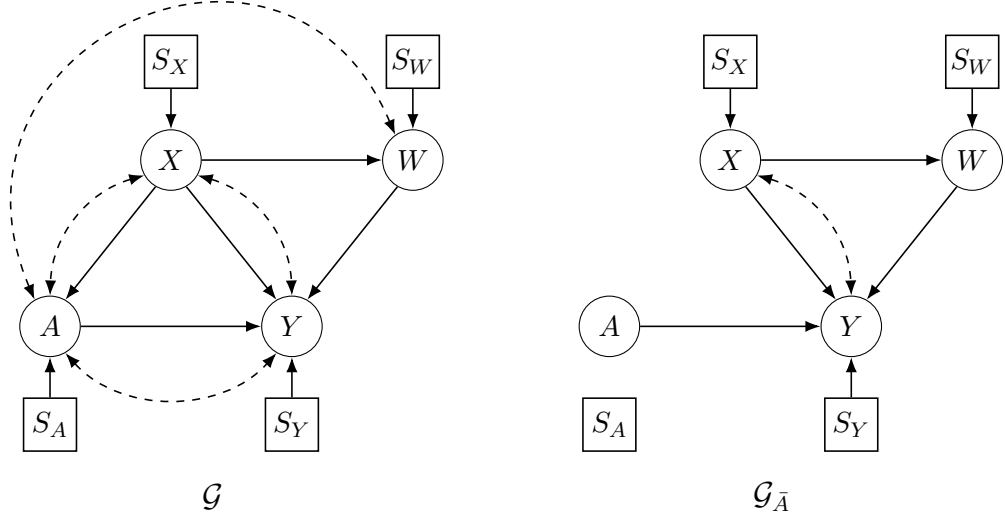

\paragraph{Transportability decomposition.}
We begin with the reduced causal structure shown in Figure~\ref{fig:ccb-negative-outcome}, where the square selection nodes indicate mechanisms that may differ between the source and target domains. Using the law of total probability and the rules of do-calculus \citep{judea1995}, we decompose the target-domain interventional distribution as
\begin{align}
	p^*(y|x,\mathrm{do}(a)) &= \int_\mathcal{W} p^*(y|w,x,\mathrm{do}(a)) p^*(w|x,\mathrm{do}(a))\mathrm{d}w\notag\\
	&= \int_\mathcal{W} p^*(y|x,\mathrm{do}(w),\mathrm{do}(a)) p^*(w|x,\mathrm{do}(a))\mathrm{d}w &(\text{by}\ Y\indep W\,|\,X,A\ \text{in}\ \mathcal{G}_{\overline{A}\underline{W}})\notag\\
	&= \int_\mathcal{W} p^*(y|x,\mathrm{do}(w),\mathrm{do}(a)) p^*(w|x)\mathrm{d}w &(\text{by}\ W\indep A\,|\,X\ \text{in}\ \mathcal{G}_{\overline{A(X)}})\notag\\
	&= \int_\mathcal{W} p^{(1)}(y|x,\mathrm{do}(w),\mathrm{do}(a)) p^{(2)}(w|x)\mathrm{d}w. &(\text{Transferring step})\label{eq:3.1}
\end{align}
The first equality introduces the negative outcome $W$. The second exchanges the observation of $W$ with an intervention on $W$, while the third uses the fact that $W$ is not causally affected by $A$ to remove the intervention on $A$ from its conditional distribution. The final equality is the transportability step: the first source domain supplies the target response mechanism under interventions on $(W,A)$, whereas the second source domain supplies the target conditional distribution of $W$.

Specifically, the following S-admissibility conditions are sufficient for the two transported components:
\begin{align*}
&\Delta^{(1)}\subseteq\{A,W\}\ \Rightarrow\  p^*(y\,|\,x,\mathrm{do}(w),\mathrm{do}(a)) = p^{(1)}(y\,|\,x,\mathrm{do}(w),\mathrm{do}(a)), &(\text{by}\ Y\indep S_A,S_W\,|\,X\ \text{in}\ \mathcal{G}_{\overline{AW}})\\
&\Delta^{(2)}\subseteq\{A,X,Y\}\ \Rightarrow\ p^*(w\,|\,x) = p^{(2)}(w\,|\,x). &(\text{by}\ W\indep S_A,S_X,S_Y\,|\,X\ \text{in}\ \mathcal{G})
\end{align*}
Thus, neither source domain is required to reproduce the entire target mechanism. Instead, each source provides one component needed to reconstruct the target causal effect.

\subsubsection{Identification with Instrumental Variables}

In the transferred expression \eqref{eq:3.1}, the second component $p^{(2)}(w|x)$ contains no intervention and is therefore directly estimable from observational data in source domain $\mathscr{C}^{(2)}$. Hence the remaining difficulty is the interventional quantity contributed by domain $\mathscr{C}^{(1)}$. To handle this, we define
\begin{equation}
    g^{(1)}(w,a,x)
    :=
    \E^{(1)}
    \left[
        Y
        \,\middle|\,
        X=x,\operatorname{do}(W=w,A=a)
    \right]
    =
    \int_{\mathcal{Y}}
    y\,
    p^{(1)}
    \left(
        y
        \mid
        x,\operatorname{do}(w),\operatorname{do}(a)
    \right)
    \mathrm{d}y.
    \label{eq:g1_neg}
\end{equation}
Because the observational process in $\mathscr{C}^{(1)}$ is confounded, $g^{(1)}$ cannot in general be recovered by ordinary conditional regression on $(W,A,X)$.

To identify $g^{(1)}$, we introduce instrumental variables $Z_1$ and $Z_2$ for the endogenous variables $(A,W)$. For brevity, let
\[
    U=\{U_{AX},U_{AY},U_{XY},U_{AW}\}
\]
collect the unobserved confounders appearing in the causal model. The resulting observational structure is illustrated in Figure~\ref{fig:ccb-negative-outcome-iv}.

\begin{assumption}[Model assumptions for the observational process in CCB-Neg]
\label{assumption:3.1}
Suppose that the observational processes of the CCB-Neg in the source domains satisfy the following conditions:
\begin{itemize}
    \item[(i)] \textbf{Structural reward.}
    The reward follows
    \[
        Y=f(W,A,X)+\epsilon,
        \qquad
        \epsilon\indep(W,A)\,|\,U,X.
    \]

    \item[(ii)] \textbf{IV independence.}
    The instrumental variables satisfy
    \[
        (Z_1,Z_2)\indep(U,\epsilon)\,|\,X.
    \]

    \item[(iii)] \textbf{IV completeness.}
    For any
    $\sigma\in L_2(\mathcal{W}\times\mathcal{A}\times\mathcal{X})$,
    \[
        \E^{(1)}
        [\sigma(W,A,X)\mid Z_1,Z_2]
        =0\quad\mathrm{a.s.}
        \quad\Longleftrightarrow\quad
        \sigma(W,A,X)=0\quad\mathrm{a.s.}
    \]

    \item[(iv)] \textbf{Domain discrepancy.}
    We have two source domains $d^{(1)}$ and $d^{(2)}$ and a target domain $d^*$, with
    \[
        \Delta^{(1)}
        \subseteq
        \{A,W,Z_1,Z_2\},
        \qquad
        \Delta^{(2)}
        \subseteq
        \{A,X,Y,Z_1,Z_2\}.
    \]
\end{itemize}
\end{assumption}

Condition (i) gives a structural regression representation of the reward. Condition (ii) requires the instruments to be independent of the latent confounders and structural noise after conditioning on $X$. Condition (iii) is a completeness condition ensuring that the conditional moment operator induced by $(Z_1,Z_2)$ is sufficiently informative to distinguish functions of $(W,A,X)$. These are standard types of assumptions in instrumental-variable identification \citep{Baiocchi2014,chen2023unified}. Notice that Assumption~\ref{assumption:3.1}(iv) additionally permits the mechanisms associated with $Z_1$ and $Z_2$ to vary across domains; these variables serve as auxiliary variables for within-source identification and do not appear in the transported decomposition \eqref{eq:3.1}.

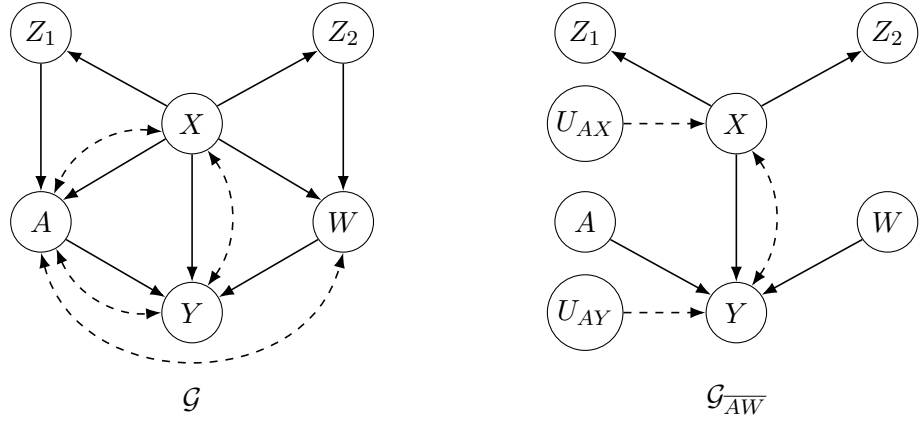
\begin{figure}[t]
    \centering
    \begin{tikzpicture}[>={Latex[length=2mm]}]
        \path[use as bounding box]
            (-2.7,-2.6) rectangle (10.0,3.2);

        \begin{scope}
            \node[dagvar] (Z11) at (-2.0,2.5) {$Z_1$};
            \node[dagvar] (Z21) at ( 2.0,2.5) {$Z_2$};

            \node[dagvar] (X1) at (0,1.3) {$X$};
            \node[dagvar] (A1) at (-2.0,0) {$A$};
            \node[dagvar] (W1) at ( 2.0,0) {$W$};
            \node[dagvar] (Y1) at (0,-1.2) {$Y$};

            \draw[dagarrow] (X1.145) -- (Z11.325);
            \draw[dagarrow] (X1.35)  -- (Z21.215);
            \draw[dagarrow] (Z11) -- (A1);
            \draw[dagarrow] (Z21) -- (W1);

            \draw[dagarrow] (X1.210) -- (A1.45);
            \draw[dagarrow] (X1.330) -- (W1.135);
            \draw[dagarrow] (X1) -- (Y1);
            \draw[dagarrow] (A1.325) -- (Y1.150);
            \draw[dagarrow] (W1.215) -- (Y1.30);

            \draw[dagconf]
                (X1.195)
                to[bend right=36]
                (A1.65);

            \draw[dagconf]
                (X1.300)
                to[bend left=36]
                (Y1.60);

            \draw[dagconf]
                (A1.300)
                to[bend right=36]
                (Y1.180);

            \draw[dagconf]
                (A1.270)
                to[out=-75,in=-105,looseness=1.25]
                (W1.270);

            \node at (0,-2.35) {$\mathcal{G}$};
        \end{scope}

        \begin{scope}[xshift=7.2cm]
            \node[dagvar] (Z12) at (-2.0,2.5) {$Z_1$};
            \node[dagvar] (Z22) at ( 2.0,2.5) {$Z_2$};

            \node[dagvar] (X2) at (0,1.3) {$X$};
            \node[dagvar] (A2) at (-2.0,0) {$A$};
            \node[dagvar] (W2) at ( 2.0,0) {$W$};
            \node[dagvar] (Y2) at (0,-1.2) {$Y$};

            \node[
                dagvar,
                minimum size=10mm,
                inner sep=1pt
            ] (UAX) at (-2.0,1.3) {$U_{AX}$};

            \node[
                dagvar,
                minimum size=10mm,
                inner sep=1pt
            ] (UAY) at (-2.0,-1.2) {$U_{AY}$};

            \draw[dagarrow,dashed] (UAX) -- (X2);
            \draw[dagarrow,dashed] (UAY) -- (Y2);

            \draw[dagarrow] (X2.145) -- (Z12.325);
            \draw[dagarrow] (X2.35)  -- (Z22.215);

            \draw[dagarrow] (X2) -- (Y2);
            \draw[dagarrow] (A2.325) -- (Y2.145);
            \draw[dagarrow] (W2.215) -- (Y2.35);

            \draw[dagconf]
                (X2.300)
                to[bend left=36]
                (Y2.60);

            \node at (0,-2.35)
                {$\mathcal{G}_{\overline{AW}}$};
        \end{scope}
    \end{tikzpicture}

    \caption{Causal diagrams of a CCB with a negative outcome and instrumental variables. The right diagram is induced by intervening on $A$ in the left diagram.}
    \label{fig:ccb-negative-outcome-iv}
\end{figure}

Under these conditions, the interventional mean $g^{(1)}$ can be characterized entirely through an observational conditional moment equation.

\begin{lemma}[IV identification in CCB-Neg]
\label{lemma:3.1}
Suppose Assumption~\ref{assumption:3.1} holds. For the observational process in source domain $\mathscr{C}^{(1)}$, if a function
$h_1\in L_2(\mathcal{W}\times\mathcal{A}\times\mathcal{X})$
satisfies
\begin{equation}
    \E^{(1)}
    \left[
        h_1(W,A,X)-Y
        \,\middle|\,
        Z_1,Z_2
    \right]
    =0
    \qquad\mathrm{a.s.},
    \label{eq:h1_neg_moment}
\end{equation}
then
\[
    h_1(W,A,X)
    =
    g^{(1)}(W,A,X)
    \qquad\mathrm{a.s.}
\]
\end{lemma}

\begin{proof}
See \S\ref{sec:A.1} for a detailed proof.
\end{proof}

Lemma~\ref{lemma:3.1} resolves the confounding problem within the first source domain: although $g^{(1)}$ is defined through interventions, it is identified by the observational conditional moment equation \eqref{eq:h1_neg_moment}. Combining this result with the transportability decomposition \eqref{eq:3.1} gives
\begin{align}
    g^*(a,x)
    &=
    \int_{\mathcal{Y}}
    y\,p^*(y|x,\operatorname{do}(a))
    \,\mathrm{d}y
    \notag\\
    &=
    \int_{\mathcal{W}\times\mathcal{Y}}
    y\,
    p^{(1)}
    \left(
        y
        \mid
        x,\operatorname{do}(w),\operatorname{do}(a)
    \right)
    p^{(2)}(w|x)
    \,\mathrm{d}y\,\mathrm{d}w
    \notag\\
    &=
    \int_{\mathcal{W}}
    h_1(w,a,x)p^{(2)}(w|x)
    \,\mathrm{d}w
    \notag\\
    &=
    \E^{(2)}
    [h_1(W,a,x)\mid X=x].
    \label{eq:gstar_neg}
\end{align}
Therefore, the target CATE is identified by combining complementary information from the two observational source domains: $\mathscr{C}^{(1)}$ identifies the causal response surface through the IV bridge $h_1$, while $\mathscr{C}^{(2)}$ supplies the target-compatible distribution of the negative outcome $W$. In particular, neither source-domain CATE is required to be directly transportable to $\mathscr{C}^*$.

\subsubsection{Conditional Moment Formulation}
\label{sec:3.1.2}

The preceding identification result can be summarized as an integral equation system (IES) for the pair of bridge functions $(h_1,g)$:
\begin{equation}
\begin{cases}
    \E^{(1)}
    [h_1(W,A,X)-Y\mid Z_1,Z_2]
    =0,
    \\[2mm]
    \E^{(2)}
    [g(a,X)-h_1(W,a,X)\mid X]
    =0,
    \qquad a\in\mathcal{A}.
\end{cases}
\label{eq:neg_ies_pointwise}
\end{equation}
Under the identification conditions, any solution $(h_1,g)$ to this system satisfies
$g=g^*$ almost surely.

\paragraph{Augmented dataset.} The second equation in \eqref{eq:neg_ies_pointwise} is indexed pointwise by the action $a\in\mathcal{A}$, so the system involves $|\mathcal{A}|+1$ conditional moment equations (or infinitely many when $\mathcal{A}$ is continuous). This makes both theoretical analysis and statistical estimation inconvenient, because one would need to enforce a separate conditional restriction for each action value.

For later estimation and analysis, it is convenient to express the entire system in the common form of conditional moment equations. To this end, we introduce a pseudo-action $A'$ that is independent of the CCB process and follows a fixed reference distribution
\[
    A'\sim\nu_A,
    \qquad
    \nu_A\in\Delta(\mathcal{A}),
\]
whose support covers the action space of interest. When $\mathcal{A}$ is finite, $\nu_A$ may simply be chosen as the uniform distribution. The second equation in \eqref{eq:neg_ies_pointwise} can then be represented as
\begin{equation*}
    \E^{(2)}
    [g(A',X)-h_1(W,A',X)\mid A',X]
    =0.
\end{equation*}
Operationally, the dataset from $\mathscr{C}^{(2)}$ can be augmented by independently drawing a pseudo-action $a_t'$ from $\nu_A$ for each observation. The resulting identification equations take the unified conditional-moment form
\begin{equation}
\begin{cases}
    \E^{(1)}
    [h_1(W,A,X)-Y\mid Z_1,Z_2]
    =0,
    \\[2mm]
    \E^{(2)}
    [g(A',X)-h_1(W,A',X)\mid A',X]
    =0.
\end{cases}
\label{eq:ccb_neg_ies}
\end{equation}
This representation allows us to treat the bridge equations from different source domains within a common framework and, in \S\ref{sec:4}, convert the conditional moments into tractable minimax estimation objectives.

\subsection{Identification in CCB with Mediation}
\label{sec:3.2}

We next consider the CCB with a mediator (CCB-Med), where $W$ lies on a causal path from the action $A$ to the reward $Y$. Such a structure also arises naturally in longitudinal studies, where the action and mediator can be interpreted as primary and secondary treatments \citep{Bind2016med}. As in \S\ref{sec:3.1}, we first transport complementary causal mechanisms from heterogeneous source domains and then identify the remaining interventional quantities from observational data. The main distinction from CCB-Neg is that, because $W$ is now affected by $A$, the transported distribution of $W$ is itself interventional and requires an additional confounding adjustment.

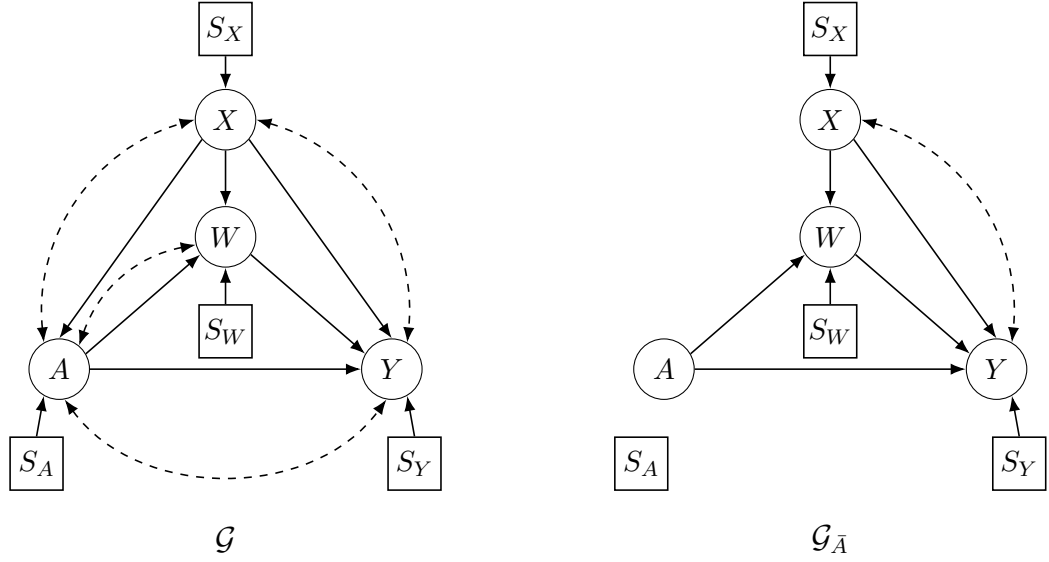
\begin{figure}[t]
    \centering
    \begin{tikzpicture}[>={Latex[length=2mm]}]
        \path[use as bounding box]
            (-3.5,-2.5) rectangle (11.0,4.5);

        \begin{scope}
            \node[dagvar] (X1) at (0,3.3) {$X$};
            \node[dagvar] (W1) at (0,1.75) {$W$};
            \node[dagvar] (A1) at (-2.2,0) {$A$};
            \node[dagvar] (Y1) at ( 2.2,0) {$Y$};

            \node[dagsel] (SX1) at (0,4.5) {$S_X$};
            \node[dagsel] (SW1) at (0,0.50) {$S_W$};
            \node[dagsel] (SA1) at (-2.5,-1.25) {$S_A$};
            \node[dagsel] (SY1) at ( 2.5,-1.25) {$S_Y$};

            \draw[dagarrow] (SX1) -- (X1);
            \draw[dagarrow] (SW1) -- (W1);
            \draw[dagarrow] (SA1.north) -- (A1.240);
            \draw[dagarrow] (SY1.north) -- (Y1.300);

            \draw[dagconf]
                (A1.120)
                to[bend left=40]
                (X1.180);

            \draw[dagconf]
                (X1.0)
                to[bend left=40]
                (Y1.60);

            \draw[dagconf]
                (A1.50)
                to[bend left=30]
                (W1.195);

            \draw[dagconf]
                (A1.280)
                to[bend right=54]
                (Y1.260);

            \draw[dagarrow] (X1) -- (W1);
            \draw[dagarrow] (X1.220) -- (A1.90);
            \draw[dagarrow] (X1.320) -- (Y1.90);
            \draw[dagarrow] (A1.30) -- (W1.215);
            \draw[dagarrow] (W1.325) -- (Y1.150);
            \draw[dagarrow] (A1) -- (Y1);

            \node at (0,-2.25) {\large $\mathcal{G}$};
        \end{scope}

        \begin{scope}[xshift=8.0cm]
            \node[dagvar] (X2) at (0,3.3) {$X$};
            \node[dagvar] (W2) at (0,1.75) {$W$};
            \node[dagvar] (A2) at (-2.2,0) {$A$};
            \node[dagvar] (Y2) at ( 2.2,0) {$Y$};

            \node[dagsel] (SX2) at (0,4.5) {$S_X$};
            \node[dagsel] (SW2) at (0,0.50) {$S_W$};
            \node[dagsel] (SA2) at (-2.5,-1.25) {$S_A$};
            \node[dagsel] (SY2) at ( 2.5,-1.25) {$S_Y$};

            \draw[dagarrow] (SX2) -- (X2);
            \draw[dagarrow] (SW2) -- (W2);
            \draw[dagarrow] (SY2.north) -- (Y2.300);


            \draw[dagconf]
                (X2.0)
                to[bend left=40]
                (Y2.60);

            \draw[dagarrow] (X2) -- (W2);
            \draw[dagarrow] (X2.320) -- (Y2.90);
            \draw[dagarrow] (A2.30) -- (W2.215);
            \draw[dagarrow] (W2.325) -- (Y2.150);
            \draw[dagarrow] (A2) -- (Y2);

            \node at (0,-2.25) {\large $\mathcal{G}_{\bar A}$};
        \end{scope}
    \end{tikzpicture}

    \caption{Causal diagrams of a CCB with a mediator.
    Selection nodes are represented by square vertices and indicate discrepancies between the target and source domains. The right diagram is induced by intervening on $A$ in the left diagram.}
    \label{fig:ccb-mediator}
\end{figure}

\paragraph{Transportability decomposition.}
The causal structure of CCB-Med is shown in Figure~\ref{fig:ccb-mediator}. Following \cite{pmlr-v31-bareinboim13a}, we apply the law of total probability and do-calculus to decompose the interventional distribution of interest as
\begin{equation}
\begin{aligned}
p^*(y|x,\mathrm{do}(a)) &= 
\int_\mathcal{W} p^*(y|w,x,\mathrm{do}(a)) p^*(w|x,\mathrm{do}(a))\mathrm{d}w\\
&= \int_\mathcal{W} p^*(y|x,\mathrm{do}(w),\mathrm{do}(a)) p^*(w|x,\mathrm{do}(a))\mathrm{d}w &(\text{By}\ Y\indep W\,|\,X,A\ \text{in}\ \mathcal{G}_{\overline{A}\underline{W}})\\
&= \int_\mathcal{W} p^{(1)}(y|x,\mathrm{do}(w),\mathrm{do}(a)) p^{(2)}(w|x,\mathrm{do}(a))\mathrm{d}w. &(\text{Transferring step}) 
\end{aligned}\label{eq:med_transport}
\end{equation}

The last equality transports the two components from different source domains. Sufficient S-admissibility conditions are:
\begin{align*}
&\text{By}\ Y\indep S_A,S_W\,|\,X\ \text{in}\ \mathcal{G}_{\overline{AW}},\quad \Delta^{(1)}\subseteq\{A,W\}\ \Rightarrow\  p^*(y|x,\mathrm{do}(w),\mathrm{do}(a)) = p^{(1)}(y|x,\mathrm{do}(w),\mathrm{do}(a)),\\
&\text{By}\ W\indep S_A,S_X,S_Y\,|\,X\ \text{in}\ \mathcal{G}_{\overline{A}},\quad \Delta^{(2)}\subseteq\{A,X,Y\}\ \Rightarrow\ p^*(w|x,\mathrm{do}(a)) = p^{(2)}(w|x,\mathrm{do}(a)).
\end{align*}
This decomposition again shows why multiple source domains can be necessary. For comparison, if a source domain $\mathscr{C}^{(3)}$ differs from the target only in the action mechanism,
\[
    \Delta^{(3)}\subseteq\{A\},
\]
then intervention on $A$ removes the relevant discrepancy and the entire CATE is directly transportable:
\[
    p^*(y\,|\,x,\operatorname{do}(a))
    =
    p^{(3)}(y\,|\,x,\operatorname{do}(a)).
\]
In the setting of interest, however, we may have, for example,
\[
    \Delta^{(1)}=\{A,W\},
    \qquad
    \Delta^{(2)}=\{A,Y\},
\]
so that the target CATE is directly transportable from neither source. Nevertheless, \eqref{eq:med_transport} shows that the two sources provide complementary components that together identify the target effect.

In addition, if we assume that $X$ and $Y$ are unconfounded, so that the bidirected edge between them is absent, then the first discrepancy condition can be relaxed to
\[
    \Delta^{(1)}\subseteq\{A,W,X\},
\]
since the selection node associated with $X$ can then be d-separated from $Y$ by conditioning on $X$.

\subsubsection{Identification with Instrumental and Proxy Variables}

After the transportability decomposition, both components in \eqref{eq:med_transport} still involve interventions. Define the source-domain causal response
\begin{equation}
    g^{(1)}(w,a,x)
    :=
    \E^{(1)}
    \left[
        Y|X=x,\operatorname{do}(W=w,A=a)
    \right]
    =
    \int_{\mathcal{Y}}
    y\,
    p^{(1)}
    \left(
        y \mid x,\operatorname{do}(w),\operatorname{do}(a)
    \right)
    \mathrm{d}y.
\label{eq:g1_med}
\end{equation}
Then the transported target CATE can be written as
\begin{equation}
    g^*(a,x)
    =\E\left[g^{(1)}(W,a,x)\Big|X=x,\operatorname{do}(A=a)\right]=
    \int_{\mathcal{W}}
    g^{(1)}(w,a,x)\,
    p^{(2)}(w|x,\operatorname{do}(a))
    \,\mathrm{d}w.
\label{eq:gstar_med_transport}
\end{equation}
The first quantity $g^{(1)}$ can be identified using the same IV argument as in \S\ref{sec:3.1}. The second step is more involved: because $W$ is a mediator affected by $A$, the distribution $p^{(2)}(w|x,\operatorname{do}(a))$ is generally different from the observational distribution $p^{(2)}(w|x,a)$ in the presence of confounding.Therefore, we use proxy variables to identify this second causal component without explicitly estimating the interventional mediator distribution. For brevity, let
\[
    U=\{U_{AX},U_{AY},U_{XY},U_{AW}\}
\]
collect the unobserved confounders. We impose the following conditions.

\begin{assumption}[Model assumptions for the observational process in CCB-Med]
\label{assumption:3.2}
Suppose that the observational processes of the CCB-Med satisfy:
\begin{itemize}
    \item[(i)] \textbf{Structural reward.}
    \[
        Y=f(W,A,X)+\epsilon,
        \quad
        \epsilon\indep(W,A)\,|\,U,X.
    \]

    \item[(ii)] \textbf{IV and PV independence.}
    \[
        (Z_1,Z_2)\indep(U,\epsilon)\,|\,X,
        \qquad
        Z_1\indep W\,|\,(U,A,X),
        \qquad
        Z_2\indep(Z_1,A)\,|\,X.
    \]

    \item[(iii)] \textbf{IV completeness.}
    For any
    $\sigma_1\in L_2(\mathcal{W}\times\mathcal{A}\times\mathcal{X})$,
    \[
        \E^{(1)}
        [\sigma_1(W,A,X)\mid Z_1,Z_2]
        =0\quad\mathrm{a.s.}
        \quad\Longleftrightarrow\quad
        \sigma_1(W,A,X)=0\quad\mathrm{a.s.}
    \]

    \item[(iv)] \textbf{PV completeness.}
    For any
    $\sigma_2\in L_2(\mathcal{U}\times\mathcal{A}\times\mathcal{X})$,
    \[
        \E^{(2)}
        [\sigma_2(U,A,X)\mid Z_1,A,X]
        =0\quad\mathrm{a.s.}
        \quad\Longleftrightarrow\quad
        \sigma_2(U,A,X)=0\quad\mathrm{a.s.}
    \]

    \item[(v)] \textbf{Domain discrepancy.}
    We have two source domains $d^{(1)},d^{(2)}$ and a target domain $d^*$ satisfying
    \[
        \Delta^{(1)}
        \subseteq
        \{A,W,Z_1,Z_2\},
        \qquad
        \Delta^{(2)}
        \subseteq
        \{A,X,Y,Z_1,Z_2\}.
    \]
\end{itemize}
\end{assumption}

Here conditions (i) and (iii) play the same roles as in Assumption~\ref{assumption:3.1}. The additional independence and completeness conditions in (ii) and (iv) allow $Z_1$ and $Z_2$ to serve as complementary proxy variables for adjusting the confounding in the mediator mechanism. These conditions are standard in proximal causal learning \citep{miao2018identifying,miao2020confounding,chen2023unified}. The corresponding causal structure is shown in Figure~\ref{fig:ccb-mediator-proxies}.

\begin{figure}[t]
    \centering
    \resizebox{\linewidth}{!}{%
    \begin{tikzpicture}[>={Latex[length=2mm]}]
        \path[use as bounding box]
            (-2.85,-2.55) rectangle (16.45,4.45);

        \begin{scope}
            \node[dagvar] (X1)  at (0,4.0) {$X$};
            \node[dagvar] (Z11) at (-2.0,2.40) {$Z_1$};
            \node[dagvar] (Z21) at (0.5,2.35) {$Z_2$};
            \node[dagvar] (W1)  at (0,0.85) {$W$};
            \node[dagvar] (A1)  at (-2.4,0) {$A$};
            \node[dagvar] (Y1)  at (2.4,0) {$Y$};

            \draw[dagarrow] (X1.215) -- (Z11.35);
            \draw[dagarrow] (Z11) -- (A1);
            \draw[dagarrow] (X1.285) -- (Z21.105);
            \draw[dagarrow] (Z21.255) -- (W1.75);

            \draw[dagarrow] (X1) -- (W1);
            \draw[dagarrow] (X1.235) -- (A1.60);
            \draw[dagarrow] (X1.315) -- (Y1.120);
            \draw[dagarrow] (A1.20) -- (W1.210);
            \draw[dagarrow] (W1.330) -- (Y1.160);
            \draw[dagarrow] (A1) -- (Y1);

            \draw[dagconf]
                (A1.120)
                to[bend left=50]
                (X1.165);

            \draw[dagconf]
                (X1.15)
                to[bend left=50]
                (Y1.60);

            \draw[dagconf]
                (A1.40)
                to[bend left=30]
                (W1.190);

            \draw[dagconf]
                (A1.285)
                to[bend right=50]
                (Y1.255);

            \node at (0,-2.30) {\Large $\mathcal{G}$};
        \end{scope}

        \begin{scope}[xshift=6.8cm]
            \node[dagvar] (X2)  at (0,4.0) {$X$};
            \node[dagvar] (Z12) at (-2.0,2.40) {$Z_1$};
            \node[dagvar] (Z22) at (0.5,2.35) {$Z_2$};
            \node[dagvar] (W2)  at (0,0.85) {$W$};
            \node[dagvar] (A2)  at (-2.4,0) {$A$};
            \node[dagvar] (Y2)  at (2.4,0) {$Y$};

            \node[
                dagvar,
                minimum size=10mm,
                inner sep=1pt
            ] (UAX2) at (-2.05,4.0) {$U_{AX}$};

            \node[
                dagvar,
                minimum size=10mm,
                inner sep=1pt
            ] (UAW2) at (-2.05,1.25) {$U_{AW}$};

            \node[
                dagvar,
                minimum size=10mm,
                inner sep=1pt
            ] (UAY2) at (0,-1.20) {$U_{AY}$};

            \draw[dagarrow,dashed] (UAX2) -- (X2);
            \draw[dagarrow,dashed] (UAW2) -- (W2);
            \draw[dagarrow,dashed] (UAY2) -- (Y2);

            \draw[dagarrow] (X2.215) -- (Z12.35);
            \draw[dagarrow] (X2.285) -- (Z22.105);
            \draw[dagarrow] (Z22.255) -- (W2.75);

            \draw[dagarrow] (X2) -- (W2);
            \draw[dagarrow] (X2.315) -- (Y2.120);
            \draw[dagarrow] (A2.20) -- (W2.210);
            \draw[dagarrow] (W2.330) -- (Y2.160);
            \draw[dagarrow] (A2) -- (Y2);

            \draw[dagconf]
                (X2.15)
                to[bend left=50]
                (Y2.60);

            \node at (0,-2.30) {\Large $\mathcal{G}_{\bar A}$};
        \end{scope}

        \begin{scope}[xshift=13.6cm]
            \node[dagvar] (X3)  at (0,4.0) {$X$};
            \node[dagvar] (Z13) at (-2.0,2.40) {$Z_1$};
            \node[dagvar] (Z23) at (0,2.35) {$Z_2$};
            \node[dagvar] (W3)  at (0,0.85) {$W$};
            \node[dagvar] (A3)  at (-2.4,0) {$A$};
            \node[dagvar] (Y3)  at (2.4,0) {$Y$};

            \node[
                dagvar,
                minimum size=10mm,
                inner sep=1pt
            ] (UAX3) at (-2.05,4.0) {$U_{AX}$};

            \node[
                dagvar,
                minimum size=10mm,
                inner sep=1pt
            ] (UAY3) at (0,-1.20) {$U_{AY}$};

            \draw[dagarrow,dashed] (UAX3) -- (X3);
            \draw[dagarrow,dashed] (UAY3) -- (Y3);

            \draw[dagarrow] (X3.215) -- (Z13.35);
            \draw[dagarrow] (X3) -- (Z23);

            \draw[dagarrow] (X3.315) -- (Y3.120);
            \draw[dagarrow] (W3.330) -- (Y3.150);
            \draw[dagarrow] (A3) -- (Y3);

            \draw[dagconf]
                (X3.15)
                to[bend left=50]
                (Y3.60);

            \node at (0,-2.30)
                {\Large $\mathcal{G}_{\overline{AW}}$};
        \end{scope}
    \end{tikzpicture}
    }
    \caption{Causal diagrams of a CCB with a mediator and instrumental/proxy variables. The second and third diagrams are induce by intervening on $\{A\}$ and $\{A,W\}$ in the first diagram, respectively.}
    \label{fig:ccb-mediator-proxies}
\end{figure}
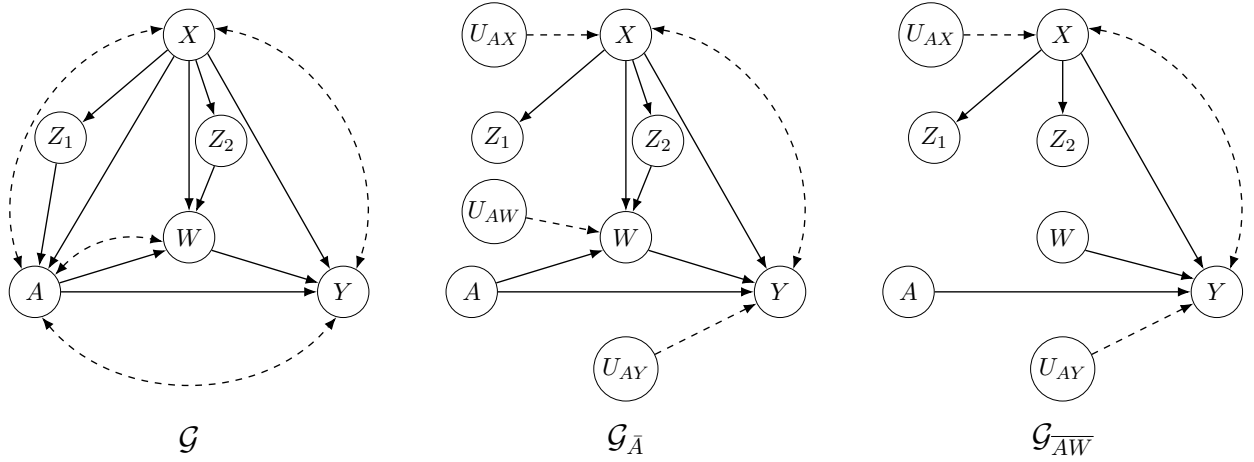

The resulting identification can be expressed through two bridge functions.

\begin{lemma}[IV/PV identification in CCB-Med]
\label{lemma:3.2}
Suppose Assumption~\ref{assumption:3.2} holds. For the observational processes in source domains $\mathscr{C}^{(1)}$ and $\mathscr{C}^{(2)}$, let functions $h_1\in L_2(\mathcal{W}\times\mathcal{A}\times\mathcal{X})$ and $h_2\in L_2(\mathcal{Z}_2\times\mathcal{A}\times\mathcal{X})$
satisfy
\begin{align}
    &\E^{(1)}[h_1(W,A,X)-Y\mid Z_1,Z_2]=0
    \qquad\mathrm{a.s.}, \label{eq:med_h1}
    \\
    &\E^{(2)}[h_2(Z_2,A,X)-h_1(W,A,X)\mid Z_1,A,X]=0
    \qquad\mathrm{a.s.}, \label{eq:med_h2}
\end{align}
where $\E^{(i)}$ denotes the expectation with respect to $p^{(i)}$, $i=1,2$. Then
\begin{align}
    &g^{(1)}(w,a,x)=h_1(w,a,x)
    \qquad\mathrm{a.s.},
    \label{eq:med_g1_identification}
    \\
    &g^*(a,x)=\E^{(2)}
    [h_2(Z_2,a,x)\mid X=x]
    \qquad\mathrm{a.s.}
    \label{eq:med_gstar_identification}
\end{align}
\end{lemma}

\begin{proof}
See \S\ref{sec:A.2} for a detailed proof.
\end{proof}

The first bridge $h_1$ identifies the causal response $g^{(1)}$ in source domain $\mathscr{C}^{(1)}$, exactly as in the CCB-Neg setting. The second bridge $h_2$ performs the additional proximal adjustment required in domain $\mathscr{C}^{(2)}$: rather than directly estimating the confounded interventional distribution $p^{(2)}(w|x,\operatorname{do}(a))$, it converts the transported expectation in \eqref{eq:gstar_med_transport} into an observational conditional moment equation. Marginalizing $h_2$ over $Z_2$ then yields the target CATE through \eqref{eq:med_gstar_identification}.

\subsubsection{Conditional Moment Formulation}
\label{sec:3.2.2}

Combining the identification results above, the target CATE $g^*$ can be characterized by the following integral equation system:
\begin{equation}
\begin{cases}
    \E^{(1)}
    [h_1(W,A,X)-Y\mid Z_1,Z_2]
    =0,
    \\[2mm]
    \E^{(2)}
    [h_2(Z_2,A,X)-h_1(W,A,X)\mid Z_1,A,X]
    =0,
    \\[2mm]
    \E^{(2)}
    [g(A',X)-h_2(Z_2,A',X)\mid A',X]
    =0.
\end{cases}
\label{eq:ccb_med_ies}
\end{equation}
Here, as in \S\ref{sec:3.1.2}, $A'$ is an auxiliary variable that is independent of the CCB process and follows the reference distribution
\[
    A'\sim\nu_A,
\]
whose support covers the action space of interest. The pseudo-action converts the pointwise characterization
\[
    g^*(a,x)
    =
    \E^{(2)}[h_2(Z_2,a,x)\mid X=x]
\]
into a single conditional moment restriction over $(A^\prime,X)$. Thus, \eqref{eq:ccb_med_ies} places all three bridge equations in the same conditional-moment form, preparing them for the unified minimax estimation framework developed in \S\ref{sec:4}.

\subsection{Identification in General CCBs}
\label{sec:3.3}

The preceding subsections develop identification results for two representative CCB structures, with negative outcomes and mediators, respectively. In both cases, the same general principle emerges: transportability is first used to decompose the target CATE into causal quantities associated with heterogeneous source domains, and instrumental- or proxy-variable methods are then used to recover the remaining interventional quantities from observational data. This raises a natural question: \textit{Can the same strategy be extended beyond these particular causal structures?}

The general transportability framework of \cite{Lee_Correa_Bareinboim_2020} provides a systematic way to determine whether a causal effect can be synthesized from heterogeneous domains when a collection of observational and interventional distributions is available. In the offline setting considered here, however, the available datasets are observational, and the required source-domain interventional distributions may themselves be confounded. We therefore combine general transportability with IV/PV identification. Before describing this strategy, we first abstract the bridge equations encountered above into a unified conditional-moment formulation.

\subsubsection{A Unified Form for the IES}
\label{sec:3.3.1}

The unconfounded examples in \S\ref{sec:2.2} and the confounded models studied in \S\ref{sec:3.1}--\S\ref{sec:3.2} all characterize the target conditional reward or CATE through an integral equation system consisting of conditional moment restrictions; see, for example, \eqref{eq:cb_neg_ies}, \eqref{eq:cb_med_ies}, \eqref{eq:ccb_neg_ies}, and \eqref{eq:ccb_med_ies}. Although the number of equations, the bridge functions involved, and their causal interpretations vary across these settings, the resulting IESs share a common mathematical structure. We now introduce a unified formulation that encompasses these different identification systems.

Suppose that an identification procedure produces $K$ conditional moment equations. Let $h=(h_1,\ldots,h_{K-1},g)$ denote the collection of bridge functions, where the final component $g$ represents the target conditional reward or CATE. For each component, let $\mathbf{X}_k$ collect the random variables on which the corresponding bridge function depends, and write
\[
    h(\mathbf{X})
    :=
    \bigl(
        h_1(\mathbf{X}_1),
        \ldots,
        h_{K-1}(\mathbf{X}_{K-1}),
        g(\mathbf{X}_K)
    \bigr),
\]
where $\mathbf{X}$ denotes the variables appearing as arguments of the bridge functions. The $k$-th equation can then be written as
\begin{equation}
    \E^{(i_k)}
    \left[
        \alpha_k\bigl(h(\mathbf{X}),\mathbf{Y}_k\bigr)
        \,\middle|\,
        \mathbf{Z}_k
    \right]=0,
    \qquad
    k\in[K],
    \label{eq:3.16}
\end{equation}
where:
\begin{itemize}
    \item $K$ is the total number of conditional moment equations;
    
    \item $\mathbf{Y}_k$ collects any additional observed variables appearing in the $k$-th moment restriction and may be empty when no such variable is required;
    
    \item $\alpha_k$ is the residual map associated with the $k$-th equation. For the bridge systems considered in this paper, $\alpha_k$ is linear in the relevant bridge-function evaluations;
    
    \item $\mathbf{Z}_k$ denotes the collection of variables conditioned on in the $k$-th equation;
    
    \item $i_k$ specifies the source domain supplying the $k$-th moment restriction, so that $\E^{(i_k)}$ is taken with respect to the distribution in domain $\mathscr{C}^{(i_k)}$.
\end{itemize}

For example, in the second bridge equation of the CCB-Med model in \S\ref{sec:3.2}, the residual is
\[
    \alpha_2
    =
    h_2(Z_2,A,X)-h_1(W,A,X),
\]
the conditioning variables are $\mathbf{Z}_2=(Z_1,A,X)$, and the expectation is taken in source domain $\mathscr{C}^{(2)}$.

For each $k\in[K]$, we write $\mu_k$ for the joint distribution of the variables involved in the $k$-th moment restriction under domain $\mathscr{C}^{(i_k)}$. Correspondingly, let $\mathcal{D}_k$ denote the projection of the source dataset $\mathcal{D}^{(i_k)}$ onto the variables required by that equation, together with any pseudo-variables introduced through the augmentation procedures in \S\ref{sec:3.1.2} and \S\ref{sec:3.2.2}. Thus, the empirical observations in $\mathcal{D}_k$ are distributed according to $\mu_k$.

We now associate each conditional moment equation with an operator $\mathcal{T}_k$ defined by
\begin{equation*}
    (\mathcal{T}_k h)(\cdot)
    :=
    \E_{\mu_k}
    \left[
        \alpha_k\bigl(h(\mathbf{X}),\mathbf{Y}_k\bigr)
        \,\middle|\,
        \mathbf{Z}_k=\cdot
    \right].
\end{equation*}
Collecting these operators gives $\mathcal{T}h
    :=
    \bigl(
        \mathcal{T}_1h,\ldots,\mathcal{T}_Kh
    \bigr),$
so that the entire IES can be written compactly as
\begin{equation}
    \mathcal{T}h(\mathbf{Z})=0,
    \label{eq:3.18}
\end{equation}
where $\mathbf{Z}$ denotes the collection of variables conditioned on in the IES. This operator representation separates the causal origin of the bridge equations from their statistical estimation. Once causal identification provides an IES of the form \eqref{eq:3.18}, the remaining problem is to learn a collection of bridge functions satisfying these conditional moment restrictions from finite offline datasets. Direct estimation of conditional moments can be difficult, particularly when the conditioning variables are high-dimensional or continuous. In \S\ref{sec:4}, we therefore transform \eqref{eq:3.18} into an unconditional minimax objective that admits empirical estimation and uncertainty quantification.

\subsubsection{Identification Strategy for General CCBs}
\label{sec:3.3.2}

We now describe how the identification procedure illustrated in \S\ref{sec:3.1} and \S\ref{sec:3.2} can be extended to more general CCB structures. \cite{Lee_Correa_Bareinboim_2020} provide the sound and complete algorithms \textsc{gTR} and \textsc{gTRu} for solving general transportability problems across heterogeneous domains with limited experiments. Informally, soundness guarantees that any returned transport formula is valid, while completeness guarantees that a transport formula is returned whenever the corresponding causal query is transportable given the available information.

In our offline setting, however, the interventional distributions required as inputs to the transportability algorithms are not directly observed. In the presence of unmeasured confounding, these quantities may first need to be identified from observational data using auxiliary variables such as instrumental variables, proxy variables, or negative controls \citep{Baiocchi2014, miao2018identifying, miao2020confounding}. We therefore combine these identification tools with general transportability as follows:

\begin{itemize}
\item \textbf{Reduce the CCB and construct available experiments.}
Identify suitable instrumental or proxy variables in the CCB. Remove the corresponding auxiliary variables, such as instrumental variables and negative controls, from the causal diagram, and add the interventional expressions identified through them to the set of available experiments.

\item \textbf{Specify domain discrepancies.}
On the reduced CCB diagram, specify the domain discrepancies $\{\Delta^{(i)}\}$ and construct the corresponding selection diagram.

\item \textbf{Apply general transportability algorithm.}
Run \textsc{gTR} \citep{Lee_Correa_Bareinboim_2020} using the reduced causal diagram, the domain discrepancies $\{\Delta^{(i)}\}$, and the available experiment set. If the target CATE is transportable from the source domains, the algorithm returns a transported decomposition; otherwise, it throws a \textsc{fail}.

\item \textbf{Construct the bridge-function IES.}
Recover the interventional expressions appearing in the transported decomposition using the IV/PV relations identified in the first step, and formulate the resulting bridge functions as an IES of the form discussed in \S\ref{sec:3.3.1}.

\end{itemize}

This procedure separates two conceptually distinct sources of difficulty. General transportability determines \emph{which causal mechanisms must be combined across domains}, while IV/PV identification determines \emph{how the required causal mechanisms can be recovered from confounded observational data within each domain}. The resulting bridge-function IES provides the interface between causal identification and the statistical estimation framework developed in the next section.

%% file: chap4.tex
\section{CAT Algorithm}
\label{sec:4}

In \S\ref{sec:3}, we reduced the identification of the target CATE to solving an integral equation system of conditional moment restrictions. In this section, we turn these population-level identification results into an offline policy learning procedure. We first develop a minimax estimator for the bridge functions in \S\ref{sec:4.1}. We then incorporate pessimism \citep{pmlr-v139-jin21e, uehara2023pessimistic, lu2023pessimism, chen2023unified} to account for statistical uncertainty and insufficient data coverage, which leads to our \textit{Causal-Adjusted Transportable} (CAT) policy learning algorithm in \S\ref{sec:4.2}.

\subsection{Bridge Function Estimation}
\label{sec:4.1}

Following \S\ref{sec:3.3.1}, we let
\[
    \mathcal{H}
    =
    \mathcal{H}_1\times\cdots\times\mathcal{H}_{K-1}\times\mathcal{G}
\]
denote the hypothesis class for the bridge-function vector
$h=(h_1,\ldots,h_{K-1},g)$. Recall that causal identification yields an IES
$\mathcal{T}h=0$. A natural population criterion for measuring violation of these conditional moment restrictions is the projected root-mean-square error (projected RMSE; see \S2 of \citealp{Dikkala2020}):
\begin{equation}
    \Vert\mathcal{T}h\Vert_{\mu,2}
    :=
    \sqrt{
        \sum_{k=1}^K
        \Vert\mathcal{T}_kh\Vert_{\mu_k,2}^2
    }
    =
    \sqrt{
        \sum_{k=1}^K
        \E_{\mu_k}
        \left[
            \left(
                \E_{\mu_k}
                [\alpha_k(h(\mathbf{X}),\mathbf{Y}_k)\mid\mathbf{Z}_k]
            \right)^2
        \right]
    },
    \label{eq:4.1}
\end{equation}
where $\mu=(\mu_1,\ldots,\mu_K)$. Hence the IES can be solved at the population level by searching for an $h\in\mathcal{H}$ with small projected RMSE.

The conditional expectations in \eqref{eq:4.1}, however, are difficult to estimate directly from finite data. Following the minimax approach for conditional moment models \citep{Dikkala2020,kallus2022causal}, we introduce a test function
$\theta_k:\mathcal{Z}_k\to\mathbb{R}$ from a dual function class $\Theta_k$ for each equation. Using Fenchel duality \citep{borwein2006convex} and the interchangeability principle \citep{pmlr-v54-dai17a}, we define the population minimax loss
\begin{equation}
    \mathcal{L}(h)
    :=\sum_{k=1}^K\mathcal{L}_k(h),
    \qquad
    \mathcal{L}_k(h)
    :=\sup_{\theta_k\in\Theta_k}
    \left\{
        \E_{\mu_k} [\alpha_k(h(\mathbf{X}),\mathbf{Y}_k)\theta_k(\mathbf{Z}_k)]
        -
        \frac{1}{2}\Vert\theta_k\Vert_{\mu_k,2}^2
    \right\}.
    \label{eq:4.2}
\end{equation}
We assume $\Theta_k$ is star-shaped around zero, so $\mathcal{L}_k(h)\geq 0$. Moreover, when the dual class is sufficiently rich so that $\mathcal{T}_kh\in\Theta_k$, the maximizer in \eqref{eq:4.2} is $\theta_k=\mathcal{T}_kh$, and therefore
\[
    \mathcal{L}(h)
    =
    \frac{1}{2}\Vert\mathcal{T}h\Vert_{\mu,2}^2.
\]
Therefore, minimizing the population minimax loss is equivalent to minimizing the projected conditional-moment error.

Replacing population expectations with their empirical counterparts gives 
\begin{equation} 
\mathcal{L}_{\mathcal{D}}(h) := \sum_{k=1}^K\mathcal{L}_{\mathcal{D},k}(h), \qquad \mathcal{L}_{\mathcal{D},k}(h) := \sup_{\theta_k\in\Theta_k} \left\{ \E_{\mathcal{D}_k} [\alpha_k(h(\mathbf{X}),\mathbf{Y}_k)\theta_k(\mathbf{Z}_k)] - \frac{1}{2}\Vert\theta_k\Vert_{\mathcal{D}_k,2}^2 \right\}, \label{eq:4.3} 
\end{equation} 
where $\mathcal{D}=\{\mathcal{D}^{(i)}\}_{i=1}^n$ denotes the collection of heterogeneous datasets, and each $\mathcal{D}_k$ is the dataset associated with the $k$-th conditional moment equation as defined in \S\ref{sec:3.3.1}. Here, $\E_{\mathcal{D}_k}$ denotes expectation under the empirical distribution induced by $\mathcal{D}_k$. The empirical loss $\mathcal{L}_{\mathcal{D}}$ therefore provides a finite-sample criterion for estimating the bridge functions, and in particular the target CATE $g^*$, from heterogeneous source datasets.

\subsection{Pessimistic Policy Learning}
\label{sec:4.2}

In regions with insufficient offline data coverage, estimates obtained by minimizing the empirical loss can be highly uncertain. Directly optimizing a policy using the corresponding estimated value may then exploit estimation errors in poorly covered regions and lead to poor performance. To account for this uncertainty, we adopt the pessimism principle commonly used in offline reinforcement learning \citep{pmlr-v139-jin21e,uehara2023pessimistic,lu2023pessimism,chen2023unified}: rather than relying on a single point estimate, we construct a confidence set of statistically plausible bridge functions and optimize the policy against the worst-case value within this set.

Regarding the empirical loss \eqref{eq:4.3} defined in \S\ref{sec:4.1}, for a threshold $e_{\mathcal{D}}>0$, we define the confidence set
\begin{equation}
    \mathrm{CI}_{\mathcal{H},\mathcal{D}}(e_{\mathcal{D}})
    :=
    \left\{
        h\in\mathcal{H}:
        \mathcal{L}_{\mathcal{D}}(h)
        \leq
        \inf_{h'\in\mathcal{H}}
        \mathcal{L}_{\mathcal{D}}(h')
        +
        e_{\mathcal{D}}
    \right\}.
    \label{eq:4.4}
\end{equation}
The threshold $e_{\mathcal{D}}$ accounts for finite-sample estimation error, so that, with an appropriate choice, the true bridge functions (or their best-in-class approximation under misspecification) belong to the confidence set with high probability; this property is formalized in \S\ref{sec:7}. 

Correspondingly, since the final component of $h$ is the target CATE, we define its induced confidence set as 
\begin{equation}
    \mathrm{CI}_{\mathcal{D}}(e_{\mathcal{D}})
    :=
    \left\{
        g\in\mathcal{G}:
        \exists h\in
        \mathrm{CI}_{\mathcal{H},\mathcal{D}}(e_{\mathcal{D}})
        \ \text{s.t.}\ 
        g=h^{(K)}
    \right\}.
    \label{eq:4.5}
\end{equation}
Thus, $\mathrm{CI}_{\mathcal{D}}(e_{\mathcal{D}})$ represents the collection of target CATEs that remain statistically plausible given the offline data.

Recall from \eqref{eq:cate_policy_value} that the value of a policy is determined by the target CATE. For any candidate CATE $g:\mathcal{A}\times\mathcal{X}\to\mathbb{R}$, define the corresponding policy value as
\begin{equation}
    v(g,\pi)
    :=
    \E_{p^\pi}[g(A,X)].
    \label{eq:4.6}
\end{equation}
When $g=g^*$, this quantity equals the true target-domain value $v^\pi$.

We then learn a policy by maximizing its pessimistic value over the CATE confidence set:
\begin{equation}
    \widehat{\pi}
    :=
    \underset{\pi\in\Pi}{\argmax}
    \;
    \min_{g\in\mathrm{CI}_{\mathcal{D}}(e_{\mathcal{D}})}
    v(g,\pi).
    \label{eq:4.7}
\end{equation}
For each policy $\pi$, the inner minimization evaluates its worst-case value over all statistically plausible CATEs, thereby providing a conservative estimate of its performance. The outer maximization then selects the policy with the largest such pessimistic value. This optimization favors policies that perform well even under plausible adverse realizations of the target CATE, thereby reducing sensitivity to statistical uncertainty in poorly covered regions. We summarize the complete procedure in Algorithm~\ref{alg:1}.

\begin{algorithm}
    \caption{Causal-Adjusted Transportable Policy Learning for CCB (CAT-CCB)}
    \label{alg:1}
    \begin{algorithmic}[1]
        \Require offline datasets
        $\mathcal{D}=\{\mathcal{D}^{(i)}\}_{i=1}^n= \{\{(o_t^{(i)},x_t^{(i)},a_t^{(i)},m_t^{(i)},y_t^{(i)})\}_{t=1}^T,\ i=1,\cdots,n\}$,
        hypothesis class $\mathcal{H}$,
        test function classes
        $\Theta=\Theta_1\times\cdots\times\Theta_K$,
        policy class $\Pi$
        \Ensure estimated policy $\widehat{\pi}$
        
        \State Identify the target CATE through an IES $\mathcal{T}h(\mathbf{Z})=0$ as in \S\ref{sec:3}.
        \State Construct the empirical minimax loss $\mathcal{L}_{\mathcal{D}}$ according to \eqref{eq:4.3}.
        \State Construct the CATE confidence set $\mathrm{CI}_{\mathcal{D}}(e_{\mathcal{D}})$ according to \eqref{eq:4.4}--\eqref{eq:4.5}.
        \State Solve the pessimistic policy optimization problem \eqref{eq:4.7}.
    \end{algorithmic}
\end{algorithm}

For CCB-Neg and CCB-Med, the identification IESs used in the first step are derived in \S\ref{sec:3.1} and \S\ref{sec:3.2}, respectively, while \S\ref{sec:3.3.2} provides an identification strategy for causal identification in more general CCB structures. Once an appropriate IES is constructed, the remaining steps depend only on its unified conditional-moment representation. The finite-sample properties of the resulting confidence set and the suboptimality guarantees for CAT-CCB are established in \S\ref{sec:7}.

The scope of CAT-CCB is not limited to confounded settings. Although the algorithm is developed to accommodate unmeasured confounding, the same estimation and pessimistic policy-learning procedure applies to the unconfounded contextual-bandit settings in \S\ref{sec:2.2}. In these cases, the corresponding IESs \eqref{eq:cb_neg_ies} and \eqref{eq:cb_med_ies} can be obtained more directly without confounding adjustment.

%% file: chap5.tex
\section{Confounded Markov Decision Processes}
\label{sec:5}

The contextual-bandit setting studied in the preceding sections describes a one-step decision problem. We now extend our framework to sequential decision making, where an action can affect not only the immediate reward but also future states and subsequent decisions. In particular, we consider episodic Markov decision processes (MDPs) in the presence of both unmeasured confounding and domain discrepancy. We first introduce the problem setting in \S\ref{sec:5.1}, and then study causal identification across heterogeneous MDPs in \S\ref{sec:5.2}. The resulting identified model will form the basis of the CAT-CMDP algorithm developed in \S\ref{sec:6}.

\subsection{Problem Setting}
\label{sec:5.1}

\paragraph{Episodic MDPs.}
An episodic MDP is specified by a tuple
\[
    \mathcal{M}
    =
    (\mathcal{S},\mathcal{A},H,\mathcal{P},r,\mu_1),
\]
where $\mathcal{S}$ and $\mathcal{A}$ are the state and action spaces, respectively, $H$ is the episode horizon, $\mathcal{P}=\{\mathcal{P}_h\}_{h=1}^H$ is a collection of transition kernels, $r=\{r_h\}_{h=1}^H$ is a collection of bounded expected reward functions, and $\mu_1$ is the distribution of the initial state. Specifically, $\mathcal{P}_h(\cdot\,|\,s,a)$ describes the distribution of the successor state at step $h$ given the current state-action pair $(s,a)$, while
\[
    r_h(s,a)
    :=
    \E[R_h|S_h=s,A_h=a]
    \in[0,1]
\]
denotes the expected immediate reward, where $R_h$ is the realized reward at step $h$. A policy $\pi=\{\pi_h\}_{h=1}^H$ specifies the agent's decision rule throughout an episode, where each $\pi_h:\mathcal{S}\to\Delta(\mathcal{A})$ assigns a distribution over actions conditional on the current state. 

At the beginning of an episode, the environment samples
$S_1\sim\mu_1$. At each step $h\in[H]$, the agent observes the current state $S_h$, chooses $A_h\sim\pi_h(\cdot\,|\,S_h)$, receives the immediate reward $R_h$, and the environment transitions to $S_{h+1}\sim \mathcal{P}_h(\cdot\,|\,S_h,A_h)$. The episode terminates after step $H$.

\paragraph{Policy evaluation.}
We use upper-case letters such as $S_h,A_h,R_h$ for random variables and lower-case letters such as $s_h,a_h,r_h$ for their realizations or, when clear from context, expected reward functions. For a policy $\pi$, define the state-value and state-action-value functions as
\begin{align*}
    V_h^\pi(s_h)
    &:=\E_\pi
    \left[
        \sum_{t=h}^H R_t\,\middle|\,S_h=s_h
    \right],
    \\
    Q_h^\pi(s_h,a_h)
    &:=\E_\pi
    \left[
        \sum_{t=h}^H R_t\,\middle|\,S_h=s_h,A_h=a_h
    \right],
\end{align*}
where $\E_\pi$ denotes expectation under the trajectory distribution induced by $\pi$. We set $V_{H+1}^\pi(s)=0$ for all $ s\in\mathcal{S}$. The expected cumulative reward of a policy is given by
\begin{equation}
    J(\pi):=\int_{\mathcal{S}}
    V_1^\pi(s_1)\,
    \mu_1(\mathrm{d}s_1).
    \label{eq:mdp_value}
\end{equation}
Given a policy class $\Pi$, we define a policy $\pi^*$ to be optimal within $\Pi$ if it satisfies
\begin{equation*}
    \pi^*
    \in
    \argmax_{\pi\in\Pi}J(\pi).
\end{equation*}
To evaluate a learned policy $\pi\in\Pi$, we use its suboptimality
\begin{equation*}
    \mathrm{SubOpt}(\pi):=J(\pi^*)-J(\pi),
\end{equation*}
which measures the gap between the expected cumulative reward achieved by $\pi$ and that of an optimal policy.

\paragraph{Confounded observational data.}
The standard MDP formulation assumes that the observed state contains sufficient information to characterize the decision, reward, and transition mechanisms, so that the value $J(\pi)$ of a given policy can be identified from the corresponding MDP model. In practice, however, it may be difficult or even impossible to encode all relevant factors into the observed state. Some unmeasured variables may influence both the actions taken by the historical agent and the subsequent rewards or state transitions, thereby inducing confounding in the observed action-reward and action-transition relationships. In the presence of such unmeasured confounding, the expected cumulative reward of a policy can no longer be identified from observational data without additional assumptions.

Specifically, following the formulation of confounded MDPs in \cite{fu2022offline}, we consider the observed variables $S_h$, $A_h$, and $R_h$ at each step $h\in[H]$, together with a time-varying latent state variable $U_h\in\mathcal{U}$. Here $S_h$ is the observed state, $A_h$ is the action taken at step $h$, and $R_h$ is the corresponding immediate reward. The latent variable $U_h$ may influence the historical action $A_h$, the immediate reward $R_h$, and the subsequent state $S_{h+1}$, thereby confounding the causal effect of $A_h$ on both the reward and future transition.

The historical data are generated through the following \textit{observational process}. Let $p$ denote the joint distribution of the relevant variables in a confounded MDP. At step $h$, the historical action mechanism is characterized by $p(a_h|u_h,s_h)$, while the reward and future transition distributions are given by $p(r_h|u_h,s_h,a_h)$ and $p(s_{h+1}|u_h,s_h,a_h)$, respectively. Each observed trajectory records
\[
    \{S_h,A_h,R_h\}_{h=1}^H,
\]
while the latent states $\{U_h\}_{h=1}^H$ are unlogged and therefore unavailable to the offline learner. As a result, simply conditioning on the observed state-action pair $(S_h,A_h)$ does not recover the causal reward or transition mechanism under an intervention on $A_h$. This confounding bias ultimately prevents direct identification of the policy value $J(\pi)$ from observational trajectories. Additional causal assumptions or auxiliary variables are therefore required for policy evaluation and learning from such observational data.

\paragraph{Heterogeneous source and target MDPs.}
Our setting further allows the environment in which the policy is deployed to differ from those in which the offline data are collected. We denote the target MDP by $\mathcal{M}^*$ and suppose that offline datasets $\mathcal{D}^{(1)},\cdots,\mathcal{D}^{(n)}$
are collected from heterogeneous source MDPs
$\mathscr{M}^{(1)},\cdots,\mathscr{M}^{(n)}$,
with corresponding data-generating distributions
$p^{(1)},\cdots,p^{(n)}$. Our objective is to learn a policy for the target environment $\mathscr{M}^*$ using these source-domain observational datasets.

Such heterogeneity arises naturally in applications. For example, when learning a policy for controlling a particular influenza strain in a target region, available data may include observations for the same strain in regions with different demographic or environmental characteristics, as well as historical data for related strains in the target region. None of these source environments need share the same optimal policy as the target environment, but they may share different components of the underlying causal dynamics.

While directly applying a policy optimized for any single source MDP can be suboptimal in $\mathscr{M}^*$, structural similarities across the source and target environments create an opportunity to combine complementary information from multiple datasets. In \S\ref{sec:5.2}, we formalize these similarities through domain discrepancies and use transportability together with instrumental-variable identification to recover the target reward and transition mechanisms.

\subsection{Causal Identification in Confounded MDPs}
\label{sec:5.2}

We now introduce the structural assumptions that allow causal information to be transported across heterogeneous MDPs. Similar to the CCB-Neg setting in \S\ref{sec:3.1}, we introduce an auxiliary variable $W_h$ at each decision step, which serves as a negative outcome in standard proxy variable literatures \citep{miao2018identifying, miao2020confounding}. We first specify the causal structure shared across domains, then characterize the domain discrepancies and derive a transportability decomposition, and finally use instrumental variables to identify the remaining interventional quantities from observational data.

\paragraph{Model assumptions.}
For the structural equations and conditional means considered below, we take $\mathcal{S}\subseteq\mathbb{R}^{d_S}$. At each step $h\in[H]$, let $W_h\in\mathcal{W}$ denote an auxiliary state variable that is logged in the source datasets. We impose the following assumptions on the MDPs in all domains. 

\begin{assumption}
\label{assumption:5.1}
For every $h\in[H]$, the following conditions hold:
\begin{itemize}
    \item[(i)] \textbf{Markov property.} For all $h\in[H]$,
    \begin{equation*}
        (S_{h+1},U_{h+1})
        \indep
        \{S_j,U_j,A_j,W_j\}_{1\leq j<h}
        \,\big|\,
        (S_h,U_h,A_h,W_h).
    \end{equation*}

    \item[(ii)] \textbf{Structural equations.}
    The next state $S_{h+1}$ and immediate reward $R_h$ depend only on $(S_h,A_h,W_h,U_h)$. In particular, they admit the regression representations
    \begin{equation}
        S_{h+1}
        =f_h^s(S_h,A_h,W_h)+\epsilon_h^s,
        \qquad
        R_h
        =f_h^r(S_h,A_h,W_h)+\epsilon_h^r,
        \label{eq:5.6}
    \end{equation}
    where
    \[
        (\epsilon_h^s,\epsilon_h^r)
        \indep
        (A_h,W_h)
        \,\big|\,
        (S_h,U_h).
    \]

    \item[(iii)] \textbf{Negative outcome.} For all $h\in[H]$,
    \begin{equation*}
        W_h
        \indep
        (A_h,U_h)
        \,\big|\,
        S_h.
    \end{equation*}
\end{itemize}
\end{assumption}

Assumption~\ref{assumption:5.1}(i) imposes a Markov structure on the observed $\{S_h\}_{h=1}^H$ and latent state processes $\{U_h\}_{h=1}^H$, as commonly assumed in confounded MDPs \citep{fu2022offline,lu2023pessimism,liao2021instrumental}. The structural equations in (ii) provide regression representations for the reward and next state, analogous to the structural reward assumptions used for CCBs in \S\ref{sec:3}. Condition (iii) states that, conditional on the current state $S_h$, the negative outcome $W_h$ is unaffected by the action and is independent of the latent confounder $U_h$.

We next specify how the source MDPs may differ from the target environment.

\begin{assumption}[MDP discrepancies]
\label{assumption:5.2}
We have two source MDPs $\mathscr{M}^{(1)},\mathscr{M}^{(2)}$ and a target MDP $\mathscr{M}^*$. Their domain discrepancies satisfy
\begin{equation*}
    \Delta^{(1)}
    \subseteq
    \{A_h,W_h\}_{h=1}^H,
    \qquad
    \Delta^{(2)}
    \subseteq
    \{A_h,S_h,R_h\}_{h=1}^H.
\end{equation*}
The corresponding selection diagram is illustrated in Figure~\ref{fig:confounded-mdp-selection}.
\end{assumption}

\begin{figure}[t]
    \centering
    \begin{tikzpicture}[>={Latex[length=2mm]}]

        \path[use as bounding box]
            (-4.55,-3.05) rectangle (4.55,3.00);

        \node[dagvar] (Sm1) at (-5,0) {$S_{h-1}$};
        \node[dagvar] (S0)  at ( 0,0) {$S_h$};
        \node[dagvar] (Sp1) at ( 5,0) {$S_{h+1}$};

        \node[dagvar] (Am1) at (-2.5,1.25) {$A_{h-1}$};
        \node[dagvar] (A0)  at ( 2.5,1.25) {$A_h$};

        \node[dagvar] (Wm1) at (-2.5,-1.25) {$W_{h-1}$};
        \node[dagvar] (W0)  at ( 2.5,-1.25) {$W_h$};

        \node[daglatent] (Um1) at (-0.40,2.10) {$U_{h-1}$};
        \node[daglatent] (U0)  at ( 4.60,2.10) {$U_h$};

        \node[dagsel] (CAm1) at (-2.5,2.65)
            {${C}_{A_{h-1}}$};

        \node[dagsel] (CA0) at (2.5,2.65)
            {${C}_{A_h}$};

        \node[dagsel] (CSm1) at (-5,-1.70)
            {${C}_{S_{h-1}}$};

        \node[dagsel] (CS0) at (0,-1.70)
            {${C}_{S_h}$};

        \node[dagsel] (CSp1) at (5,-1.70)
            {${C}_{S_{h+1}}$};

        \node[dagsel] (CWm1) at (-2.5,-2.60)
            {${C}_{W_{h-1}}$};

        \node[dagsel] (CW0) at (2.5,-2.60)
            {${C}_{W_h}$};

        \draw[dagarrow] (Sm1) -- (Am1);
        \draw[dagarrow] (Am1) -- (S0);

        \draw[dagarrow] (Sm1) -- (S0);

        \draw[dagarrow] (Sm1) -- (Wm1);
        \draw[dagarrow] (Wm1) -- (S0);

        \draw[dagarrow] (S0) -- (A0);
        \draw[dagarrow] (A0) -- (Sp1);

        \draw[dagarrow] (S0) -- (Sp1);

        \draw[dagarrow] (S0) -- (W0);
        \draw[dagarrow] (W0) -- (Sp1);

        \draw[dagarrow,dashed]
            (Um1.198) -- (Am1.15);

        \draw[dagarrow,dashed]
            (Um1) -- (S0);

        \draw[dagarrow,dashed]
            (U0.198) -- (A0.15);

        \draw[dagarrow,dashed]
            (U0) -- (Sp1);

        \draw[dagarrow] (CAm1) -- (Am1);
        \draw[dagarrow] (CA0)  -- (A0);

        \draw[dagarrow] (CSm1) -- (Sm1);
        \draw[dagarrow] (CS0)  -- (S0);
        \draw[dagarrow] (CSp1) -- (Sp1);

        \draw[dagarrow] (CWm1) -- (Wm1);
        \draw[dagarrow] (CW0)  -- (W0);

    \end{tikzpicture}

    \caption{Selection diagram $\mathcal{G}$ of the confounded MDP model. Square vertices labeled by ${C}$ denote selection nodes representing possible discrepancies between the source and target domains.}
    \label{fig:confounded-mdp-selection}
\end{figure}
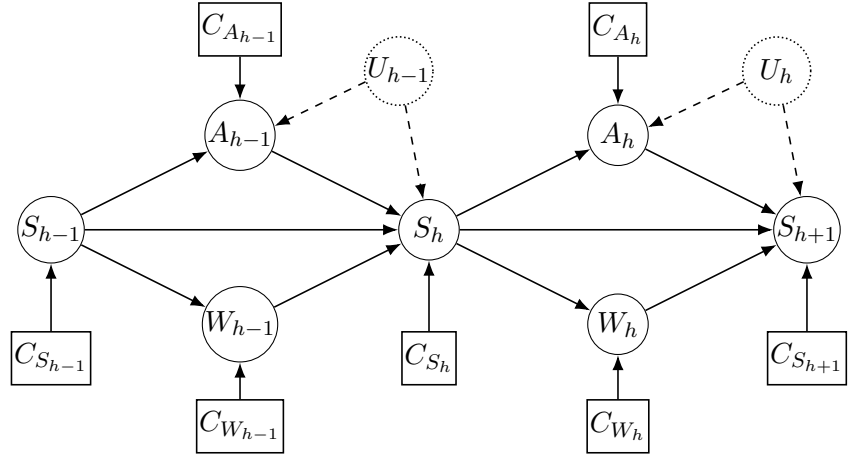

Under Assumption~\ref{assumption:5.2}, the target interventional transition distribution $p^*(s_{h+1}\,|\,s_h,\operatorname{do}(a_h))$ is directly transportable from neither source domain in general. In $\mathscr{M}^{(1)}$, the mechanism generating $W_h$ may differ from that in the target domain $\mathscr{M}^*$, whereas in $\mathscr{M}^{(2)}$, the mechanism relating $W_h$ to the subsequent state may differ. Consequently, neither source alone necessarily reproduces the complete target transition mechanism.

Nevertheless, the target interventional transition distribution can be recovered by combining complementary mechanisms from the two source domains. Following the $\mu$-transportability framework of \cite{pmlr-v31-bareinboim13a}, we obtain
\begin{align}
    p^*(s_{h+1}|s_h,\operatorname{do}(a_h))
    &=
    \int_{\mathcal{W}}
    p^*(s_{h+1}|s_h,w_h,\operatorname{do}(a_h))
    p^*(w_h|s_h,\operatorname{do}(a_h))
    \,\mathrm{d}w_h
    \notag\\
    &=
    \int_{\mathcal{W}}
    p^*(s_{h+1}|s_h,w_h,\operatorname{do}(a_h))
    p^*(w_h|s_h)
    \,\mathrm{d}w_h
    \notag\\
    &=
    \int_{\mathcal{W}}
    p^{(1)}(s_{h+1}|s_h,w_h,\operatorname{do}(a_h))
    p^{(2)}(w_h|s_h)
    \,\mathrm{d}w_h,
    \label{eq:5.7}
\end{align}
where the second equality follows from
\[
    (W_h\indep A_h\,|\,S_h)_{\mathcal{G}_{\overline{A}_h}},
\]
while the final equality follows from the S-admissibility conditions
\begin{align*}
    \left(
        \{C_{W_j},C_{A_j}\}_{j\in[H]}
        \indep
        S_{h+1}
        \,\middle|\,
        S_h,W_h
    \right)_{\mathcal{G}_{\overline{A}_h}}
    &\Rightarrow
    p^*(s_{h+1}|s_h,w_h,\operatorname{do}(a_h))
    =
    p^{(1)}(s_{h+1}|s_h,w_h,\operatorname{do}(a_h)),
    \\
    \left(
        \{C_{S_j},C_{A_j}\}_{j\in[H]}
        \indep
        W_h
        \,\middle|\,
        S_h
    \right)_{\mathcal{G}_{\overline{A}_h}}
    &\Rightarrow
    p^*(w_h|s_h)
    =
    p^{(2)}(w_h|s_h).
\end{align*}

Thus, $\mathscr{M}^{(1)}$ supplies the target-compatible conditional transition mechanism given $(S_h,A_h,W_h)$, while $\mathscr{M}^{(2)}$ supplies the target-compatible distribution of the negative outcome $W_h$ given $S_h$. The reward mechanism admits an analogous decomposition.

\paragraph{Identification with instrumental variables.}
The remaining term $p^{(1)}(s_{h+1}|s_h,w_h,\operatorname{do}(a_h))$ in \eqref{eq:5.7} is interventional and cannot generally be recovered by ordinary regression because the observational action $A_h$ is confounded by $U_h$. Following \cite{fu2022offline,liao2021instrumental}, we introduce an instrumental variable $Z_h$ at each step $h\in[H]$. The resulting observational causal structure is illustrated in Figure~\ref{fig:confounded-mdp-negative-iv}.

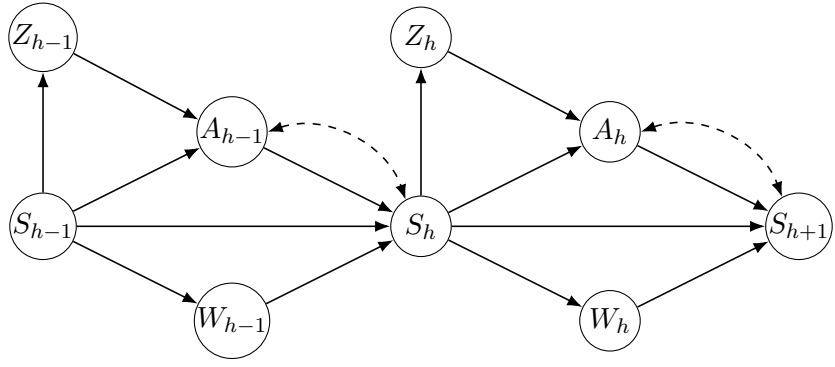
\begin{figure}[t]
    \centering
    \begin{tikzpicture}[>={Latex[length=2mm]}]

        \path[use as bounding box]
            (-4.55,-1.80) rectangle (4.55,2.65);

        \node[dagvar] (Sm1) at (-5,0) {$S_{h-1}$};
        \node[dagvar] (S0)  at ( 0,0) {$S_h$};
        \node[dagvar] (Sp1) at ( 5,0) {$S_{h+1}$};

        \node[dagvar] (Zm1) at (-5,2.5) {$Z_{h-1}$};
        \node[dagvar] (Z0)  at ( 0,2.5) {$Z_h$};

        \node[dagvar] (Am1) at (-2.5,1.25) {$A_{h-1}$};
        \node[dagvar] (A0)  at ( 2.5,1.25) {$A_h$};

        \node[dagvar] (Wm1) at (-2.5,-1.25) {$W_{h-1}$};
        \node[dagvar] (W0)  at ( 2.5,-1.25) {$W_h$};


        \draw[dagarrow] (Sm1) -- (Zm1);
        \draw[dagarrow] (Zm1) -- (Am1.160);

        \draw[dagarrow] (Sm1) -- (Am1);
        \draw[dagarrow] (Am1) -- (S0);

        \draw[dagarrow] (Sm1) -- (S0);

        \draw[dagarrow] (Sm1) -- (Wm1);
        \draw[dagarrow] (Wm1) -- (S0);

        \draw[dagconf]
            (Am1.0)
            to[bend left=45]
            (S0.120);


        \draw[dagarrow] (S0) -- (Z0);
        \draw[dagarrow] (Z0) -- (A0.160);

        \draw[dagarrow] (S0) -- (A0);
        \draw[dagarrow] (A0) -- (Sp1);

        \draw[dagarrow] (S0) -- (Sp1);

        \draw[dagarrow] (S0) -- (W0);
        \draw[dagarrow] (W0) -- (Sp1);

        \draw[dagconf]
            (A0.0)
            to[bend left=45]
            (Sp1.120);

    \end{tikzpicture}

    \caption{Causal graph of the confounded MDP with negative outcomes and instrumental variables in the observational setting. For brevity, latent confounders $\{U_h\}_{h=1}^H$ are represented by dashed bidirected arcs, and the reward variables $\{R_h\}_{h=1}^H$ are omitted.}
    \label{fig:confounded-mdp-negative-iv}
\end{figure}

We impose the following standard IV conditions.

\begin{assumption}
\label{assumption:5.3}
For every $h\in[H]$, the following conditions hold:
\begin{itemize}
    \item[(i)] \textbf{Exclusion restriction.}
    \begin{equation*}
        (S_{h+1},R_h)
        \indep
        Z_h
        \,\big|\,
        (S_h,U_h,A_h,W_h).
    \end{equation*}

    \item[(ii)] \textbf{Latent unconfoundedness.}
    \begin{equation*}
        Z_h\indep U_h\,|\,S_h,
        \qquad
        Z_h\indep(\epsilon_h^s,\epsilon_h^r)\,|\,(U_h,S_h).
    \end{equation*}

    \item[(iii)] \textbf{IV completeness.}
    For any
    $\sigma\in L^2(\mathcal{S}\times\mathcal{A}\times\mathcal{W})$,
    \begin{equation*}
        \E[\sigma(S_h,A_h,W_h)\mid S_h,Z_h]
        =0
        \quad\mathrm{a.s.}
        \quad\Longleftrightarrow\quad
        \sigma(S_h,A_h,W_h)
        =0
        \quad\mathrm{a.s.}
    \end{equation*}
    for each MDP under consideration.
\end{itemize}
\end{assumption}

Conditions (i) and (ii) encode the exclusion and exogeneity requirements for a valid instrument, while the completeness condition in (iii) ensures that the conditional moment equations induced by $Z_h$ uniquely identify the bridge functions, as in standard nonparametric IV methods \citep{Baiocchi2014}.

Together with the transportability decomposition above, these conditions allow the source-domain interventional reward and transition means to be recovered through bridge functions and subsequently transported to the target domain. The following lemma formalizes this identification.

\begin{lemma}[Identification of reward and transition means]
\label{lemma:5.1}
Suppose Assumptions~\ref{assumption:5.1}-\ref{assumption:5.3} hold. If there exist square-integrable functions
\[
    g_h^s:
    \mathcal{S}\times\mathcal{A}\times\mathcal{W}
    \to
    \mathbb{R}^{d_S},
    \qquad
    g_h^r:
    \mathcal{S}\times\mathcal{A}\times\mathcal{W}
    \to
    \mathbb{R},
\]
such that, for every $h\in[H]$,
\begin{align}
    \E^{(1)}
    \left[
        g_h^s(S_h,A_h,W_h)-S_{h+1}
        \,\middle|\,
        S_h,Z_h
    \right]
    &=0
    \qquad\mathrm{a.s.},
    \label{eq:mdp_bridge_s}
    \\
    \E^{(1)}
    \left[
        g_h^r(S_h,A_h,W_h)-R_h
        \,\middle|\,
        S_h,Z_h
    \right]
    &=0
    \qquad\mathrm{a.s.},
    \label{eq:mdp_bridge_r}
\end{align}
where the first equality is interpreted componentwise, then for any
$(s_h,a_h,w_h)\in\mathcal{S}\times\mathcal{A}\times\mathcal{W}$, the following hold for a.s. realizations of the underlying variables:
\begin{align}
    \E^{(1)}
    \left[
        S_{h+1}
        \,\middle|\,
        S_h=s_h,W_h=w_h,
        \operatorname{do}(A_h=a_h)
    \right]
    &=
    g_h^s(s_h,a_h,w_h),
    \label{eq:cmdp_s_bridge1}
    \\
    \E^*
    \left[
        S_{h+1}
        \,\middle|\,
        S_h=s_h,
        \operatorname{do}(A_h=a_h)
    \right]
    &=
    \int_{\mathcal{W}}
    g_h^s(s_h,a_h,w_h)
    p^{(2)}(w_h|s_h)
    \,\mathrm{d}w_h,
    \label{eq:cmdp_s_bridge2}
    \\
    \E^{(1)}
    \left[
        R_h
        \,\middle|\,
        S_h=s_h,W_h=w_h,
        \operatorname{do}(A_h=a_h)
    \right]
    &=
    g_h^r(s_h,a_h,w_h),
    \label{eq:cmdp_r_bridge1}
    \\
    \E^*
    \left[
        R_h
        \,\middle|\,
        S_h=s_h,
        \operatorname{do}(A_h=a_h)
    \right]
    &=
    \int_{\mathcal{W}}
    g_h^r(s_h,a_h,w_h)
    p^{(2)}(w_h|s_h)
    \,\mathrm{d}w_h.
    \label{eq:cmdp_r_bridge2}
\end{align}
\end{lemma}

\begin{proof}
See \S\ref{sec:A.3} for a detailed proof.
\end{proof}

Lemma~\ref{lemma:5.1} identifies the target conditional means of the immediate reward and next state by combining IV adjustment in $\mathscr{M}^{(1)}$ with the transported distribution of $W_h$ from $\mathscr{M}^{(2)}$. Later in \S\ref{sec:6}, we impose additional Gaussian-dynamics assumptions under which these identified quantities determine the target reward functions and transition kernels, which enables pessimistic offline policy learning in the target MDP.

%% file: chap6.tex
\section{CAT Algorithm for CMDPs under Gaussian Dynamics}
\label{sec:6}

In \S\ref{sec:5}, we established identification of the target-domain conditional reward and transition means from heterogeneous observational datasets. To turn these population-level identification results into an offline policy learning procedure, two additional steps are required. First, policy evaluation in an MDP depends on the full transition kernel rather than only its conditional mean. We therefore introduce a Gaussian dynamics model under which the transition kernel is fully determined by the identified mean. Second, the corresponding bridge functions and model parameters must be estimated from finite offline data while accounting for statistical uncertainty. We address these problems through function approximation and confidence-set construction, and then apply the pessimism principle to obtain our CAT-CMDP algorithm.

\subsection{Gaussian Dynamics and Model Specification}
\label{sec:6.1}

When the state space $\mathcal{S}$ is continuous or high-dimensional, the target transition and reward mechanisms cannot generally be estimated pointwise from finite offline data. We therefore introduce a structured function-approximation model. Motivated by the kernelized nonlinear regulator (KNR) framework \citep{kakade2020information,mania2020active}, we assume Gaussian dynamics together with a linear feature representation for the negative-outcome mechanism.

\begin{assumption}[Gaussian dynamics and linear negative-outcome model]
\label{assumption:6.1}
Suppose that
$\mathcal{S}\subseteq\mathbb{R}^{d_{\mathcal{S}}}$ and
$\mathcal{W}\subseteq\mathbb{R}^{d_{\mathcal{W}}}$.
For every $h\in[H]$, the following conditions hold:
\begin{itemize}
    \item[(i)] Given the state $S_h$, the target-compatible negative-outcome mechanism, shared by domains
    $\mathscr{M}^{(2)}$ and $\mathscr{M}^*$ under Assumption~\ref{assumption:5.2}, satisfies
    \begin{equation*}
        W_h\,|\, S_h=s_h\sim\mathcal{N}
        \bigl(
            \mathbf{M}_h^\top\phi(s_h),
            \sigma_w^2\mathbf{I}_{d_{\mathcal W}}
        \bigr),
    \end{equation*}
    where $\phi:\mathcal{S}\to\mathbb{R}^d$ is a known feature map satisfying
    $\Vert\phi(S_h)\Vert_2\leq 1$ almost surely,
    $\sigma_w>0$ is known, and
    $\mathbf{M}_h\in\mathbb{R}^{d\times d_{\mathcal W}}$
    satisfies $\Vert\mathbf{M}_h\Vert_2\leq m$ for some constant $m>0$.

    \item[(ii)] For every $(s_h,a_h)\in\mathcal{S}\times\mathcal{A}$,
    \begin{equation*}
        S_{h+1}
        \mid
        S_h=s_h,\operatorname{do}(A_h=a_h)
        \sim
        \mathcal{N}
        \left(
            m_h^*(s_h,a_h),
            \sigma_s^2\mathbf{I}_{d_{\mathcal S}}
        \right),
    \end{equation*}
    where $\sigma_s>0$ is known and, by \eqref{eq:cmdp_s_bridge2} in Lemma~\ref{lemma:5.1},
    \begin{equation*}
        m_h^*(s_h,a_h)
        :=
        \int_{\mathcal W}
        g_h^s(s_h,a_h,w_h)
        p^{(2)}(w_h|s_h)
        \,\mathrm{d}w_h.
    \end{equation*}
\end{itemize}
\end{assumption}

Assumption~\ref{assumption:6.1} connects the identification result of \S\ref{sec:5.2} to a fully specified target MDP. In particular, the Gaussian assumption determines the entire target transition distribution from the conditional transition mean identified in Lemma~\ref{lemma:5.1}, while the target expected reward is already determined by the corresponding reward bridge.

\paragraph{Bellman equations.} We now combine Lemma~\ref{lemma:5.1} with Assumption~\ref{assumption:6.1} to recover the target reward and transition mechanisms and derive the corresponding Bellman equations. Let
\[
    \varphi_{d_\mathcal{S}}(x)=
    (2\pi)^{-d_{\mathcal S}/2}
    \exp\left(-\frac{\Vert x\Vert_2^2}{2}\right)
\]
denote the density of a $d_{\mathcal S}$-dimensional standard Gaussian distribution. From the perspective of an offline learner, we define the target interventional transition kernel by
\[
    P_h^*(s_{h+1}|s_h,a_h):=p^*(s_{h+1}|s_h,\operatorname{do}(a_h)).
\]
Under Assumption~\ref{assumption:6.1}, it is given by
\begin{equation}
    P_h^*(s_{h+1}|s_h,\operatorname{do}(a_h))
    =\frac{1}{\sigma_s^{d_{\mathcal S}}}
    \varphi_{d_\mathcal{S}}
    \left(
        \frac{s_{h+1}-m_h^*(s_h,a_h)}{\sigma_s}
    \right).
    \label{eq:6.1}
\end{equation}
Similarly, the target expected reward function is
\begin{equation}
    r_h^*(s_h,a_h)
    :=
    \E^*
    \left[
        R_h
        \,\middle|\,
        S_h=s_h,\operatorname{do}(A_h=a_h)
    \right]
    =
    \int_{\mathcal W}
    g_h^r(s_h,a_h,w_h)
    p^{(2)}(w_h|s_h)
    \,\mathrm{d}w_h.
    \label{eq:6.2}
\end{equation}
For any measurable function $f:\mathcal{S}\to\mathbb{R}$, define the target transition operator by
\begin{equation}
    (P_h^*f)(s_h,a_h)
    :=
    \int_{\mathcal S}
    f(s_{h+1})
    P_h^*(s_{h+1}|s_h,a_h)
    \,\mathrm{d}s_{h+1}.
    \label{eq:6.3}
\end{equation}
Thus, $P_h^*f$ defines a real-valued function on
$\mathcal{S}\times\mathcal{A}$.

With the target transition kernels and expected reward functions specified above, the value functions of any policy
$\pi=\{\pi_h\}_{h=1}^H$ satisfy the Bellman equations
\begin{align}
    V_{H+1}^\pi(s)
    &=0,
    \label{eq:6.4}\\
    Q_h^\pi(s_h,a_h)
    &=r_h^*(s_h,a_h)+(P_h^*V_{h+1}^\pi)(s_h,a_h),
    \label{eq:6.5}\\
    V_h^\pi(s_h)
    &=\int_{\mathcal A}Q_h^\pi(s_h,a_h)\pi_h(\mathrm{d}a_h|s_h),
    \qquad h=H,\ldots,1.
    \label{eq:6.6}
\end{align}
The policy value $J(\pi)$ then follows from \eqref{eq:mdp_value}. Hence, under Assumption~\ref{assumption:6.1}, the quantities identified in Lemma~\ref{lemma:5.1} are sufficient to specify the target MDP and evaluate any policy.

\subsection{Estimation and Confidence Sets}
\label{sec:6.2}

Under the model specified in \S\ref{sec:6.1}, the target reward and transition mechanisms are determined by 
\[ 
\{g_h^r,g_h^s,\mathbf{M}_h\}_{h=1}^H, 
\] 
where the bridge functions $g_h^r$ and $g_h^s$ are identified from the observational data in domain $\mathscr{M}^{(1)}$ through the conditional moment equations in Lemma~\ref{lemma:5.1}, while $\mathbf{M}_h$, characterizing the conditional distribution of $W_h$, is estimated from $\mathscr{M}^{(2)}$. We now construct finite-sample estimators and confidence sets for these quantities.

\paragraph{Bridge-function estimation.}
We first consider the reward bridge $g_h^r$. Following the minimax estimation framework in \S\ref{sec:4.1}, define its population conditional-moment error by
\begin{equation*}
    \mathcal{L}_h^r(g)
    :=
    \E^{(1)}
    \left[
        \left(
            \E^{(1)}
            [R_h-g(S_h,A_h,W_h)\mid S_h,Z_h]
        \right)^2
    \right].
\end{equation*}
The inner conditional expectation is generally difficult to estimate directly from finite offline data. As in \S\ref{sec:4.1}, we use a minimax formulation to transform this conditional moment into an unconditional expectation, which admits a natural empirical counterpart. Introducing a test function $\theta:\mathcal{S}\times\mathcal{Z}\to\mathbb{R}$, Fenchel duality gives
\begin{equation*}
    \frac{1}{2}\mathcal{L}_h^r(g)
    =
    \sup_{\theta\in L^2(\mathcal{S}\times\mathcal{Z})}
    \E^{(1)}
    \left[
        (R_h-g(S_h,A_h,W_h))\theta(S_h,Z_h)
        -
        \frac{1}{2}\theta(S_h,Z_h)^2
    \right].
\end{equation*}
Let $\mathcal{G}$ and $\Theta$ denote the bridge and test function classes, respectively. Using the dataset $\mathcal{D}^{(1)}_h$, we define the empirical objective
\begin{equation}
    \widehat{\Phi}_h^r(g,\theta)
    :=
    \E_{\mathcal{D}^{(1)}_h}
    \left[
        (R_h-g(S_h,A_h,W_h))\theta(S_h,Z_h)
        -
        \frac{1}{2}\theta(S_h,Z_h)^2
    \right],
    \label{eq:6.7}
\end{equation}
and obtain the minimax estimator by
\begin{equation}
    \widetilde{g}_h^r
    \in
    \argmin_{g\in\mathcal{G}}
    \sup_{\theta\in\Theta}
    \widehat{\Phi}_h^r(g,\theta).
    \label{eq:6.8}
\end{equation}
For a threshold $\alpha_1>0$, we construct the corresponding confidence set
\begin{equation}
    \mathcal{C}_h^r(\alpha_1)
    :=
    \left\{
        g\in\mathcal{G}:
        \sup_{\theta\in\Theta}\widehat{\Phi}_h^r(g,\theta)
        -
        \sup_{\theta\in\Theta}\widehat{\Phi}_h^r(\widetilde{g}_h^r,\theta)
        \leq \alpha_1
    \right\}.
    \label{eq:6.9}
\end{equation}
The transition bridge $g_h^s$ is vector-valued. We estimate each component $g_h^{s(i)}$, $i\in[d_{\mathcal S}]$, analogously:
\begin{align}
    \widehat{\Phi}_h^{s(i)}(g,\theta)
    &:=
    \E_{\mathcal{D}^{(1)}_h}
    \left[
        (S_{h+1,i}-g(S_h,A_h,W_h))\theta(S_h,Z_h)
        -
        \frac{1}{2}\theta(S_h,Z_h)^2
    \right],
    \label{eq:6.10}\\
    \widetilde{g}_h^{s(i)}
    &\in
    \argmin_{g\in\mathcal{G}}
    \sup_{\theta\in\Theta}
    \widehat{\Phi}_h^{s(i)}(g,\theta),
    \label{eq:6.11}\\
    \mathcal{C}_h^{s(i)}(\alpha_2)
    &:=
    \left\{
        g\in\mathcal{G}:
        \sup_{\theta\in\Theta}\widehat{\Phi}_h^{s(i)}(g,\theta)
        -
        \sup_{\theta\in\Theta}
        \widehat{\Phi}_h^{s(i)}(\widetilde{g}_h^{s(i)},\theta)
        \leq\alpha_2
    \right\}.
    \label{eq:6.12}
\end{align}

The confidence sets collect bridge functions whose empirical losses are close to the corresponding minimax optima, thereby accounting for finite-sample estimation uncertainty. With appropriately chosen thresholds $\alpha_1$ and $\alpha_2$, the true bridge functions are contained in these sets with high probability; this property is formally established in \S\ref{sec:7.2}. 

For tractable estimation and finite-sample analysis, we adopt linear function classes for both the bridge and test functions, following common function-approximation settings in nonlinear control and offline reinforcement learning \citep{kakade2020information,pmlr-v139-jin21e,uehara2023pessimistic}. 

\begin{assumption}[Linear function approximation]
\label{assumption:6.2}
The function classes $\mathcal{G}$ and $\Theta$ satisfy:
\begin{itemize}
    \item[(i)]
    \[
        \mathcal{G}
        =
        \left\{
            g(s,a,w)
            =
            \psi(s,a,w)^\top\zeta:
            \zeta\in\mathbb{R}^d,\ 
            \Vert\zeta\Vert_2\leq1
        \right\},
    \]
    where
    $\psi:\mathcal{S}\times\mathcal{A}\times\mathcal{W}\to\mathbb{R}^d$
    is a known feature map satisfying
    $\Vert\psi(s,a,w)\Vert_2\leq1$.

    \item[(ii)]
    \[
        \Theta
        =
        \left\{
            \theta(s,z)
            =
            \omega(s,z)^\top\eta:
            \eta\in\mathbb{R}^d,\ 
            \Vert\eta\Vert_2\leq1
        \right\},
    \]
    where
    $\omega:\mathcal{S}\times\mathcal{Z}\to\mathbb{R}^d$
    is a known feature map satisfying
    $\Vert\omega(s,z)\Vert_2\leq1$.

    \item[(iii)] \textbf{Completeness.}
    For every $g\in\mathcal{G}$ and $i\in[d_{\mathcal S}]$,
    \begin{align*}
        \E^{(1)}
        [S_{h+1,i}-g(S_h,A_h,W_h)\mid S_h=\cdot,Z_h=\cdot]
        &\in\Theta,\\
        \E^{(1)}
        [R_h-g(S_h,A_h,W_h)\mid S_h=\cdot,Z_h=\cdot]
        &\in\Theta.
    \end{align*}
\end{itemize}
\end{assumption}

The completeness condition ensures that the test-function class is rich enough to represent the conditional residuals induced by the bridge class, thereby linking the minimax objectives above to the underlying conditional moment restrictions. Similar completeness conditions are commonly used in finite-sample analyses of minimax and pessimistic offline reinforcement learning, including \cite{uehara2023pessimistic,uehara2022finite}.

\paragraph{Estimation of the negative-outcome mechanism.}
It remains to estimate the conditional distribution of $W_h$ specified in Assumption~\ref{assumption:6.1}. Since the feature map $\phi$ and covariance $\sigma_w^2\mathbf{I}$ are known, the only unknown parameter is $\mathbf{M}_h\in\mathbb{R}^{d\times d_{\mathcal W}}$. Moreover, under Assumption~\ref{assumption:5.2}, this mechanism is shared by $\mathscr{M}^{(2)}$ and the target domain $\mathscr{M}^*$. Following standard ridge-regression approaches in offline reinforcement learning and related sequential learning problems \citep{pmlr-v139-jin21e,chang2021cov_shift,uehara2023pessimistic}, we estimate $\mathbf{M}_h$ from $\mathcal{D}_h^{(2)}$ by
\begin{equation}
    \widetilde{\mathbf{M}}_h
    :=
    \argmin_{\mathbf{M}\in\mathbb{R}^{d\times d_{\mathcal W}}}
    \E_{\mathcal{D}^{(2)}_h}
    \left[
        \Vert
        W_h-\mathbf{M}^\top\phi(S_h)
        \Vert_2^2
    \right]
    +
    \lambda\Vert\mathbf{M}\Vert_F^2,
    \label{eq:6.13}
\end{equation}
where $\lambda>0$ is a regularization parameter that stabilizes the estimation under finite data.

In order to quantify the estimation uncertainty of $\widetilde{\mathbf{M}}_h$, we introduce the regularized empirical covariance matrix
\begin{equation*}
    \Lambda_h
    :=
    \lambda\mathbf{I}
    +
    \sum_{k=1}^{K_2}
    \phi(S_h^{(2),k})
    \phi(S_h^{(2),k})^\top.
\end{equation*}
The matrix $\Lambda_h$ summarizes the empirical coverage of the state features in $\mathcal{D}^{(2)}_h$ and determines the geometry of the uncertainty around the ridge estimator. For a threshold $\alpha_3>0$, we construct the ellipsoidal confidence region
\begin{equation}
    \mathcal{C}_h^M(\alpha_3)
    :=
    \left\{
        \mathbf{M}\in\mathbb{R}^{d\times d_{\mathcal W}}:
        \left\|
            (\mathbf{M}-\widetilde{\mathbf{M}}_h)^\top
            \Lambda_h^{1/2}
        \right\|_2
        \leq
        \alpha_3
    \right\}.
    \label{eq:6.14}
\end{equation}
This region adapts its shape around $\widetilde{\mathbf{M}}_h$ to the empirical feature coverage encoded by $\Lambda_h$: directions that are better covered by the offline data admit smaller uncertainty. With an appropriate choice of $\alpha_3$, the true parameter $\mathbf{M}_h$ lies in $\mathcal{C}_h^M(\alpha_3)$ with high probability, as established in \S\ref{sec:7}.

Together, the confidence sets
$\mathcal{C}_h^r(\alpha_1)$, $\{\mathcal{C}_h^{s(i)}(\alpha_2)\}_{i=1}^{d_{\mathcal S}}$, and $\mathcal{C}_h^M(\alpha_3)$ capture the finite-sample uncertainty in all components required to specify the target MDP. In \S\ref{sec:6.3}, we combine these sets to construct the pessimistic CAT-CMDP policy-learning objective.

\subsection{CAT Algorithm for Confounded MDPs}
\label{sec:6.3}

Based on \S\ref{sec:6.1} and \S\ref{sec:6.2}, the target MDP is characterized by the collection
\[
    \beta
    :=
    \{g_h^r,g_h^s,\mathbf{M}_h\}_{h=1}^H.
\]
To account for the finite-sample uncertainty in these components, we combine the confidence sets constructed in \S\ref{sec:6.2} into the model confidence set
\begin{equation*}
    \mathcal{C}_\beta
    :=
    \prod_{h=1}^H
    \left(
        \mathcal{C}_h^r(\alpha_1)
        \times
        \prod_{i=1}^{d_{\mathcal S}}
        \mathcal{C}_h^{s(i)}(\alpha_2)
        \times
        \mathcal{C}_h^M(\alpha_3)
    \right).
\end{equation*}
With appropriate choices of the confidence thresholds, the true target model $\beta^*$ belongs to $\mathcal{C}_\beta$ with high probability, as established in \S\ref{sec:7.2}.

For any candidate model $\beta\in\mathcal{C}_\beta$, the corresponding transition kernels
$\{P_h^\beta\}_{h=1}^H$ and expected reward functions
$\{r_h^\beta\}_{h=1}^H$ are specified through \eqref{eq:6.1} and \eqref{eq:6.2}. The model-specific value functions
$V_h^\pi(\cdot;\beta)$ and $Q_h^\pi(\cdot,\cdot;\beta)$ are then obtained from the Bellman equations \eqref{eq:6.4}--\eqref{eq:6.6}, yielding the model-specific policy value
\begin{equation*}
    J(\pi;\beta)
    :=
    \int_{\mathcal S}
    V_1^\pi(s_1;\beta)\,
    \mu_1(\mathrm{d}s_1).
\end{equation*}
When $\beta=\beta^*$, this quantity coincides with the true target-domain value $J(\pi)$.

Following the pessimism principle used for CAT-CCB in \S\ref{sec:4.2}, we evaluate each policy by its worst-case value over all statistically plausible target models and optimize
\begin{equation}
    \widehat{\pi}
    \in
    \argmax_{\pi\in\Pi}
    \min_{\beta\in\mathcal{C}_\beta}
    J(\pi;\beta).
    \label{eq:6.20}
\end{equation}
This optimization guards against overly optimistic value estimates caused by insufficient coverage of the offline data. We summarize the resulting CAT-CMDP procedure in Algorithm~\ref{alg:2}.

\begin{algorithm}[H]
    \caption{Causal-Adjusted Transportable Policy Learning for CMDPs (CAT-CMDP)}
    \label{alg:2}
    \begin{algorithmic}[1]
        \Require Confidence thresholds $\{\alpha_1,\alpha_2,\alpha_3\}$, ridge parameter $\lambda$, offline datasets $\mathcal{D}^{(1)}$ and $\mathcal{D}^{(2)}$, function classes $\mathcal{G}$ and $\Theta$, policy class $\Pi$
        \Ensure Estimated policy $\widehat{\pi}$

        \For{$h=1,\ldots,H$}
            \State Estimate the reward bridge $\widetilde{g}_h^r$ and construct $\mathcal{C}_h^r(\alpha_1)$ according to \eqref{eq:6.7}--\eqref{eq:6.9}.
            \State For each component $i\in[d_{\mathcal S}]$, estimate the transition bridge $\widetilde{g}_h^{s(i)}$ and construct $\mathcal{C}_h^{s(i)}(\alpha_2)$ according to \eqref{eq:6.10}--\eqref{eq:6.12}. \vspace{0.1cm}
            \State Estimate $\widetilde{\mathbf{M}}_h$ by ridge regression and construct $\mathcal{C}_h^M(\alpha_3)$ according to \eqref{eq:6.13}--\eqref{eq:6.14}.
        \EndFor
        \State Construct the model confidence set $\mathcal{C}_\beta$.
        \State Solve the pessimistic policy optimization problem \eqref{eq:6.20}.
    \end{algorithmic}
\end{algorithm}

To further interpret the pessimistic optimization in \eqref{eq:6.20}, it is useful to make explicit the least favorable model associated with each policy. For a fixed policy $\pi\in\Pi$, let
\begin{equation*}
    \widehat{\beta}^{\,\pi}
    \in
    \argmin_{\beta\in\mathcal{C}_\beta}
    J(\pi;\beta)
\end{equation*}
denote a pessimistic model for $\pi$. Then \eqref{eq:6.20} can equivalently be viewed as the adversarial optimization
\begin{equation*}
    \widehat{\beta}^{\,\pi}
    \in
    \argmin_{\beta\in\mathcal{C}_\beta}J(\pi;\beta),
    \qquad
    \widehat{\pi}
    \in
    \argmax_{\pi\in\Pi}
    J(\pi;\widehat{\beta}^{\,\pi}).
\end{equation*}
Thus, CAT-CMDP selects a policy that performs well even under the least favorable target model consistent with the offline data. The finite-sample suboptimality algorithm for this procedure is established in \S\ref{sec:7.2}.

%% file: chap7.tex
\section{Theoretical Results}
\label{sec:7}
In this section, we establish finite-sample guarantees for the proposed CAT algorithms. We first analyze CAT-CCB in \S\ref{sec:7.1}, where the general integral-equation formulation in \S\ref{sec:3.3} allows us to characterize the statistical uncertainty of the estimated bridge functions and translate it into policy suboptimality. We then study CAT-CMDP in \S\ref{sec:7.2}, deriving a finite-sample suboptimality bound under the structural and function-approximation assumptions introduced in \S\ref{sec:5}-\S\ref{sec:6}.

\subsection{Analysis for CAT-CCB}
\label{sec:7.1}

In this section, we establish finite-sample guarantees for the CAT-CCB algorithm proposed in \S\ref{sec:4}. We first characterize the approximation properties of the bridge-function and test-function classes, and then use these properties to analyze the confidence set and the resulting policy suboptimality.

\begin{assumption}[Approximate realizability of the hypothesis class]
\label{assumption:7.1}
There exist $\varepsilon_\mathcal{H}\geq 0$ and
$h_\mathcal{H}^*
=
\{h_{\mathcal{H},1}^*,\ldots,h_{\mathcal{H},K-1}^*,g_\mathcal{H}^*\}
\in\mathcal{H}$
such that
\begin{itemize}
    \item[(i)]
    \[
        \Vert\mathcal{T}h_\mathcal{H}^*\Vert_{\mu,2}
        \leq
        \varepsilon_\mathcal{H},
    \]
    where $\Vert\mathcal{T}h\Vert_{\mu,2}$ is the projected RMSE
    defined in \eqref{eq:4.1}.

    \item[(ii)]
    \[
        \sup_{v\in\mathcal{V}}
        \Vert g_\mathcal{H}^*-g^*\Vert_{v,2}
        \leq
        \varepsilon_\mathcal{H},
    \]
    where
    \[
        \mathcal{V}
        :=
        \{
            v:
            v(x,a)=p^*(x)\pi(a|x),\
            \pi\in\Pi
        \}.
    \]
\end{itemize}
\end{assumption}

Assumption~\ref{assumption:7.1} allows the hypothesis class to be misspecified: the exact bridge solution need not belong to $\mathcal H$. Condition (i) controls how well $h_\mathcal{H}^*$ satisfies the population integral equations, while condition (ii) directly controls the approximation error of the induced CATE under the state-action distributions generated by policies in $\Pi$. The case $\varepsilon_\mathcal{H}=0$ corresponds to exact
realizability.

We next impose a compatibility condition on the test-function classes used in the minimax estimation.

\begin{assumption}[Compatibility of the test-function class]
\label{assumption:7.2}
There exists $\varepsilon_\Theta\geq0$ such that, for every
$h\in\mathcal{H}$ and every $k\in\{1,\ldots,K\}$,
\begin{equation*}
    \inf_{\theta_k\in\Theta_k}
    \Vert
        \theta_k-\mathcal{T}_kh
    \Vert_{\mu_k,2}
    \leq
    \varepsilon_\Theta.
\end{equation*}
\end{assumption}

Thus, $\varepsilon_\Theta$ measures how well the test-function class $\Theta_k$ can approximate the conditional moment $\mathcal{T}_kh$. This compatibility condition connects the minimax objective in \S\ref{sec:4.1} to the population conditional-moment error.

\begin{assumption}[Regularity]
\label{assumption:7.3}
For every $k\in\{1,\ldots,K\}$, we assume:
\begin{itemize}
    \item[(i)]
    $\alpha_k$ is $L_{\alpha,1}$-Lipschitz continuous with respect to
    $\mathbf{Y}_k$ and each component of $h$;

    \item[(ii)]
    the support of $\mathbf{Y}_k$ is bounded:
    \[
        \Vert\operatorname{supp}(\mathbf{Y}_k)\Vert_\infty
        \leq L_Y;
    \]

    \item[(iii)]
    \[
        \sup_{h\in\mathcal{H}}\Vert h\Vert_\infty\leq L_h,
        \qquad
        \sup_{\theta\in\Theta}\Vert\theta\Vert_\infty\leq L_\theta,
    \]
    where $\Theta=\Theta_1\times\cdots\times\Theta_K$.
\end{itemize}
\end{assumption}

For the integral equation systems considered in this paper, the moment functions $\alpha_k$ are continuous, and the remaining regularity conditions can be enforced by choosing bounded hypothesis and test-function classes. Under Assumption~\ref{assumption:7.3}, we further denote by $L_\alpha$ a uniform bound satisfying
\[
    \Vert
        \alpha_k(h(\mathbf X),\mathbf Y_k)
    \Vert_\infty
    \leq L_\alpha
\]
for all $h\in\mathcal H$ and $k\in\{1,\ldots,K\}$.

For the subsequent finite-sample analysis, we fix a failure probability $\xi\in(0,1)$, which can be chosen arbitrarily. For each $k\in\{1,\ldots,K\}$, define the product function class
\[
    \mathcal{Q}_k
    :=
    \left\{
        (\mathbf{x},\mathbf{y}_k,\mathbf{z}_k)
        \mapsto
        \alpha_k(h(\mathbf{x}),\mathbf{y}_k)
        \theta_k(\mathbf{z}_k):
        h\in\mathcal H,\,
        \theta_k\in\Theta_k
    \right\}.
\]
Let $\bar{\eta}_k$ denote a critical radius associated with
$\mathcal{Q}_k$ and $\Theta_k$. We choose $\eta_k$ to account for both the complexity of these function classes and the desired confidence level. Up to universal constants,
\[
    \eta_k\gtrsim\bar{\eta}_k
    +\sqrt{
        \frac{\log(1/\xi)}{|\mathcal{D}_k|}
    }.
\]
We further define
\[
    \eta^2:=\sum_{k=1}^K\eta_k^2.
\]
We then define the event $\mathcal{E}$ on which the relevant empirical quantities concentrate uniformly around their population counterparts:
\begin{align}
\mathcal{E}
=
\biggl\{
&
\left|
\E_{\mathcal{D}_k}
[
    \alpha_k(h(\mathbf{X}),\mathbf{Y}_k)
    \theta_k(\mathbf{Z}_k)
]
-
\E_{\mu_k}
[
    \alpha_k(h(\mathbf{X}),\mathbf{Y}_k)
    \theta_k(\mathbf{Z}_k)
]
\right|
\leq\eta_k
\left(L_\alpha\Vert\theta_k\Vert_{\mu_k,2}+\eta_k\right),
\nonumber\\
&
\left|
    \Vert\theta_k\Vert_{\mathcal{D}_k,2}^2
    -
    \Vert\theta_k\Vert_{\mu_k,2}^2
\right|
\leq
\frac{1}{2}
\left(
    \Vert\theta_k\Vert_{\mu_k,2}^2+\eta_k^2
\right),\quad
\forall h\in\mathcal H,\ 
\forall\theta_k\in\Theta_k,\ 
\forall k\in\{1,\ldots,K\}
\biggr\}.
\label{eq:7.2}
\end{align}
or each $k$, the two inequalities above hold uniformly with failure probability at most $\xi$. A union bound over the two inequalities and the $K$ integral equations therefore yields
\[
    \Pr(\mathcal{E})\geq 1-2K\xi.
\]

The following result, adapted from \cite{chen2023unified}, uses this concentration event to characterize the finite-sample properties of the confidence set.

\begin{theorem}[Uncertainty quantification, \citealp{chen2023unified}]
\label{thm:7.1}
Suppose that Assumptions~\ref{assumption:7.1},
\ref{assumption:7.2}, and \ref{assumption:7.3} hold. Then event
$\mathcal E$ holds with probability at least $1-2K\xi$. Moreover, on
$\mathcal E$, the following properties hold:
\begin{itemize}
    \item[(i)]
    For $h_\mathcal H^*$ satisfying
    Assumption~\ref{assumption:7.1},
    \begin{equation*}
        \mathcal{L}_\mathcal{D}(h_\mathcal{H}^*)
        \leq
        2\varepsilon_\mathcal{H}^2
        +
        \left(
            2L_\alpha^2+\frac{5}{4}
        \right)\eta^2.
    \end{equation*}
    Consequently, if
    \[
        e_\mathcal{D}
        >
        2\varepsilon_\mathcal{H}^2
        +
        \left(
            2L_\alpha^2+\frac{5}{4}
        \right)\eta^2,
    \]
    then
    $h_\mathcal H^*\in
    \mathrm{CI}_{\mathcal H,\mathcal D}(e_\mathcal D)$,
    and hence
    $g_\mathcal H^*\in
    \mathrm{CI}_{\mathcal D}(e_\mathcal D)$.

    \item[(ii)]
    For every
    $h\in
    \mathrm{CI}_{\mathcal H,\mathcal D}(e_\mathcal D)$,
    \begin{equation*}
        \sup_{k\in\{1,\ldots,K\}}
        \Vert\mathcal T_kh\Vert_{\mu_k,2}
        \overset{\mathcal E}{\leq}
        \mathcal O(\varepsilon_\Theta)
        +
        \mathcal O(\varepsilon_\mathcal H)
        +
        \mathcal O(\sqrt{e_\mathcal D})
        +
        \mathcal O(\eta).
    \end{equation*}
\end{itemize}
\end{theorem}

\begin{proof}
See \S C.1 of \cite{chen2023unified}.
\end{proof}

Theorem~\ref{thm:7.1} provides the uncertainty-quantification result needed
for pessimistic policy learning. Part (i) guarantees that the confidence set
contains a near-population solution with high probability, while part (ii)
controls the population conditional-moment error uniformly over all bridge
functions retained by the confidence set.

We first apply this result to the unconfounded contextual-bandit settings
introduced in \S\ref{sec:2.2}.

\begin{theorem}[Suboptimality for unconfounded contextual bandits]
\label{thm:7.2}
Suppose that Assumption~\ref{assump:cb_neg} for contextual bandits with a
negative outcome holds, and that
Assumptions~\ref{assumption:7.1}--\ref{assumption:7.3} hold for
$\mathcal H$ and $\Theta$. Set
\[
    e_\mathcal D
    >
    2\varepsilon_\mathcal H^2
    +
    \left(
        2L_\alpha^2+\frac{5}{4}
    \right)\eta^2.
\]
Define
$b_1:\mathcal W\times\mathcal A\times\mathcal X\to\mathbb R$
and
$b_2:\mathcal A\times\mathcal X\to\mathbb R$ by
\begin{equation}
    b_1(w,a,x)
    =
    \frac{
        p^{(2)}(w|x)p^*(x)\pi^*(a|x)
    }{
        p^{(1)}(w,a,x)
    },
    \qquad
    b_2(a,x)
    =
    \frac{
        p^*(x)\pi^*(a|x)
    }{
        p^{(2)}(a,x)
    }.
    \label{eq:7.5}
\end{equation}
Then, on event $\mathcal E$, which occurs with probability at least
$1-4\xi$,
\begin{equation}
    \mathrm{SubOpt}(\widehat{\pi})
    \overset{\mathcal E}{\leq}
    \sum_{k=1}^2
    \Vert b_k\Vert_{\mu_k,2}
    \left(
        \mathcal O(\varepsilon_\Theta)
        +
        \mathcal O(\varepsilon_\mathcal H)
        +
        \mathcal O(\sqrt{e_\mathcal D})
        +
        \mathcal O(\eta)
    \right).
    \label{eq:7.6}
\end{equation}
\end{theorem}

The functions $b_1$ and $b_2$ reweight source-domain moments toward the
state-action distribution induced by the optimal target policy. Their
$L^2$ norms therefore quantify the corresponding overlap, or distribution
shift, between the available source data and the target policy.

The same argument applies to the unconfounded contextual-bandit model with a
mediator. If Assumption~\ref{assump:cb_med} and
Assumptions~\ref{assumption:7.1}--\ref{assumption:7.3} hold, then
\eqref{eq:7.6} remains valid on $\mathcal E$ with probability at least
$1-4\xi$, after replacing \eqref{eq:7.5} by
\begin{equation*}
    b_1(w,a,x)
    =
    \frac{
        p^{(2)}(w|x,a)p^*(x)\pi^*(a|x)
    }{
        p^{(1)}(w,a,x)
    },
    \qquad
    b_2(a,x)
    =
    \frac{
        p^*(x)\pi^*(a|x)
    }{
        p^{(2)}(a,x)
    }.
\end{equation*}

\begin{proof}
See \S\ref{sec:B.1} and \S\ref{sec:B.2} for detailed proofs of the two
unconfounded cases.
\end{proof}

We next turn to the confounded contextual-bandit models in
\S\ref{sec:3}. In these settings, the same pessimistic analysis applies, but
the weighting functions must additionally account for the bridge equations
used to adjust for unmeasured confounding.

\begin{theorem}[Suboptimality for CCB-Neg]
\label{thm:7.3}
Suppose that Assumption~\ref{assumption:3.1} for CCB-Neg holds and that
Assumptions~\ref{assumption:7.1}--\ref{assumption:7.3} hold for
$\mathcal H$ and $\Theta$. Set
\[
    e_\mathcal D
    >
    2\varepsilon_\mathcal H^2
    +
    \left(
        2L_\alpha^2+\frac{5}{4}
    \right)\eta^2.
\]
Suppose there exists
$b_1:\mathcal Z_1\times\mathcal Z_2\to\mathbb R$
satisfying
\begin{equation*}
    \E^{(1)}
    [
        b_1(Z_1,Z_2)
        \mid W=w,A=a,X=x
    ]
    =
    \frac{
        p^{(2)}(w|x)p^*(x)\pi^*(a|x)
    }{
        p^{(1)}(w,a,x)
    },
\end{equation*}
and define
$b_2:\mathcal A\times\mathcal X\to\mathbb R$ by
\begin{equation*}
    b_2(a,x)
    =
    \frac{
        p^*(x)\pi^*(a|x)
    }{
        p^{(2)}(a,x)
    }.
\end{equation*}
Then, on event $\mathcal E$, which occurs with probability at least
$1-4\xi$,
\begin{equation*}
    \mathrm{SubOpt}(\widehat{\pi})
    \overset{\mathcal E}{\leq}
    \sum_{k=1}^2
    \Vert b_k\Vert_{\mu_k,2}
    \left(
        \mathcal O(\varepsilon_\Theta)
        +
        \mathcal O(\varepsilon_\mathcal H)
        +
        \mathcal O(\sqrt{e_\mathcal D})
        +
        \mathcal O(\eta)
    \right).
\end{equation*}
\end{theorem}

\begin{proof}
See \S\ref{sec:B.1} and \S\ref{sec:B.3} for a detailed proof.
\end{proof}

Here, the existence of $b_1$ is a representability condition linking the target-policy weighting ratio to the IV variables $(Z_1,Z_2)$. Its norm, together with that of $b_2$, controls how strongly estimation errors in the source-domain moment equations are amplified when evaluating the target policy.

Similarly, the suboptimality guarantee for CCB-Med is given below.

\begin{theorem}[Suboptimality for CCB-Med]
\label{thm:7.4}
Suppose that Assumption~\ref{assumption:3.2} for CCB-Med holds and that
Assumptions~\ref{assumption:7.1}--\ref{assumption:7.3} hold for
$\mathcal H$ and $\Theta$. Set
\[
    e_\mathcal D
    >
    2\varepsilon_\mathcal H^2
    +
    \left(
        2L_\alpha^2+\frac{5}{4}
    \right)\eta^2.
\]
Suppose there exists
$b_1:\mathcal Z_1\times\mathcal Z_2\to\mathbb R$
such that
\begin{equation*}
    \E^{(1)}
    [
        b_1(Z_1,Z_2)
        \mid W=w,A=a,X=x
    ]
    =
    \frac{
        p^*(x)\pi^*(a|x)
        \int_{\mathcal Z_1}
        p^{(2)}(w|z_1,a,x)
        p^{(2)}(z_1|x)
        \mathrm dz_1
    }{
        p^{(1)}(w,a,x)
    },
\end{equation*}
and define
$b_2:\mathcal Z_1\times\mathcal A\times\mathcal X\to\mathbb R$
and
$b_3:\mathcal A\times\mathcal X\to\mathbb R$ by
\begin{equation*}
    b_2(z_1,a,x)
    =
    \frac{
        p^{(2)}(z_1|x)p^*(x)\pi^*(a|x)
    }{
        p^{(2)}(z_1,a,x)
    },
    \qquad
    b_3(a,x)
    =
    \frac{
        p^*(x)\pi^*(a|x)
    }{
        p^{(2)}(a,x)
    }.
\end{equation*}
Then, on event $\mathcal E$, which occurs with probability at least
$1-6\xi$,
\begin{equation*}
    \mathrm{SubOpt}(\widehat{\pi})
    \overset{\mathcal E}{\leq}
    \sum_{k=1}^3
    \Vert b_k\Vert_{\mu_k,2}
    \left(
        \mathcal O(\varepsilon_\Theta)
        +
        \mathcal O(\varepsilon_\mathcal H)
        +
        \mathcal O(\sqrt{e_\mathcal D})
        +
        \mathcal O(\eta)
    \right).
\end{equation*}
\end{theorem}

\begin{proof}
See \S\hyperref[sec:B.1]{B.1} and
\S\hyperref[sec:B.4]{B.4} for a detailed proof.
\end{proof}

Theorems~\ref{thm:7.2}--\ref{thm:7.4} reveal a common structure in the finite-sample performance of CAT-CCB. Across both unconfounded and confounded settings, the suboptimality is controlled by the approximation errors $\varepsilon_{\mathcal H}$ and $\varepsilon_\Theta$, the statistical uncertainty characterized by $e_{\mathcal D}$ and $\eta$, and the $L^2$ norms of the weighting functions $b_k$, which quantify the difficulty of transporting source-domain moment errors to the target policy value. Consequently, when the bridge and test-function classes are well specified and the statistical errors vanish with increasing sample size, CAT-CCB approaches the optimal target policy provided these weighting norms remain finite.

In Appendix~\ref{sec:E}, we further specialize the bridge and test functions to bounded finite-dimensional linear classes and derive explicit bounds on their critical radii. In particular, letting $$N:=\min_{k\in[K]}|\mathcal D_k|$$ denote the minimum sample size across the offline datasets, and assuming that the relevant function classes have dimensions of order $d$. Then we obtain $$\eta=\widetilde{\mathcal O}\left(\sqrt{\frac{d}{N}}\right).$$ Under exact realizability and compatibility, and with bounded weighting norms, this yields
$$\mathrm{SubOpt}(\widehat{\pi})=\widetilde{\mathcal O} \left(\sqrt{\frac{d}{N}}\right).$$
Therefore, when the relevant function-class dimensions are treated as fixed, CAT-CCB achieves the standard $\widetilde{\mathcal O}(N^{-1/2})$ statistical rate.

\subsection{Analysis for CAT-CMDP}
\label{sec:7.2}

In this section, we establish the finite-sample suboptimality guarantee for the CAT-CMDP algorithm proposed in \S\ref{sec:6.3}. We first characterize the distribution shift between the offline source datasets and the target trajectories through the following coverage coefficient.

\begin{definition}[Distribution shift coefficient (DSC)]
\label{def:7.1}
The distribution shift coefficient between the distribution induced by a policy $\pi$ in the target MDP $\mathscr{M}^*$ and the marginal visitation distribution in a source MDP $\mathscr{M}^{(i)}$ is defined as
\begin{equation*}
    C_i^{\pi}
    :=
    \sup_{h\in[H]}
    \E^{(i)}
    \left[
        \left(
            \frac{d^\pi_h(S_h,A_h)}
            {p^{(i)}(S_h,A_h)}
        \right)^2
    \right],
\end{equation*}
where $d^\pi=\{d_h^\pi\}_{h=1}^H$ denotes the collection of state-action visitation measures induced by policy $\pi$ in the target MDP $\mathscr{M}^*$.
\end{definition}

The coefficient $C_i^\pi$ measures the mismatch between the state-action visitation distribution induced by $\pi$ in the target MDP and the observational distribution in source domain $\mathscr{M}^{(i)}$. In particular,
\[
    C_i^\pi
    =
    1+
    \sup_{h\in[H]}
    D_{\chi^2}
    \left(
        d_h^\pi
        \,\middle\|\,
        p_h^{(i)}
    \right),
\]
where $D_{\chi^2}$ denotes the Pearson $\chi^2$-divergence. Thus, $C_i^\pi=1$ corresponds to perfect distributional alignment, while a larger value indicates greater mismatch between the target visitation distribution and the available offline data.

\begin{assumption}[Distribution shift]
\label{assumption:7.4}
We impose the following conditions on the distribution shift between the source and target MDPs:
\begin{itemize}
    \item[(i)] \textbf{Partial coverage.}
    The visitation distribution induced by the optimal target policy $\pi^*$ is covered by both source MDPs:
    \[
        C_1^*:=C_1^{\pi^*}<\infty,
        \qquad
        C_2^*:=C_2^{\pi^*}<\infty.
    \]

    \item[(ii)] \textbf{Conditional distribution shift.}
    The conditional distribution of $W_h$ in $\mathscr{M}^{(1)}$ sufficiently covers its target-domain counterpart. Specifically,
    \begin{align*}
        C_w&=\sup_{h\in[H],\,s_h\in\mathcal{S}}
        \E^{(1)}
        \left[\left(
                \frac{p^*(W_h\mid S_h)}{p^{(1)}(W_h\mid S_h)}
            \right)^2\middle|S_h=s_h\right]\\
        &=1+\sup_{h\in[H],\,s\in\mathcal S}
        D_{\chi^2}
        \left(
            p^*(W_h\mid S_h=s)
            \,\middle\|\,
            p^{(1)}(W_h\mid S_h=s)
        \right)<\infty.
    \end{align*}
\end{itemize}
\end{assumption}

Assumption~\ref{assumption:7.4} (i) is a partial-coverage condition similar to those commonly used in offline reinforcement learning \citep{kidambi2021morel,pmlr-v139-jin21e}: it requires coverage only of the optimal target policy rather than uniformly over all policies. Assumption~\ref{assumption:7.4} (ii) controls the additional conditional distribution shift of the negative outcome $W_h$ between $\mathscr{M}^{(1)}$ and the target domain. Under Assumption~\ref{assumption:5.2}, the target mechanism $p^*(W_h\mid S_h)$ is shared with $\mathscr{M}^{(2)}$, so $C_w$ characterizes the discrepancy in this mechanism across the two source domains.

We next study the stability of the conditional moment equations for reward and transition bridge estimation. In source MDP $\mathscr{M}^{(1)}$, let $\rho_h$ be the joint distribution of $(S_h,Z_h,A_h,W_h)$, and let the same notation denote its relevant marginals when there is no ambiguity. Define the ambient $L^2$ spaces
\begin{equation*}
\mathcal{B}_{s,a,w,h}:=L^2(\mathcal{S}\times\mathcal{A}\times\mathcal{W};\rho_h),\quad \mathcal{B}_{s,z,h}:=L^2(\mathcal{S}\times\mathcal{Z};\rho_h),
\end{equation*}
equipped with the corresponding $L^2$ norm $\Vert\cdot\Vert_{\rho_h}$.

Under the linear function approximation in Assumption~\ref{assumption:6.2}, it suffices to consider the linear span of the bridge-function class,
\begin{equation*}
    \mathcal{V}_h
    :=
    \left\{
        \sigma(s,a,w)
        =
        \psi(s,a,w)^\top\zeta:
        \zeta\in\mathbb{R}^d
    \right\}
    \subseteq
    \mathcal{B}_{s,a,w,h}.
\end{equation*}
Since $\mathcal{V}_h$ is finite-dimensional, it is a closed subspace of $\mathcal{B}_{s,a,w,h}$ and is therefore itself a Banach space. We define the conditional moment operator $\mathcal{T}_h:\mathcal{V}_h\to\mathcal{B}_{s,z,h}$ by
\begin{equation*}
    (\mathcal{T}_h\sigma)(s,z)
    :=
    \E^{(1)}
    [\sigma(S_h,A_h,W_h)\mid S_h=s,Z_h=z],
    \qquad
    s\in\mathcal{S},\ z\in\mathcal{Z}.
\end{equation*}
This is the conditional expectation operator underlying the IV moment equations in Lemma~\ref{lemma:5.1}.

\begin{lemma}[Properties of the conditional moment operator]
\label{lemma:7.1}
Under Assumption~\ref{assumption:5.3}, for any $h\in[H]$, the operator $\mathcal{T}_h$ satisfies:
\begin{itemize}
    \item[(i)] \textbf{Linearity.}
    For $\sigma_1,\sigma_2\in\mathcal{V}_h$ and $\alpha,\beta\in\mathbb{R}$,
    \[
        \mathcal{T}_h(\alpha\sigma_1+\beta\sigma_2)
        =\alpha\mathcal{T}_h\sigma_1
        +\beta\mathcal{T}_h\sigma_2.
    \]

    \item[(ii)] \textbf{Boundedness.}
    For any $\sigma\in\mathcal{V}_h$,
    \[
        \Vert\mathcal{T}_h\sigma\Vert_{\rho_h}
        \leq\Vert\sigma\Vert_{\rho_h}.
    \]

    \item[(iii)] \textbf{Injectivity.}
    For any $\sigma_1,\sigma_2\in\mathcal{V}_h$,
    $\mathcal{T}_h\sigma_1=\mathcal{T}_h\sigma_2$ implies
    $\sigma_1=\sigma_2$.
\end{itemize}
\end{lemma}

The linearity in Lemma~\ref{lemma:7.1} (i) follows directly from the definition of conditional expectation, while (ii) follows from conditional Jensen's inequality. Property (iii) is a consequence of the IV completeness condition in Assumption~\ref{assumption:5.3} (iii). Thus, $\mathcal{T}_h$ is a bounded and injective linear operator on the bridge-function span. Let
\begin{equation*}
    \mathfrak{R}_h
    :=
    \mathcal{T}_h(\mathcal{V}_h)=\{\mathcal{T}_h\sigma:\sigma\in\mathcal{B}_{s,a,w,h}\}=
    \subseteq
    \mathcal{B}_{s,z,h}
\end{equation*}
denote its range. Since $\mathcal{V}_h$ is finite-dimensional and $\mathcal{T}_h$ is linear, $\mathfrak{R}_h$ is also finite-dimensional and hence closed in $\mathcal{B}_{s,z,h}$. Therefore, $\mathfrak{R}_h$ is a Banach space, and $\mathcal{T}_h:\mathcal{V}_h\to\mathfrak{R}_h$ is a bounded linear bijection. The bounded inverse theorem then yields the following stability property.

\begin{lemma}[Stability of the conditional moment operator]
\label{lemma:7.2}
For each $h\in[H]$, there exists a constant $0<L_h\leq1$ such that
\begin{equation*}
    \Vert\mathcal{T}_h\sigma\Vert_{\rho_h}
    \geq
    L_h\Vert\sigma\Vert_{\rho_h},
    \qquad
    \forall\sigma\in\mathcal{V}_h.
\end{equation*}
Define
\begin{equation*}
    L:=\min_{h\in[H]}L_h>0.
\end{equation*}
\end{lemma}

The constant $L$ quantifies the stability of recovering bridge functions from their conditional moments within the linear function class. A smaller $L$ corresponds to a less informative conditional moment equation and leads to greater statistical uncertainty.

We are now ready to characterize the finite-sample suboptimality of CAT-CMDP.

\begin{theorem}[Convergence of suboptimality in CAT-CMDP]
\label{thm:7.5}
Set the confidence-set thresholds in Algorithm~\ref{alg:2} as
\begin{align}
    \alpha_1
    &=
    \frac{C_1d}{K_1}
    \log\left(\frac{HdK_1}{\delta}\right),
    \notag\\
    \alpha_2
    &=
    \frac{C_2d}{K_1}
    \log\left(\frac{Hdd_\mathcal{S}K_1}{\delta}\right),
    \notag\\
    \alpha_3
    &=
    C_3
    \sup_{h\in[H]}
    \sqrt{
        m+d_\mathcal{W}
        +
        \mathrm{rank}(\Sigma_h^{(2)})
        \left\{
            \mathrm{rank}(\Sigma_h^{(2)})
            +
            \log\left(\frac{c_2H}{\delta}\right)
        \right\}
        \log(1+K_2)
    },
    \notag
\end{align}
where $C_1,C_2,C_3$ and $c_2$ are universal constants,
$m$ is defined in Assumption~\ref{assumption:6.1}(i), and
\[
    \Sigma_h^{(2)}
    :=
    \E^{(2)}
    [\phi(S_h)\phi(S_h)^\top].
\]
Suppose that Assumptions~\ref{assumption:5.1},
\ref{assumption:5.2},
\ref{assumption:5.3},
\ref{assumption:6.1},
\ref{assumption:6.2}, and
\ref{assumption:7.4} hold.
Then, with probability at least $1-\delta$, the suboptimality of Algorithm~\ref{alg:2} satisfies
\begin{align}
    \mathrm{SubOpt}(\widehat{\pi})
    \leq
    \mathcal{O}
    \left\{
        H^2
        \left(\sqrt{
        \frac{C_1^*C_wdd_\mathcal{S}}{L^2K_1}
        \log\left(\frac{Hdd_\mathcal{S}K_1}{\delta}\right)
        }
        +\sqrt{C_2^*\overline{R}d_\mathcal{S}
        \frac{m+d_\mathcal{W}+\overline{R}\log(1+K_2)}{K_2}}
        \right)
    \right\},
    \label{eq:7.23}
\end{align}
where
\begin{equation*}
    \overline{R}
    :=
    \max_{h\in[H]}
    \mathrm{rank}(\Sigma_h^{(2)})
    \left\{
        \mathrm{rank}(\Sigma_h^{(2)})
        +\log\left(\frac{c_2H}{\delta}\right)
    \right\}.
\end{equation*}
\end{theorem}

\begin{proof}
See \S\ref{sec:C} for a detailed proof.
\end{proof}

The two terms in \eqref{eq:7.23} correspond to uncertainty from the two source datasets. The first arises from estimating the reward and transition bridges using $\mathcal{D}^{(1)}$ and depends on the coverage coefficient $C_1^*$, the shift in $W_h$ measured by $C_w$, and the stability constant $L$. The second arises from estimating the negative-outcome mechanism using $\mathcal{D}^{(2)}$ and depends on $C_2^*$ and the rank of its feature covariance.

Treating $d,d_\mathcal{S},d_\mathcal{W},H$, the coverage and stability constants, and $\overline{R}$ as fixed, the suboptimality of CAT-CMDP converges to zero at the rate
\[
    \max\left\{
        \frac{1}{\sqrt{K_1}}, \frac{1}{\sqrt{K_2}}
    \right\}
\]
with high probability. Thus, as both source datasets grow, CAT-CMDP consistently approaches the performance of the optimal target policy under the stated assumptions.

For a more interpretable specialization, suppose $K_1=K_2=K$. If the feature covariance associated with
$\phi:\mathbb{R}^{d_{\mathcal S}}\to\mathbb{R}^d$
has low rank, the dependence on the ambient feature dimension can be improved. In particular, if we have
\[
    \mathrm{rank}(\Sigma_h^{(2)})\leq\widetilde{\mathcal{O}}(\sqrt{d}),
    \qquad\forall h\in[H],
\]
then $\overline{R}\leq\widetilde{\mathcal{O}}(d)$, and
\begin{equation*}
    \mathrm{SubOpt}(\widehat{\pi})
    \leq\widetilde{\mathcal{O}}\left(\sqrt{\max\left(
    \frac{C_1^*C_w}{L^2},C_2^*\right)\frac{H^4dd_\mathcal{S}d_\mathcal{W}}{K}}\right),
\end{equation*}
where $\widetilde{\mathcal{O}}$ suppresses logarithmic factors.

%% file: chap8.tex
\section{Numerical Experiments}
\label{sec:experiments}
%
%

We evaluate the confounded contextual bandit and confounded MDP instantiations of our framework on a deliberately transparent discrete structural causal model. In our experiments, finite context and action spaces make the true reward bridge and the target causal mean available in closed form, the empirical minimax loss exactly quadratic, the confidence set exactly an ellipsoid, and the full policy class enumerable. Every reported quantity is therefore measured against ground truth rather than against an estimate, and the mechanism behind each result can be audited directly. A complete specification of the model, with all coefficients and evaluation protocols, is given in Appendices~\S\ref{app:scm}-\ref{app:cmdp}.


\subsection{Confounded Contextual Bandits}\label{sec:exp-ccbneg}

\paragraph{Data-generating process.}
We use a discrete CCB-Neg model specified in \S\ref{sec:3.1} with $K$ context strata $X \in \{1,\dots,K\}$ and $M$ actions $A \in \{0,\dots,M-1\}$, a negative outcome $W$, an action instrument $Z_A \in \mathbb{R}^M$, a negative-outcome instrument $Z_W \in \mathbb{R}$, and an unobserved confounder $U$. For each domain $d \in \{\star,1,2\}$, the structural equations are
\begin{equation*}
\begin{cases}
W = \omega^{(d)}_X + \kappa^{(d)}_W Z_W + \rho_W U + \sigma_W \varepsilon_W, \\[2pt]
A = \displaystyle\argmax_a
\bigl(
\alpha^{(d)}_{a,X}
+\kappa_A Z_{A,a}
+\rho_a U
+\sigma_A\varepsilon_{A,a}
\bigr), \\[2pt]
Y = f^{(d)}(W,A,X)+\rho_YU+\sigma_Y\varepsilon_Y.
\end{cases}
\end{equation*}
Thus, $U$ confounds $A$, $W$, and $Y$, while the action does not cause $W$, making $W$ a valid negative outcome. The instruments are independent of $U$ and of the structural outcome noise. The reward surface is cell-saturated and linear in $W$:
\[
f^{(d)}(w,a,x)
=\mu^{(d)}_x+b^{(d)}_xw+\tau^{(d)}_{x,a}+\lambda^{(d)}_aw,
\]
The two source domains realize the complementary discrepancies of Assumption~\ref{assumption:3.1}:
\begin{itemize}
    \item Source~$1$ shifts the $W$-mechanism and the behavior policy while preserving the reward mechanism, so $\Delta^{(1)}=\{A,W\}$ and $f^{(1)}=f^\star$;
    \item Source~$2$ shifts the reward mechanism and the behavior policy while preserving the $W$-mechanism, so $\Delta^{(2)}=\{A,Y\}$ and $p^{(2)}(w\mid x)=p^\star(w\mid x)$.
\end{itemize}
Neither source alone transports to the target. Their identified components combine through
\begin{equation}
\label{eq:exp-identification}
\begin{cases}
h_1^\star(w,a,x) = f^\star(w,a,x),\\
g^\star(a,x)
=\mathbb{E}^{(2)}
\bigl[
h_1^\star(W,a,x)
\mid X=x
\bigr],
\end{cases}
\end{equation}
where the reward bridge $h_1^\star=f^\star$ is identified from source~$1$ through the instrumental moment restriction, while source~$2$ supplies the target-compatible conditional distribution of $W$. Since
\[
\mathbb{E}^{(2)}[W\mid X=x]=\omega_x^\star,
\]
the target causal mean is available in closed form:
\[
g^\star(a,x)
=
\mu^\star_{x}+\tau^\star_{x,a}
+ (b_x^\star + \lambda^\star_a)\omega_x^\star.
\]
Consequently, the optimal target policy $\pi^\star(x)=\argmax_a g^\star(a,x)$ is also available exactly.

\paragraph{Estimation and policy learning.}
Let $\beta$ collect the cell coefficients of the reward bridge $h_\beta(w,a,x)$, and let $\gamma$ collect the target causal effects $g_\gamma(a,x)$. We estimate $\theta=(\beta,\gamma)$ using two sets of moment restrictions: source~$1$ identifies the reward bridge through the instruments $(Z_A,Z_W)$, while source~$2$ transports the bridge to the target by integrating it against the target-compatible conditional distribution of $W$. The complete integral-equation system, test-function choice, stacked empirical moments, and derivation of closed-form minimax loss are provided in Appendix~\ref{app:ccb-neg-estimation}.

Let $\widehat\theta=(\widehat\beta,\widehat\gamma)$ denote the minimizer of the empirical minimax loss, and define
\[
\widehat Q
=
\widehat J^\top
\widehat\Omega^{-1}
\widehat J,
\]
where $\widehat J$ is the empirical Jacobian of the stacked moment equations and $\widehat\Omega$ is the empirical test-function Gram matrix induced by the quadratic $L^2$ penalty. Their complete expressions are given in \eqref{eq:complete-jacobian-q} and
\eqref{eq:complete-test-gram}, respectively; see
Appendix~\S\ref{app:ccb-neg-estimation}. The matrix $\widehat Q$ is the curvature matrix of the empirical
minimax loss. Because this loss is quadratic, its radius-$r$ loss-level uncertainty set is the ellipsoid
\begin{equation}
\label{eq:joint-loss-ellipsoid-main}
\mathcal C_\theta(r)
=
\left\{
\theta:
\frac{1}{2}
(\theta-\widehat\theta)^\top
\widehat Q
(\theta-\widehat\theta)
\le r
\right\}.
\end{equation}
The plug-in CAT-Greedy policy selects the action maximizing the estimated target causal mean
$\widehat g(a,x)=\phi_g(a,x)^\top\widehat\gamma$. CAT-Pess instead evaluates each policy under the worst-case causal mean in the exact projection of $\mathcal C_\theta(r)$ onto the $\gamma$ coordinates. We use only this single pessimistic variant and treat the radius $r$ as a sensitivity parameter. The primary experimental comparisons therefore focus on transportation across domains and adjustment for latent confounding.

\paragraph{Methods.}
The identification experiment compares the full estimator, \textsc{CAT}, against four one-ingredient ablations: \textsc{Naive-Pool} regresses $Y$ on CATE cell indicators over the pooled sources with no causal adjustment; \textsc{Source1-Only} and \textsc{Source2-Only} apply the instrumental adjustment within a single source and so omit transport; \textsc{Transport-NoIV} performs the transport split but replaces the instrument test functions with cell indicators and so omits deconfounding. 
We report suboptimality $\mathrm{SubOpt} = v^{\pi^\star} - v^{\hat\pi}$ and the CATE root-mean-square error $\mathrm{RMSE} = (\mathbb{E}[(\hat g - g^\star)^2])^{1/2}$, averaged over $60$ to $120$ seeds with error bars of one standard error. The RMSE reference measure draws $X$ from the target context law and weights all $(x, a)$ cells uniformly.

\subsubsection{Identification and the necessity of each component}
\label{sec:exp-identification}

We vary the per-source sample size $n \in \{250, \dots, 16000\}$ and record both metrics for every method. Figure~\ref{fig:identification} shows that \textsc{CAT} is the only method whose CATE error continues to decrease across the tested range, falling from roughly $1.7$ at $n = 250$ to $0.085$ at $n = 16000$, while each ablation settles onto a distinct nonzero floor. These floors are consistent with a bias that does not diminish with sample size: \textsc{Source2-Only} settles near $0.37$, \textsc{Naive-Pool} near $0.31$, \textsc{Transport-NoIV} near $0.30$, and \textsc{Source1-Only} near $0.24$.

Within this controlled data-generating process the paired ablations isolate the empirical contribution of each ingredient. The separation between \textsc{CAT} and \textsc{Transport-NoIV}, which differ only in the first stage, reflects the role of deconfounding. The separation between \textsc{CAT} and \textsc{Source1-Only}, which differ only in whether the $W$-average uses the source with the target-compatible $W$-law, reflects the role of transport. \textsc{Source2-Only} is the least accurate: applying a valid instrumental adjustment inside the source whose reward mechanism is wrong produces a confidently fitted but incorrect structure. These results provide empirical evidence consistent with the identification result of \eqref{eq:exp-identification}; the identification and consistency statements themselves are established by the theory.

\begin{figure}[t]
  \centering
  \begin{minipage}[t]{0.49\linewidth}
    \centering
    \includegraphics[width=\linewidth]{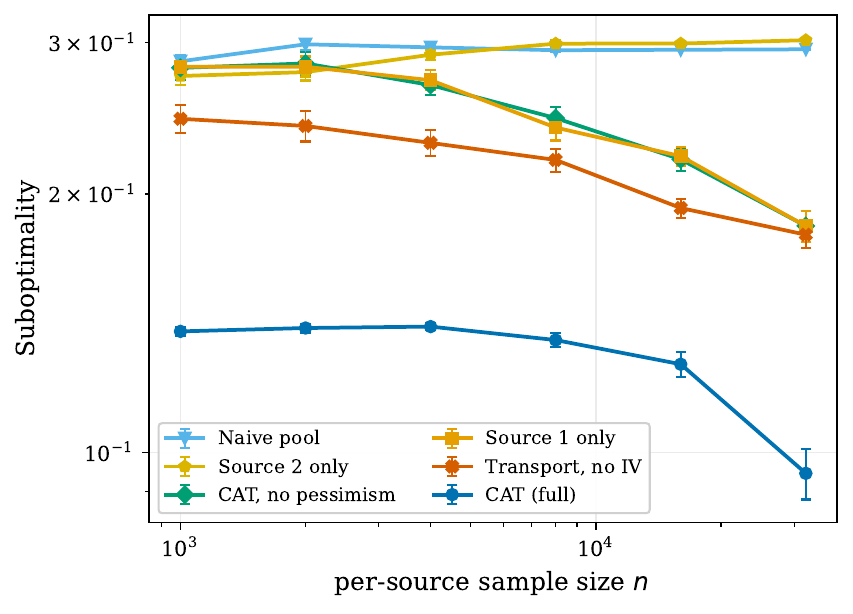}\\
    {\footnotesize (a) Suboptimality}
  \end{minipage}
  \hfill
  \begin{minipage}[t]{0.49\linewidth}
    \centering
    \includegraphics[width=\linewidth]{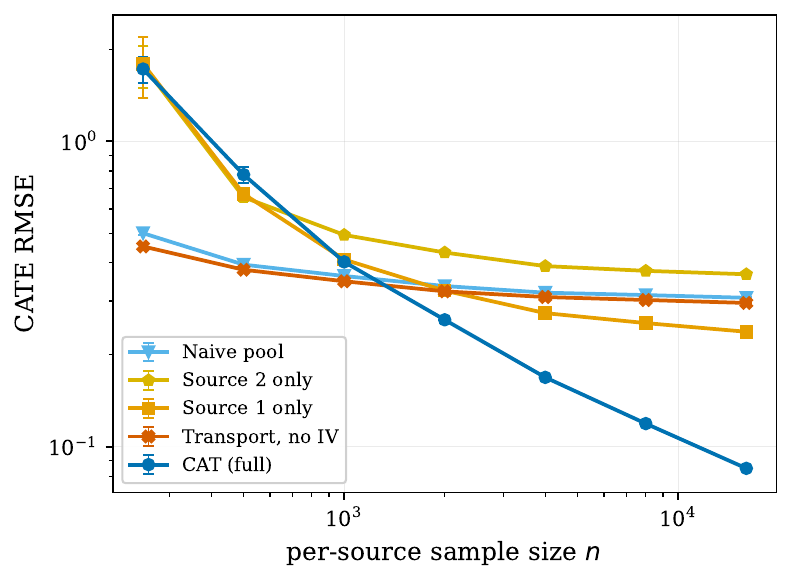}\\
    {\footnotesize (b) CATE RMSE}
  \end{minipage}
  \caption{%
    Identification and component necessity as a function of per-source sample size $n$ (log-log axes). The full estimator \textsc{CAT} decreases across the tested range; each ablation settles onto a bias floor set by the ingredient it omits. Curves show the mean over $60$ seeds and error bars show one standard error.}
  \label{fig:identification}
\end{figure}

\subsubsection{Leading quantities of the bound}
\label{sec:exp-factors}

The suboptimality bound of Theorem~\ref{thm:7.3} factors as a distribution-shift coefficient times a sum of an approximation error, a confidence-set term, and a critical radius, with the instrument informativeness entering the last two. Rather than varying informal proxies, we compute each theoretical quantity numerically and plot the suboptimality of \textsc{CAT} against the quantity itself (Figure~\ref{fig:sensitivity}).

\paragraph{Instrument strength.}
We vary the instrument loading $\kappa_A$ and plot against the smallest singular value $\sigma_{\min}$ of the whitened source-1 moment operator $\mathbb{E}[\psi(Z)\,\phi_h(W, A, X)^\top]$. Suboptimality decreases as $\sigma_{\min}$ grows, which is the multivariate analogue of a vanishing denominator in the scalar instrumental estimator: a less informative instrument shrinks the smallest singular value and amplifies estimation noise along the poorly identified directions.

\paragraph{Misspecification.}
We add a reward component quadratic in $W$ that the cell-linear bridge cannot represent, and plot against the best-in-class projected error $\inf_{h \in \mathcal{H}} \|T h\|_{\mu,2}$ computed at the population level. Suboptimality rises smoothly as this error grows, so the approximation term enters the bound as a gradual offset rather than a discontinuity.

\paragraph{Distribution shift.}
We concentrate the target context law on source-sparse strata and plot against the computed distribution-shift coefficient $\sum_k \|b_k\|_{\mu_k, 2}$, the target-policy reweighting norm over the two sources. Suboptimality grows with the coefficient. We hold the sample size small so that coverage is genuinely stressed; at larger sample sizes the same coefficient multiplies a smaller statistical error, so its contribution to the realized suboptimality becomes correspondingly small rather than disappearing.

\begin{figure}[t]
  \centering
  \includegraphics[width=\linewidth]{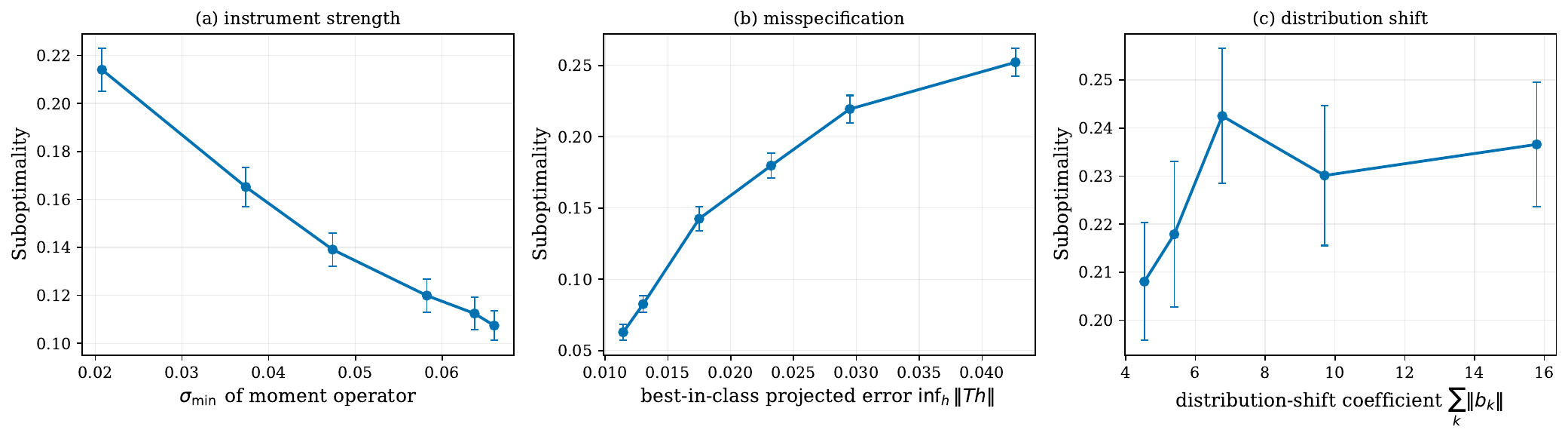}
  \caption{%
    Suboptimality of \textsc{CAT} plotted against the actual bound quantities. (a) Against the smallest singular value of the moment operator; weaker instruments raise the error. (b) Against the best-in-class projected error; misspecification enters as a gradual offset. (c) Against the distribution-shift coefficient $\sum_k\|b_k\|$. Error bars show one standard error over seeds.}
  \label{fig:sensitivity}
\end{figure}

\subsection{Confounded MDP}
\label{sec:exp-cmdp}

The CCB-Neg experiments above test the single-stage instantiation of the framework. The sequential methods of Section~\ref{sec:5} raise two questions the bandit model cannot address: whether the transport-plus-deconfounding construction still identifies the target value once it must be applied at every stage, and whether the confounding bias of the baselines compounds with the horizon. We therefore repeat the identification and pessimism analyses on a discrete confounded MDP, built to the same standard of auditability as the bandit model. The full specification is given in Appendix~\ref{app:cmdp}.

\paragraph{Model.}
We use a finite confounded MDP with $|\mathcal{S}| = 4$ states, $|\mathcal{A}| = 4$ actions, and horizon $H = 5$. At each stage an unobserved confounder $U$ enters the behaviour policy, the reward, and the transition, so a direct regression of reward or next state on the realized state-action pair is biased. A negative-outcome proxy $W$ and instruments $(Z_A, Z_W)$ play the same structural roles as in the bandit model, now at each stage. Two source domains realize the complementary discrepancies of Assumption~\ref{assumption:3.1}: source~$1$ shares the reward and transition mechanisms with the target but shifts the proxy mechanism and the behaviour policy, and source~$2$ shares the proxy mechanism but shifts the reward and transition mechanisms. The target reward means, the target transition kernel, and hence the optimal value are computed in closed form by marginalizing the structural equations over $(W, U)$, so suboptimality is measured against the exact optimum.

\paragraph{Estimator.}
At each stage the reward bridge and the transition bridge are identified in source~$1$ through the instrumental moment restriction and transported to the target proxy law supplied by source~$2$, exactly as in \eqref{eq:exp-identification}. The resulting stage-wise reward means $\hat R_h$ and transition kernels $\hat P_h$ are combined by backward induction to produce a policy, denoted \textsc{CAT-CMDP}. Because the structural mechanisms are stationary across stages, we pool the moment equations over the $H$ stages when forming the bridge, which cuts the per-stage variance by a factor of order $H$; this is the primary estimator. We compare against the same ablation family as the bandit case: \textsc{Transport-NoIV} performs the transport split without the instrumental deconfounding, \textsc{Source1-Only} omits transport, and \textsc{Naive-Pool} regresses on the pooled realized state-action pairs with no causal adjustment.

\paragraph{Results.}
Figure~\ref{fig:cmdp-identification} varies the per-source trajectory count $n$ and records suboptimality. \textsc{CAT-CMDP} is again the only method whose suboptimality decreases across the tested range, from $0.119$ at $n = 1000$ to $0.005$ at $n = 32000$, while each ablation settles onto a bias floor: \textsc{Transport-NoIV} near $0.055$, \textsc{Source1-Only} near $0.02$ to $0.04$, and \textsc{Naive-Pool} near $0.44$. The separation reproduces the bandit conclusions in the sequential setting: deconfounding is what separates \textsc{CAT-CMDP} from \textsc{Transport-NoIV}, and transport is what separates it from \textsc{Source1-Only}. 

\begin{figure}[t]
  \centering
  \includegraphics[width=0.6\linewidth]{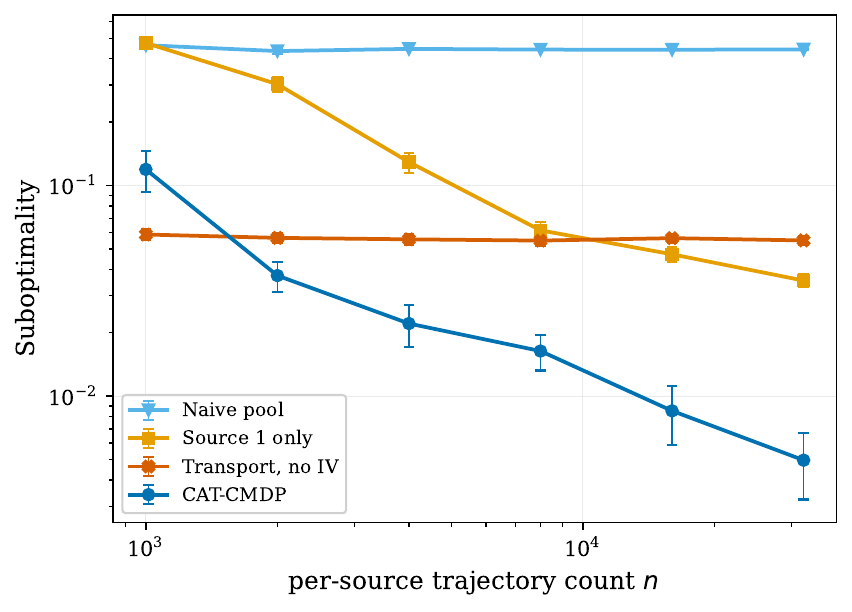}
  \caption{%
    Confounded MDP. (a) Suboptimality against per-source trajectory count $n$ (log-log axes): \textsc{CAT-CMDP} decreases while each ablation settles onto a bias floor. (b) Suboptimality against episode horizon $H$ at fixed sample size (log-scale $y$-axis): \textsc{CAT-CMDP} stays flat while the confounding bias of the baselines compounds with the horizon. Curves show the mean over $40$ seeds and error bars show one standard error.}
  \label{fig:cmdp-identification}
\end{figure}

%% file: chap9.tex
\section{Conclusion}

In this paper, we develop a framework for transportable offline policy learning under domain discrepancy and unmeasured confounding.  Our approach combines causal transportability with bridge-function identification to recover target-domain causal quantities from multiple heterogeneous source datasets. For contextual bandits, we formulate the resulting identification problems as systems of conditional moment equations and develop the CAT-CCB algorithm, which estimates the bridge functions through minimax optimization and learns a policy pessimistically over a confidence set of plausible CATEs. We further extend this framework to confounded episodic MDPs, where CAT-CMDP combines bridge-function estimation with model estimation to construct confidence sets for the target reward and transition mechanisms and optimizes the policy against the least favorable plausible target model.

Our theoretical analysis establishes finite-sample guarantees that characterize the effects of function approximation, statistical estimation, source-target distribution shift, offline coverage, and auxiliary-variable informativeness. Numerical experiments support the effectiveness of the overall framework. These results show how complementary source domains can support causal identification and target-domain policy learning under explicit structural, completeness, and coverage assumptions.

%% file: append1.tex
\section{Proofs for CATE Identification}
\subsection{Proof of Lemma \ref{lemma:3.1}}
\label{sec:A.1}
\begin{proof} Let $\mathbf{Z}=\{Z_1,Z_2\}$. For the structural reward $Y$, we have
\begin{align*}
\E^{(1)}[Y|\mathbf{Z}]
&= \E^{(1)}\left[\E^{(1)}[Y|X,\mathbf{Z}]|\mathbf{Z}\right]\\
&= \E^{(1)}\left[\E^{(1)}[f(W,A,X) + \epsilon|X,\mathbf{Z}]|\mathbf{Z}\right]\\
&= \E^{(1)}\left[f(W,A,X) + \E^{(1)}\left[\E^{(1)}[\epsilon|X,\mathbf{U},\mathbf{Z}]|X,\mathbf{Z}\right]|\mathbf{Z}\right]\\
&= \E^{(1)}\left[f(W,A,X) + \E^{(1)}\left[\E^{(1)}[\epsilon|X,\mathbf{U}]|X,\mathbf{Z}\right]|\mathbf{Z}\right] &(\text{By}\ \mathbf{Z}\indep\epsilon\,|\,X,\mathbf{U})\\
&= \E^{(1)}\left[f(W,A,X) + \E^{(1)}\left[\epsilon|X\right]|\mathbf{Z}\right], &(\text{By}\ \mathbf{Z}\indep \mathbf{U}\,|\,X)
\end{align*}
then we have
\begin{equation}
\E^{(1)}[h_1(W,A,X)|\mathbf{Z}] = \E^{(1)}\left[f(W,A,X) + \E^{(1)}\left[\epsilon|X\right]|\mathbf{Z}\right]. \label{eq:A.2}
\end{equation}
For the CATE $g^{(1)}$, it holds
\begin{align}
g^{(1)}(w,a,x) &= \int_\mathcal{Y} yp^{(1)}\left(y|\mathrm{do}(w),\mathrm{do}(a),x\right)\mathrm{d}y\notag\\
&= \int_{\mathcal{Y}\times\mathcal{U}} yp^{(1)}(y|\mathrm{do}(w),\mathrm{do}(a),\mathbf{u},x)p^{(1)}(\mathbf{u}|\mathrm{do}(w),\mathrm{do}(a),x)\mathrm{d}y\mathrm{d}\mathbf{u} &(\text{By the law of total probability})\notag\\
&= \int_{\mathcal{Y}\times\mathcal{U}} yp^{(1)}(y|\mathbf{u},w,a,x)p^{(1)}(\mathbf{u}|x,\mathrm{do}(w),\mathrm{do}(a))\mathrm{d}y\mathrm{d}\mathbf{u} &(\text{By}\ Y\indep (A,W)\,|\,X,\mathbf{U}\ \text{in}\ \mathcal{G}_{\underline{AW}})\notag\\
&= \int_{\mathcal{Y}\times\mathcal{U}} yp^{(1)}(y|\mathbf{u},w,a,x)p^{(1)}(\mathbf{u}|x)\mathrm{d}y\mathrm{d}\mathbf{u} &(\text{By}\ \mathbf{U}\indep (A,W)\,|\,X\ \text{in}\ \mathcal{G}_{\overline{AW}})\notag\\
&= \E^{(1)}\left[\E^{(1)}[Y|\mathbf{U},W=w,A=a,X]|X=x\right]\notag\\
&= f(w,a,x) + \E^{(1)}\left[\E^{(1)}[\epsilon|\mathbf{U},W=w,A=a,X]|X=x\right]\notag\\
&= f(w,a,x) + \E^{(1)}\left[\E^{(1)}[\epsilon|\mathbf{U},X]|X=x\right] &(\text{By}\ \epsilon\indep (W,A)\,|\,\mathbf{U},X)\notag\\
&= f(w,a,x) + \E^{(1)}[\epsilon|X=x].\label{eq:A.3}
\end{align}
Then
\begin{equation}
\E^{(1)}[g^{(1)}(W,A,X)|\mathbf{Z}] = \E^{(1)}\left[f(W,A,X) + \E^{(1)}\left[\epsilon|X\right]|\mathbf{Z}\right]. \label{eq:A.4}
\end{equation}
Combining (\ref{eq:A.2}) and (\ref{eq:A.4}) yields that
$\E^{(1)}[g^{(1)}(W,A,X) - h_1(W,A,X)|\mathbf{Z}] \overset{\text{a.s.}}{=} 0.$ Following the IV completeness in Assumption \hyperref[assumption:3.1]{3.1}, we have
\begin{equation*}
h_1(W,A,X)=g^{(1)}(W,A,X)\ \ \mathrm{a.s.},
\end{equation*}
which concludes the proof.
\end{proof}

\subsection{Proof of Lemma \ref{lemma:3.2}}
\label{sec:A.2}
\begin{proof} Analogous to the proof of Lemma \hyperref[lemma:3.1]{3.1} we can show that $g^{(1)}(W,A,X)\overset{\mathrm{a.s.}}{=}h_1(W,A,X)$. Then with the transferring result, we have for the CATE $g^*$ that
\begin{align}
g^*(a,x) &= \int_{\mathcal{Y}} yp^*(y|\mathrm{do}(a),x)\mathrm{d}y\notag\\
&= \int_{\mathcal{Y}\times\mathcal{W}} yp^{(1)}(y|x,\mathrm{do}(w),\mathrm{do}(a)) p^{(2)}(w|x,\mathrm{do}(a))\mathrm{d}w\mathrm{d}y\notag\\
&= \int_{\mathcal{W}} h_1(w,a,x)p^{(2)}(w|x,\mathrm{do}(a))\mathrm{d}w\notag\\
&= \int_{\mathcal{W}\times\mathcal{U}} h_1(w,a,x)p^{(2)}(w|\mathbf{u},x,\mathrm{do}(a))p^{(2)}(\mathbf{u}|x,\mathrm{do}(a))\mathrm{d}\mathbf{u}\mathrm{d}w &(\text{By the law of total probability})\notag\\
&= \int_{\mathcal{W}\times\mathcal{U}} h_1(w,a,x)p^{(2)}(w|\mathbf{u},x,a)p^{(2)}(\mathbf{u}|x,\mathrm{do}(a))\mathrm{d}\mathbf{u}\mathrm{d}w &(\text{By}\ W\indep A\,|\,X,\mathbf{U}\ \text{in}\ \mathcal{G}_{\underline{A}})\notag\\
&=\int_{\mathcal{W}\times\mathcal{U}} h_1(w,a,x)p^{(2)}(w|\mathbf{u},x,a)p^{(2)}(\mathbf{u}|x)\mathrm{d}\mathbf{u}\mathrm{d}w. &(\text{By}\ \mathbf{U}\indep A\,|\,X\ \text{in}\ \mathcal{G}_{\overline{A}}) \label{eq:A.6}
\end{align}
Now we consider the moment equation \eqref{eq:med_h2}:
\begin{align*}
0&=\E^{(2)}[h_2(Z_2,A,X) - h_1(W,A,X)|Z_1,A,X]\\ &= \E^{(2)}\left[\E^{(2)}[h_2(Z_2,A,X) - h_1(W,A,X)|\mathbf{U},Z_1,A,X]|Z_1,A,X\right]\\
&= \E^{(2)}\left[\E^{(2)}[h_2(Z_2,A,X) -h_1(W,A,X)|\mathbf{U},X,A]|Z_1,A,X\right] &(\text{By}\ Z_1\indep (Z_2,W)\,|\,\mathbf{U},X,A)
\end{align*}
From the PV completeness in Assumption \hyperref[assumption:3.2]{3.2}, we have
\begin{equation}
\E^{(2)}[h_2(Z_2,A,X)|\mathbf{U},X,A] = \E^{(2)}[h_1(W,A,X)|\mathbf{U},X,A]\ \ \mathrm{a.s.}.\label{eq:A.8}
\end{equation}
Then
\begin{align}
\E^{(2)}[h_2(Z_2,a,x)|X=x] &= \int_{\mathcal{Z}_2} h_2(z_2,a,x)p^{(2)}(z_2|x)\mathrm{d}z_2\notag\\
&= \int_{\mathcal{Z}_2\times\mathcal{U}} h_2(z_2,a,x)p^{(2)}(z_2|\mathbf{u},x)p^{(2)}(\mathbf{u}|x)\mathrm{d}\mathbf{u}\mathrm{d}z_2\notag\\
&= \int_{\mathcal{Z}_2\times\mathcal{U}} h_2(z_2,a,x)p^{(2)}(z_2|\mathbf{u},x,a)p^{(2)}(\mathbf{u}|x)\mathrm{d}\mathbf{u}\mathrm{d}z_2 &(\text{By}\ Z_2\indep A\,|\,\mathbf{U},X)\notag\\
&\overset{\text{(\ref{eq:A.8})}}{=}\int_{\mathcal{W}\times\mathcal{U}} h_1(w,a,x)p^{(2)}(w|\mathbf{u},x,a)p^{(2)}(\mathbf{u}|x)\mathrm{d}\mathbf{u}\mathrm{d}w.\label{eq:A.9}
\end{align}
Combining (\ref{eq:A.6}) and (\ref{eq:A.9}), we have $\E^{(2)}[h_2(Z_2,a,x)|X=x] = g^*(a,x)$, which concludes the proof of Lemma \hyperref[lemma:3.2]{3.2}.
\end{proof}

\subsection{Proof of Lemma \ref{lemma:5.1}}
\label{sec:A.3}
\begin{proof}
We only show the proof of \eqref{eq:cmdp_s_bridge1} and \eqref{eq:cmdp_s_bridge2}. The proof of \eqref{eq:cmdp_r_bridge1} and \eqref{eq:cmdp_r_bridge2} is parallel. According to the structural equation \eqref{eq:5.6} in Assumption \ref{assumption:5.1}, we have
\begin{align*}
\E^{(1)}[S_{h+1}|S_h,Z_h] &= \E^{(1)}[f_h^s(S_h,A_h,W_h) + \epsilon_s|S_h,Z_h]\\
&= \E^{(1)}\left[f_h^s(S_h,A_h,W_h) + \E^{(1)}[\epsilon_s|U_h,S_h,Z_h]|S_h,Z_h\right]\\
&= \E^{(1)}\left[f_h^s(S_h,A_h,W_h) + \E^{(1)}[\epsilon_s|U_h,S_h]|S_h,Z_h\right]\\
&= \E^{(1)}\left[f_h^s(S_h,A_h,W_h)|S_h,Z_h\right] + \E^{(1)}[\E^{(1)}[\epsilon_s|U_h,S_h]|S_h]\\
&= \E^{(1)}\left[f_h^s(S_h,A_h,W_h)|S_h,Z_h\right] + \E^{(1)}[\epsilon_s|S_h].\label{eq:A.10}
\end{align*}
The third and fourth equalities follow from Assumption \hyperref[assumption:5.3]{5.3} (ii). With equation (\ref{eq:5.6}) and Assumption \hyperref[assumption:5.3]{5.3} (iii), we have
\begin{equation}
g_h^s(S_h,A_h,W_h) = f_h^s(S_h,A_h,W_h) + \E^{(1)}[\epsilon_s|S_h]\ \ \mathrm{a.s..}\label{eq:A.11}
\end{equation}
For the CATE in (\ref{eq:cmdp_s_bridge1}), we have
\begin{align}
\E^{(1)}[S_{h+1}(a_h)&|S_h=s_h,W_h=w_h] \notag\\
&= \int_{\mathcal{S}}s_{h+1}p^{(1)}(s_{h+1}|s_h,w_h,\mathrm{do}(a_h))\mathrm{d}s_{h+1}\notag\\
&=\int_{\mathcal{S}\times\mathcal{U}}s_{h+1}p^{(1)}(s_{h+1}|u_h,s_h,w_h,\mathrm{do}(a_h))p^{(1)}(u_h|s_h,w_h,\mathrm{do}(a_h))\mathrm{d}u_h\mathrm{d}s_{h+1}\notag\\
&=\int_{\mathcal{S}\times\mathcal{U}}s_{h+1}p^{(1)}(s_{h+1}|u_h,s_h,w_h,a_h)p^{(1)}(u_h|s_h,w_h)\mathrm{d}u_h\mathrm{d}s_{h+1}\notag\\
&=\int_{\mathcal{S}\times\mathcal{U}}s_{h+1}p^{(1)}(s_{h+1}|u_h,s_h,w_h,a_h)p^{(1)}(u_h|s_h)\mathrm{d}u_h\mathrm{d}s_{h+1}\notag\\
&=\E^{(1)}\left[\E^{(1)}[S_{h+1}|U_h,S_h,W_h=w_h,A_h=a_h]|S_h=s_h\right]\notag\\
&=f_h^s(s_h,a_h,w_h) + \E^{(1)}\left[\E^{(1)}[\epsilon_s|U_h,S_h,W_h=w_h,A_h=a_h]|S_h=s_h\right]\notag\\
&=f_h^s(s_h,a_h,w_h) + \E^{(1)}\left[\E^{(1)}[\epsilon_s|U_h,S_h]|S_h=s_h\right]\notag\\
&=f_h^s(s_h,a_h,w_h) + \E^{(1)}\left[\epsilon_s|S_h=s_h\right].\label{eq:A.12}
\end{align}
The third equality follows from $S_{h+1}\indep A_h\,|\,(U_h,S_h,W_h)$ in $\mathcal{G}_{\underline{A_h}}$ and $U_h\indep A_h\,|\,(S_h,W_h)$ in $\mathcal{G}_{\underline{A_h}}$. The fourth equality follows by Assumption \hyperref[assumption:5.1]{5.1} (iii) that $W_h\indep U_h\,|\,S_h$. The seventh equality follows by Assumption \hyperref[assumption:5.1]{5.1} (ii) that $\epsilon_s\indep (A_h,W_h)\,|\,S_h,U_h$.

By combining \eqref{eq:A.11} and \eqref{eq:A.12}, we can obtain \eqref{eq:cmdp_s_bridge1}:
\begin{equation}
\E^{(1)}[S_{h+1}(a_h)|S_h=s_h,W_h=w_h] = g_h^s(s_h,a_h,w_h)\ \ \mathrm{a.s.}.\label{eq:A.13}
\end{equation}
For the CATE in \eqref{eq:cmdp_s_bridge2}, we have
\begin{align*}
\E^*[S_{h+1}(a_h)|S_h=s_h] &= \int_\mathcal{S}s_{h+1}p^*(s_{h+1}|s_h,\mathrm{do}(a_h))\mathrm{d}s_{h+1}\\
&= \int_\mathcal{S\times\mathcal{W}}s_{h+1}p^{(1)}(s_{h+1}|s_h,w_h,\mathrm{do}(a_h))p^{(2)}(w_h|s_h)\mathrm{d}w_h\mathrm{d}s_{h+1}\\
&\overset{\mathrm{a.s.}}{=} \int_\mathcal{W} g_h^s(s_h,a_h,w_h)p^{(2)}(w_h|s_h)\mathrm{d}w_h.
\end{align*}
Then we conclude the proof of \eqref{eq:cmdp_s_bridge1} and \eqref{eq:cmdp_s_bridge2}.
\end{proof}

%% file: append2.tex
\section{Proof for the Convergence of CAT-CCB}
\subsection{Decomposition for the Suboptimality with Pessimism}\label{sec:B.1}
In this section, we investigate the suboptimality of the estimated policy $\widehat{\pi}$ with pessimism. Recall the definition of $\widehat{g}^\pi$ and $\hat{\pi}$:
\begin{align}
&\widehat{g}^\pi = \underset{g\in\mathrm{CI}_\mathcal{D}(e_\mathcal{D})}{\mathrm{argmin}}\ v(g,\pi),\label{eq:B.1}\\
&\widehat{\pi} = \underset{\pi\in\Pi}{\mathrm{argmax}}\min_{g\in\mathrm{CI}_\mathcal{D}(e_\mathcal{D})} v(g,\pi) = \underset{\pi\in\Pi}{\mathrm{argmax}}\  v(\widehat{g}^\pi,\pi).\label{eq:B.2}
\end{align}
Then on the event $\mathcal{E}$ defined in (\ref{eq:7.2}), the suboptimality of $\widehat{\pi}$ can be decomposed:
\begin{align*}
\mathrm{SubOpt}(\widehat{\pi}) &= v^{\pi^*} - v^{\widehat{\pi}}\\
&= \underbrace{v^{\pi^*} - v(\widehat{g}^{\pi^*},\pi^*)}_{\displaystyle\mathrm{(i)}}
+ \underbrace{v(\widehat{g}^{\pi^*},\pi^*) - v(\widehat{g}^{\widehat{\pi}},\widehat{\pi})}_{\displaystyle\mathrm{(ii)}}
+ \underbrace{v(\widehat{g}^{\widehat{\pi}},\widehat{\pi}) - v({g}^{*}_\mathcal{H},\widehat{\pi})}_{\displaystyle\mathrm{(iii)}}
+ \underbrace{v({g}^{*}_\mathcal{H},\widehat{\pi}) - v^{\widehat{\pi}}}_{\displaystyle\mathrm{(iv)}}.
\end{align*}
By definition of $\widehat{\pi}$, it holds $\mathrm{(ii)} \leq 0.$ On the event $\mathcal{E}$ we have $g^*_\mathcal{H}\in\mathrm{CI}_\mathcal{D}(e_\mathcal{D})$ by Theorem \ref{thm:7.1}, hence it follows from (\ref{eq:B.1}) that $\mathrm{(iii)}\overset{\mathcal{E}}{\lesssim}0$. The term (iv) can be bounded as below:
\begin{align*}
\mathrm{(iv)} &= \int_{\mathcal{A}\times\mathcal{X}} \left(g^*_\mathcal{H}(a,x) - g^*(x,a)\right)p^*(x)\widehat{\pi}(a|x)\mathrm{d}x\mathrm{d}a\\
&\leq \sup_{v\in\mathcal{V}}\Vert g_\mathcal{H} - g^*\Vert_{v,2}\leq\varepsilon_\mathcal{H},
\end{align*}
where the first equality holds by the definition of $\mathcal{V}$ that $\mathcal{V}=\{v:v(a,x) = p^*(x)\pi(a|x),\pi\in\Pi\}$, and the second from Assumption \ref{assumption:7.1} (ii). Consequently, we have
\begin{align}
	\mathrm{SubOpt}(\widehat{\pi}) \overset{\mathcal{E}}{\leq} v^{\pi^*} - v(\widehat{g}^{\pi^*},\pi^*) + \varepsilon_\mathcal{H}.\label{eq:B.5}
\end{align}

\subsection{Proof of Theorem \ref{thm:7.2}}\label{sec:B.2}
\begin{proof}
We first consider the unconfounded contextual bandit with negative outcome. For any $a\in\mathcal{A},\ x\in\mathcal{X}$, the error in the estimated conditional moment is given by
\begin{align}
&g^*(a,x) - \widehat{g}^{\pi^*}(a,x)\notag\\
&\quad=\E^{(2)}\left[\left(h_1^*(W,A^\prime,X) - \widehat{h}_1^{\pi^*}(W,A^\prime,X)\right) + \left(\widehat{h}_1^{\pi^*}(W,A^\prime,X) - \widehat{g}^{\pi^*}(A',X)\right)|A^\prime=a, X=x\right] \notag\\
&\quad=\E^{(2)}\left[h_1^*(W,A^\prime,X) - \widehat{h}_1^{\pi^*}(W,A^\prime,X)|A^\prime=a, X=x\right] - \mathcal{T}_2\widehat{h}^{\pi^*}(a, x)\notag\\
&\quad=\E^{(2)}\left[h_1^*(W,a,x) - \widehat{h}_1^{\pi^*}(W,a,x)|X=x\right] - \mathcal{T}_2\widehat{h}^{\pi^*}(a, x)\notag\\
&\quad=\int_\mathcal{W} \left(h_1^*(w,a,x) - \widehat{h}_1^{\pi^*}(w,a,x)\right)p^{(2)}(w|x)\mathrm{d}w - \mathcal{T}_2\widehat{h}^{\pi^*}(a, x),\label{eq:B.6}
\end{align}
where the third equality follows from the independence between the random variable $A^\prime$ and the contextual bandit. Now we plug (\ref{eq:B.6}) into the term (i) in the decomposition of the suboptimality:
\begin{align}
v^{\pi^*} - v(\widehat{g}^{\pi^*},\pi^*) &= \int_{\mathcal{A}\times\mathcal{X}}\left(g^*(a,x) - \widehat{g}^{\pi^*}(a,x)\right)p^*(x)\pi^*(a|x)\mathrm{d}x\mathrm{d}a\notag\\
&= \underbrace{\int_{\mathcal{A}\times\mathcal{X}\times\mathcal{W}} \left(h_1^*(w,a,x) - \widehat{h}_1^{\pi^*}(w,a,x)\right)p^{(2)}(w|x)p^*(x)\pi^*(a|x)\mathrm{d}w\mathrm{d}x\mathrm{d}a}_{(a)}\notag\\
&\quad  \underbrace{-\int_{\mathcal{A}\times\mathcal{X}}\mathcal{T}_2\widehat{h}^{\pi^*}(a, x)p^*(x)\pi^*(a|x)\mathrm{d}x\mathrm{d}a}_{(b)}. \label{eq:B.7}
\end{align}
The term (b) is easy to bound:
\begin{align}
(b) &= -\int_{\mathcal{A}\times\mathcal{X}}\mathcal{T}_2\widehat{h}^{\pi^*}(a, x)p^*(x)\pi^*(a|x)\mathrm{d}x\mathrm{d}a = -\mathbb{E}^{(2)}\left[\mathcal{T}_2\widehat{h}^{\pi^*}(A, X)\frac{p^*(X)\pi^*(A|X)}{p^{(2)}(A,X)}\right]. \label{eq:B.8}
\end{align}
Define likelihood ratio functions $b_1:\mathcal{W}\times\mathcal{A}\times\mathcal{X}\to\mathbb{R},\ b_2:\mathcal{A}\times\mathcal{X}\to\mathbb{R}$ by
\begin{equation}
b_1=\frac{p^{(2)}(w|x)\pi^*(a|x)p^*(x)}{p^{(1)}(w,a,x)},\quad b_2(a,x) = \frac{p^*(x)\pi^*(a|x)}{p^{(2)}(a,x)}, \label{eq:B.9}
\end{equation}
then
\begin{equation}
(b) = -\mathbb{E}^{(2)}\left[\mathcal{T}_2\widehat{h}^{\pi^*}(A, X)b_2(A,X)\right] \leq \left\Vert\mathcal{T}_2\widehat{h}^{\pi^*}\right\Vert_{\mu_2,2}\left\Vert b_2 \right\Vert_{\mu_2,2} \label{eq:B.10}
\end{equation}
Likewise, we can bound term (a) by
\begin{align}
    (a) &= -\int_{\mathcal{A}\times\mathcal{X}\times\mathcal{W}}\mathcal{T}_1\widehat{h}^{\pi^*}(w,a,x)p^{(2)}(w|x)p^*(x)\pi^*(a|x)\d w\d x\d a\notag\\
    &= -\E^{(1)}\left[{T}_1\widehat{h}^{\pi^*}(W,A,X)b_1(W,A,X)\right]\leq \left\Vert\mathcal{T}_1\widehat{h}^{\pi^*}\right\Vert_{\mu_1,2}\left\Vert b_1 \right\Vert_{\mu_1,2} \label{eq:B.11}
\end{align}
Combining (\ref{eq:B.7}), (\ref{eq:B.10}) and (\ref{eq:B.11}), we obtain
\begin{equation}
v^{\pi^*} - v(\widehat{g}^{\pi^*},\pi^*) \leq \sum_{k=1}^2 \left\Vert\mathcal{T}_k\widehat{h}^{\pi^*}\right\Vert_{\mu_k,2}\left\Vert b_k\right\Vert_{\mu_k,2} \leq \max_{k\in\{1,2\}} \left\Vert\mathcal{T}_k\widehat{h}^{\pi^*}\right\Vert_{\mu_k,2}\sum_{k=1}^2\left\Vert b_k\right\Vert_{\mu_k,2}.\label{eq:B.12}
\end{equation}
According to the decomposition of the suboptimality in (\ref{eq:B.5}) and Theorem \hyperref[thm:7.1]{7.1}, it holds
\begin{align}
\mathrm{SubOpt}(\widehat{\pi}) &\leq \max_{k\in\{1,2\}} \left\Vert\mathcal{T}_k\widehat{h}^{\pi^*}\right\Vert_{\mu_k,2}\sum_{k=1}^2\left\Vert b_k\right\Vert_{\mu_k,2} + \varepsilon_\mathcal{H}\notag\\
& \overset{\mathcal{E}}{\leq} \sum_{k=1}^2\left\Vert b_k\right\Vert_{\mu_k,2}\bigl(\mathcal{O}(\varepsilon_\Theta) + \mathcal{O}(\varepsilon_\mathcal{H}) + \mathcal{O}(\sqrt{e_\mathcal{D}}) + \mathcal{O}(\eta)\bigl)\,+\,\varepsilon_\mathcal{H}\notag\\
& \leq\sum_{k=1}^2\left\Vert b_k\right\Vert_{\mu_k,2}\bigl(\mathcal{O}(\varepsilon_\Theta) + \mathcal{O}(\varepsilon_\mathcal{H}) + \mathcal{O}(\sqrt{e_\mathcal{D}}) + \mathcal{O}(\eta)\bigl).\label{eq:B.13}
\end{align}
Thus we complete the proof for unconfounded contextual bandits with negative outcome. By changing thee definition of likelihood ratio functions in (\ref{eq:B.9}), we can apply the same procedure to finish the proof for unconfounded contextual bandits with mediator.
\end{proof}

\subsection{Proof of Theorem \ref{thm:7.3}}\label{sec:B.3}
\begin{proof}
Following $\S$\hyperref[sec:B.2]{B.2}, we obtain the error (\ref{eq:B.6}) in estimated CATE for CCB-Neg, and the decomposition (\ref{eq:B.7}) holds. Clearly, the term (b) in  (\ref{eq:B.7}) can be bounded as in (\ref{eq:B.10}). Then it suffices to bound (a) in (\ref{eq:B.7}). By assumption in Theorem \hyperref[thm:7.2]{7.2} there exists $b_1:\mathcal{Z}\to\mathbb{R}$ such that
\begin{equation}
\mathbb{E}^{(1)}[b_1(\mathbf{Z})|W=w,A=a,X=x] = \frac{p^{(2)}(w|x)p^*(x)\pi^*(a|x)}{p^{(1)}(w,a,x)},\label{eq:B.14}
\end{equation}
Plugging in (\ref{eq:B.14}) to (\ref{eq:B.7}), we have
\begin{align}
(a) &= \int_{\mathcal{A}\times\mathcal{X}\times\mathcal{W}} \left(h_1^*(w,a,x) - \widehat{h}_1^{\pi^*}(w,a,x)\right)\mathbb{E}^{(1)}[b_1(\mathbf{Z})|W=w,A=a,X=x]p^{(1)}(w,a,x)\mathrm{d}w\mathrm{d}x\mathrm{d}a\notag\\
&= \int_{\mathcal{A}\times\mathcal{X}\times\mathcal{W}\times\mathcal{Z}} \left(h_1^*(w,a,x) - \widehat{h}_1^{\pi^*}(w,a,x)\right)b_1(\mathbf{z})p^{(1)}(\mathbf{z}|w,a,x)p^{(1)}(w,a,x)\mathrm{d}\mathbf{z}\mathrm{d}w\mathrm{d}x\mathrm{d}a\notag\\
&= \int_{\mathcal{Z}}\left\{ \int_{\mathcal{A}\times\mathcal{X}\times\mathcal{W}} \left(h_1^*(w,a,x) - \widehat{h}_1^{\pi^*}(w,a,x)\right)p^{(1)}(w,a,x|\mathbf{z})\mathrm{d}w\mathrm{d}x\mathrm{d}a\right\}b_1(\mathbf{z})p^{(1)}(\mathbf{z})\mathrm{d}\mathbf{z}\notag\\
&= -\int_{\mathcal{Z}}\mathcal{T}_1\widehat{h}^{\pi^*}b_1(\mathbf{z})p^{(1)}(\mathbf{z})\mathrm{d}\mathbf{z} = -\E^{(1)}[\mathcal{T}_1\widehat{h}^{\pi^*}(\mathbf{Z})b_1(\mathbf{Z})], \label{eq:B.15}
\end{align}
where the last equality holds by the definition of bridge function $h_1^*$. Denote the distribution of $\mathbf{Z}=(Z_1,Z_2)$ in domain $d^{(1)}$ by $\mu_1$, then
\begin{equation}
(a) = -\E^{(1)}[\mathcal{T}_1\widehat{h}^{\pi^*}(\mathbf{Z})b_1(\mathbf{Z})]\leq\left\Vert\mathcal{T}_1\widehat{h}^{\pi^*}\right\Vert_{\mu_1,2}\left\Vert b_1\right\Vert_{\mu_1,2}, \label{eq:B.16}
\end{equation}
The remaining part follows from (\ref{eq:B.12}) and (\ref{eq:B.13}).
\end{proof}

\subsection{Proof of Theorem \ref{thm:7.4}}\label{sec:B.4}
\begin{proof}
For any $a\in\mathcal{A},\ x\in\mathcal{X}$, the error in the estimated CATE is given by
\begin{align}
	&g^*(a,x) - \widehat{g}^{\pi^*}(a,x)\notag\\
	&= \E^{(2)}\left[\left(h_2^*(Z_2,A^\prime,X) - \widehat{h}_2^{\pi^*}(Z_2,A^\prime,X)\right) + \left(\widehat{h}_2^{\pi^*}(Z_2,A^\prime,X) - \widehat{g}^{\pi^*}(A^\prime,X)\right)|A^\prime=a, X=x\right] \notag\\
	&= \E^{(2)}\left[h_2^*(Z_2,A^\prime,X) - \widehat{h}_2^{\pi^*}(Z_2,A^\prime,X)|A^\prime=a, X=x\right] - \mathcal{T}_3\widehat{h}^{\pi^*}(a, x)\notag\\
	&= \E^{(2)}\left[h_2^*(Z_2,a,X) - \widehat{h}_2^{\pi^*}(Z_2,a,X)|X=x\right] - \mathcal{T}_3\widehat{h}^{\pi^*}(a, x).\label{eq:B.17}
\end{align}
And the first term can be further expanded as:
\begin{align}
	&\E^{(2)}\left[h_2^*(Z_2,a,X) - \widehat{h}_2^{\pi^*}(Z_2,a,X)|X\right]\notag\\
	&= \E^{(2)}\left[\left.\E^{(2)}\left[h_2^*(Z_2,a,X) - \widehat{h}_2^{\pi^*}(Z_2,a,X)|Z_1,X\right]\right|X\right]\notag\\
	&= \E^{(2)}\left[\left.\E^{(2)}\left[h_2^*(Z_2,A,X) - \widehat{h}_2^{\pi^*}(Z_2,A,X)|Z_1,A=a,X\right]\right|X\right]\notag\\
	&= \E^{(2)}\left[\left.\E^{(2)}\left[\left(h_1^*(W,A,X) - \widehat{h}_1^{\pi^*}(W,A,X)\right) + \left(\widehat{h}_1^{\pi^*}(W,A,X) - \widehat{h}_2^{\pi^*}(Z_2,A,X)\right)|Z_1,A=a,X\right]\right|X\right]\notag\\
	&= \E^{(2)}\left[\left.\E^{(2)}\left[h_1^*(W,A,X) - \widehat{h}_1^{\pi^*}(W,A,X)|Z_1,A=a,X\right]\right|X\right] - \E^{(2)}\left[\mathcal{T}_2h^{\pi^*}(Z_1,a,X)|X\right], \label{eq:B.18}
\end{align}
where the second equality follows from $Z_2\indep A\,|\,(Z_1,X),$ and the third from the conditional equivalence (\ref{eq:med_h2}) between $h_1^*$ and $h_2^*$. Combining (\ref{eq:B.17}) and (\ref{eq:B.18}) yields
\begin{align}
v^{\pi^*} &- v(\widehat{g}^{\pi^*},\pi^*) \notag\\
&= \int_{\mathcal{A}\times\mathcal{X}}\left(g^*(a,x) - \widehat{g}^{\pi^*}(a,x)\right)p^*(x)\pi^*(a|x)\mathrm{d}x\mathrm{d}a\notag\\
&= \underbrace{\int_{\mathcal{A}\times\mathcal{X}\times\mathcal{Z}_1\times\mathcal{W}}\left(h_1^*(w,a,x) - \widehat{h}^{\pi^*}_1(w,a,x)\right) p^{(2)}(w|z_1,a,x)p^{(2)}(z_1|x)p^*(x)\pi^*(a|x) \mathrm{d}w\mathrm{d}z_1\mathrm{d}x\mathrm{d}a}_{(a)}\notag\\
&\quad \underbrace{- \int_{\mathcal{A}\times\mathcal{X}\times\mathcal{Z}_1}\mathcal{T}_2h^{\pi^*}(z_1,a,x)p^{(2)}(z_1|x)p^*(x)\pi^*(a|x)\mathrm{d}z_1\mathrm{d}x\mathrm{d}a}_{(b)}\notag\\
&\quad \underbrace{- \int_{\mathcal{A}\times\mathcal{X}}\mathcal{T}_3h^{\pi^*}(a,x)p^*(x)\pi^*(a|x)\mathrm{d}x\mathrm{d}a}_{(c)}. \label{eq:B.19}
\end{align}

Now we deal with this decomposition. We first bound the term (a). By assumption in Theorem \hyperref[thm:7.3]{7.3}, there exists $b_1:\mathcal{Z}_1\times\mathcal{Z}_2\to\mathbb{R}$ such that
\begin{equation*}
\E[b_1(Z_1,Z_2)|W=w,A=a,X=x] = \frac{p^*(x)\pi^*(a|x)\int_{\mathcal{Z}_1}p^{(2)}(w|z_1,a,x)p^{(2)}(z_1|x)\mathrm{d}z_1}{p^{(1)}(w,a,x)}.\label{eq:B.20}
\end{equation*}
Then we have
\begin{align}
(a) &= \int_{\mathcal{A}\times\mathcal{X}\times\mathcal{W}}\left(h_1^*(w,a,x) - \widehat{h}_1^{\pi^*}(w,a,x)\right) \E[b_1(Z_1,Z_2)|W=w,A=a,X=x]p^{(1)}(w,a,x)\mathrm{d}w\mathrm{d}x\mathrm{d}a\notag\\
&= \int_{\mathcal{A}\times\mathcal{X}\times\mathcal{W}\times\mathcal{Z}_1\times\mathcal{Z}_2}\left(h_1^*(w,a,x) - \widehat{h}_1^{\pi^*}(w,a,x)\right) b_1(z_1,z_2)p^{(1)}(z_1,z_2|w,a,x)p^{(1)}(w,a,x)\mathrm{d}z_1\mathrm{d}z_2\mathrm{d}w\mathrm{d}x\mathrm{d}a\notag\\
&= \int_{\mathcal{Z}_1\times\mathcal{Z}_2}\left\{\int_{\mathcal{A}\times\mathcal{X}\times\mathcal{W}}\left(h_1^*(w,a,x) - \widehat{h}_1^{\pi^*}(w,a,x)\right)p^{(1)}(w,a,x|z_1,z_2)\mathrm{d}w\mathrm{d}x\mathrm{d}a\right\}b_1(z_1,z_2)p^{(1)}(z_1,z_2)\mathrm{d}z_1\mathrm{d}z_2\notag\\
&= -\int_{\mathcal{Z}_1\times\mathcal{Z}_2}\mathcal{T}_1\widehat{h}^{\pi^*}(z_1,z_2)b_1(z_1,z_2)p^{(1)}(z_1,z_2)\mathrm{d}z_1\mathrm{d}z_2\notag\\
&= -\E^{(1)}\left[\mathcal{T}_1\widehat{h}^{\pi^*}(Z_1,Z_2)b_1(Z_1,Z_2)\right] \leq \left\Vert\mathcal{T}_1\widehat{h}^{\pi^*}\right\Vert_{\mu_1,2}\left\Vert b_1\right\Vert_{\mu_1,2}. \label{eq:B.21}
\end{align}

It remains to bound term (b) and (c). Using the idea of importance sampling, we define the ratio functions
\begin{equation*}
b_2(z_1,a,x) = \frac{p^{(2)}(z_1|x)p^*(x)\pi^*(a|x)}{p^{(2)}(z_1,a,x)},\ \ b_3(a,x) = \frac{p^*(x)\pi^*(a|x)}{p^{(2)}(a,x)}.
\end{equation*}
Then we have
\begin{align}
(b) &= -\E^{(2)}[\mathcal{T}_2\widehat{h}^{\pi^*}(Z_1,A,X)b_2(Z_1,A,X)] \leq \left\Vert\mathcal{T}_2\widehat{h}^{\pi^*}\right\Vert_{\mu_2,2}\left\Vert b_2\right\Vert_{\mu_2,2},\label{eq:B.22}\\
(c) &= -\E^{(2)}[\mathcal{T}_3\widehat{h}^{\pi^*}(A,X)b_3(A,X)] \leq \left\Vert\mathcal{T}_3\widehat{h}^{\pi^*}\right\Vert_{\mu_3,2}\left\Vert b_3\right\Vert_{\mu_3,2},\label{eq:B.23}
\end{align}
Combining (\ref{eq:B.19}), (\ref{eq:B.21}), (\ref{eq:B.22}) and (\ref{eq:B.23}), we have
\begin{equation*}
	v^{\pi^*} - v(\widehat{g}^{\pi^*},\pi^*) \leq \sum_{k=1}^3 \left\Vert\mathcal{T}_k\widehat{h}^{\pi^*}\right\Vert_{\mu_k,2}\left\Vert b_k\right\Vert_{\mu_k,2} \leq \max_{k\in\{1,2,3\}} \left\Vert\mathcal{T}_k\widehat{h}^{\pi^*}\right\Vert_{\mu_k,2}\sum_{k=1}^3\left\Vert b_k\right\Vert_{\mu_k,2}.
\end{equation*}
According to the decomposition of the suboptimality in (\ref{eq:B.5}) and Theorem \hyperref[thm:7.1]{7.1}, it holds
\begin{align*}
\mathrm{SubOpt}(\widehat{\pi}) &\leq \max_{k\in\{1,2,3\}} \left\Vert\mathcal{T}_k\widehat{h}^{\pi^*}\right\Vert_{\mu_k,2}\sum_{k=1}^3\left\Vert b_k\right\Vert_{\mu_k,2} + \varepsilon_\mathcal{H}\\
& \overset{\mathcal{E}}{\leq} \sum_{k=1}^3\left\Vert b_k\right\Vert_{\mu_k,2}\bigl(\mathcal{O}(\varepsilon_\Theta) + \mathcal{O}(\varepsilon_\mathcal{H}) + \mathcal{O}(\sqrt{e_\mathcal{D}}) + \mathcal{O}(\eta)\bigl)\,+\, \varepsilon_\mathcal{H}\\
& \leq\sum_{k=1}^3\left\Vert b_k\right\Vert_{\mu_k,2}\bigl(\mathcal{O}(\varepsilon_\Theta) + \mathcal{O}(\varepsilon_\mathcal{H}) + \mathcal{O}(\sqrt{e_\mathcal{D}}) + \mathcal{O}(\eta)\bigl).
\end{align*}
Thus we complete the proof.
\end{proof}

%% file: append3.tex
\section{Critical Radius for CAT-CCB}
\label{sec:E}

In \S\ref{sec:7.1}, the finite-sample guarantee for CAT-CCB depends on the complexity of two function classes for each integral equation: the test-function class $\Theta_k$, which controls the empirical quadratic penalty, and the product class
\[
    \mathcal{Q}_k
    =
    \left\{
        (\mathbf{x},\mathbf{y}_k,\mathbf{z}_k)
        \mapsto
        \alpha_k(h(\mathbf{x}),\mathbf{y}_k)
        \theta_k(\mathbf{z}_k):
        h\in\mathcal H,\,
        \theta_k\in\Theta_k
    \right\},
\]
which controls the empirical conditional-moment term.

Let $\bar{\eta}_{\mathcal Q,k}$ and
$\bar{\eta}_{\Theta,k}$ denote critical radii of
$\mathcal Q_k$ and $\Theta_k$, respectively, and define the maximal
critical radius
\[
    \bar{\eta}_k
    :=
    \max\{
        \bar{\eta}_{\mathcal Q,k},
        \bar{\eta}_{\Theta,k}
    \}.
\]
The high-probability quantity $\eta_k$ used in
\S\ref{sec:7.1} can then be chosen, up to universal constants, as
\[
    \eta_k
    \gtrsim
    \bar{\eta}_k
    +
    \sqrt{
        \frac{\log(1/\xi)}
             {|\mathcal D_k|}
    }.
\]
We next derive explicit upper bounds for radii $\bar{\eta}_{\mathcal Q,k}$ and $\bar{\eta}_{\Theta,k}$ under bounded finite-dimensional linear function approximation, following the
covering-number arguments of \citet{Dikkala2020} and \citet{chen2023unified}.

\subsection{Critical Radius for Bounded Linear Classes}
\label{sec:E.1}

Consider a bounded linear function class
\[
    \mathcal{F}
    =
    \left\{
        f(\cdot)=w^\top\phi(\cdot):
        w\in\mathbb{R}^{m},\
        \|w\|_2\leq C,\
        \|\phi(\cdot)\|_2\leq1
    \right\}.
\]
Its empirical covering number satisfies
\[
    \log N(t;\mathcal{F},\|\cdot\|_{\mathcal D})
    \lesssim
    m\log\left(1+\frac{C}{t}\right).
\]
Consequently, standard entropy bounds imply a critical radius of order
$\widetilde{\mathcal O}(\sqrt{m/n})$, where $n=|\mathcal D|$.

More generally, suppose $\mathcal{F}_1$ and $\mathcal{F}_2$ are bounded
linear classes of dimensions $m_1$ and $m_2$, respectively, and define
their product class
\[
    \mathcal{Q}
    :=
    \{f_1f_2:f_1\in\mathcal{F}_1,\ f_2\in\mathcal{F}_2\}.
\]
Following Lemma F.5 of \citet{chen2023unified}, which is based on the metric-entropy techniques of \citet{Dikkala2020}, the critical
radius of $\mathcal Q$ satisfies
\begin{equation}
    \bar{\eta}_{\mathcal Q}
    \leq
    \mathcal{O}
    \left(
        C_{\mathcal Q}
        \sqrt{
            \frac{m_1+m_2}{n}
            \log\left(
                1+\frac{n}{m_1+m_2}
            \right)
        }
    \right),
    \label{eq:E.linear_product}
\end{equation}
where $C_{\mathcal Q}$ depends only on the uniform bounds of the two
linear classes. In particular,
\[
    \bar{\eta}_{\mathcal Q}
    =
    \widetilde{\mathcal O}
    \left(
        \sqrt{\frac{m_1+m_2}{n}}
    \right)
\]
when these bounds are treated as constants.

We apply \eqref{eq:E.linear_product} to CAT-CCB as follows. Suppose each
bridge function belongs to a bounded linear class
\begin{equation}
    \mathcal H_j
    =
    \left\{
        h_j(\cdot)
        =
        w_j^\top\phi_j(\cdot):
        w_j\in\mathbb{R}^{m_j},\
        \|w_j\|_2\leq C_j,\
        \|\phi_j(\cdot)\|_2\leq1
    \right\},
    \label{eq:E.bridge_linear}
\end{equation}
the CATE belongs to
\begin{equation}
    \mathcal G
    =
    \left\{
        g(\cdot)
        =
        w_g^\top\phi_g(\cdot):
        w_g\in\mathbb{R}^{m_g},\
        \|w_g\|_2\leq C_g,\
        \|\phi_g(\cdot)\|_2\leq1
    \right\},
    \label{eq:E.cate_linear}
\end{equation}
and the $k$-th test-function class is
\begin{equation}
    \Theta_k
    =
    \left\{
        \theta_k(\cdot)
        =
        \beta_k^\top\psi_k(\cdot):
        \beta_k\in\mathbb{R}^{d_k},\
        \|\beta_k\|_2\leq D_k,\
        \|\psi_k(\cdot)\|_2\leq1
    \right\}.
    \label{eq:E.test_linear}
\end{equation}

Since $\alpha_k$ is linear in the bridge functions, the residual class
\[
    \mathcal U_k
    :=
    \left\{
        (\mathbf{x},\mathbf{y}_k)
        \mapsto
        \alpha_k(h(\mathbf{x}),\mathbf{y}_k):
        h\in\mathcal H
    \right\}
\]
can be embedded into a bounded linear class by concatenating the feature
maps of the bridge functions appearing in the $k$-th equation. When an
observed outcome appears in the residual, the feature vector can be
augmented by one additional coordinate. Assumption~\ref{assumption:7.3}
ensures that this coordinate remains bounded.

Let $r_k$ denote the resulting dimension of $\mathcal U_k$. Since
$\mathcal Q_k$ is the product of $\mathcal U_k$ and $\Theta_k$,
\eqref{eq:E.linear_product} yields
\begin{equation}
    \bar{\eta}_k
    =
    \widetilde{\mathcal O}
    \left(
        \sqrt{
            \frac{r_k+d_k}{|\mathcal D_k|}
        }
    \right).
    \label{eq:E.general_radius}
\end{equation}
The critical radius of $\Theta_k$ alone is of order
$\widetilde{\mathcal O}
(\sqrt{d_k/|\mathcal D_k|})$ and is therefore controlled at the same
order by \eqref{eq:E.general_radius}.

\subsection{Critical Radius for CCB-Neg}
\label{sec:E.2}

For CCB-Neg, the CATE is identified by the two conditional moment
equations
\begin{align*}
    \E^{(1)}
    [
        h_1(W,A,X)-Y
        \mid Z_1,Z_2
    ]
    &=0,\\
    \E^{(2)}
    [
        g(A',X)-h_1(W,A',X)
        \mid A',X
    ]
    &=0.
\end{align*}
Suppose $h_1$ and $g$ are represented by bounded linear classes of
dimensions $m_1$ and $m_g$, respectively, and the corresponding
test-function classes have dimensions $d_1$ and $d_2$.

For the first equation, the residual
\[
    h_1(W,A,X)-Y
\]
is contained in a bounded linear class of dimension $m_1+1$ by
augmenting the feature map of $h_1$ with the bounded outcome $Y$.
Therefore,
\begin{equation}
    \bar{\eta}_1
    =
    \widetilde{\mathcal O}
    \left(
        \sqrt{
            \frac{m_1+d_1+1}{|\mathcal D_1|}
        }
    \right).
    \label{eq:E.neg_radius1}
\end{equation}

For the second equation, the residual
\[
    g(A',X)-h_1(W,A',X)
\]
can be represented through the concatenated feature maps of $g$ and
$h_1$, and is therefore contained in a bounded linear class of dimension
$m_g+m_1$. It follows that
\begin{equation}
    \bar{\eta}_2
    =
    \widetilde{\mathcal O}
    \left(
        \sqrt{
            \frac{m_g+m_1+d_2}{|\mathcal D_2|}
        }
    \right).
    \label{eq:E.neg_radius2}
\end{equation}

Hence, if all bridge and test-function dimensions are of order $d$ and
$|\mathcal D_1|,|\mathcal D_2|\gtrsim n$, then
\begin{equation}
    \bar{\eta}
    :=
    \sqrt{\bar{\eta}_1^2+\bar{\eta}_2^2}
    =
    \widetilde{\mathcal O}
    \left(
        \sqrt{\frac{d}{n}}
    \right).
    \label{eq:E.neg_radius}
\end{equation}

\subsection{Critical Radius for CCB-Med}
\label{sec:E.3}

For CCB-Med, the identification IES contains three conditional moment
equations:
\begin{align*}
    \E^{(1)}
    [
        h_1(W,A,X)-Y
        \mid Z_1,Z_2
    ]
    &=0,\\
    \E^{(2)}
    [
        h_2(Z_2,A,X)-h_1(W,A,X)
        \mid Z_1,A,X
    ]
    &=0,\\
    \E^{(2)}
    [
        g(A',X)-h_2(Z_2,A',X)
        \mid A',X
    ]
    &=0.
\end{align*}
Suppose the linear classes for $h_1$, $h_2$, and $g$ have dimensions
$m_1$, $m_2$, and $m_g$, while the corresponding test-function classes
have dimensions $d_1$, $d_2$, and $d_3$.

As in the CCB-Neg case, the first residual is contained in a bounded
linear class of dimension $m_1+1$. The second and third residuals are
contained in bounded linear classes obtained by concatenating the
corresponding bridge features, with dimensions $m_1+m_2$ and
$m_2+m_g$, respectively. Applying \eqref{eq:E.general_radius} gives
\begin{align}
    \bar{\eta}_1
    &=
    \widetilde{\mathcal O}
    \left(
        \sqrt{
            \frac{m_1+d_1+1}{|\mathcal D_1|}
        }
    \right),
    \label{eq:E.med_radius1}\\
    \bar{\eta}_2
    &=
    \widetilde{\mathcal O}
    \left(
        \sqrt{
            \frac{m_1+m_2+d_2}{|\mathcal D_2|}
        }
    \right),
    \label{eq:E.med_radius2}\\
    \bar{\eta}_3
    &=
    \widetilde{\mathcal O}
    \left(
        \sqrt{
            \frac{m_2+m_g+d_3}{|\mathcal D_3|}
        }
    \right).
    \label{eq:E.med_radius3}
\end{align}
Consequently, if all bridge and test-function dimensions are of order $d$ and $|\mathcal D_k|\gtrsim n$ for $k=1,2,3$, then
\begin{equation}
    \bar{\eta}
    :=
    \sqrt{
        \bar{\eta}_1^2+\bar{\eta}_2^2+\bar{\eta}_3^2
    }
    =\widetilde{\mathcal O}
    \left(
        \sqrt{\frac{d}{n}}
    \right).
    \label{eq:E.med_radius}
\end{equation}

\paragraph{Implication for CAT-CCB.}
Recall from \S\ref{sec:7.1} that the high-probability quantity used in
the uncertainty-quantification result can be chosen as
\[
    \eta_k
    \gtrsim
    \bar{\eta}_k
    +
    \sqrt{
        \frac{\log(1/\xi)}
             {|\mathcal D_k|}
    }.
\]
Therefore, for bounded finite-dimensional linear function classes with
dimensions of order $d$ and datasets of size at least $n$,
\[
    \eta
    =
    \widetilde{\mathcal O}
    \left(
        \sqrt{\frac{d}{n}}
        +
        \sqrt{\frac{\log(1/\xi)}{n}}
    \right).
\]
In particular, under exact realizability and exact compatibility,
$\varepsilon_{\mathcal H}=\varepsilon_\Theta=0$, and choosing
$e_{\mathcal D}$ on the order of $\eta^2$ reduces the suboptimality
bounds in Theorems~\ref{thm:7.2}--\ref{thm:7.4} to
\begin{equation}
    \mathrm{SubOpt}(\widehat{\pi})
    =
    \widetilde{\mathcal O}
    \left(
        \sum_{k=1}^K
        \|b_k\|_{\mu_k,2}
        \sqrt{\frac{d}{n}}
    \right),
    \label{eq:E.subopt}
\end{equation}
up to the confidence-dependent logarithmic factors. Thus, when the
weighting functions have bounded $L^2$ norms, CAT-CCB achieves the
standard $n^{-1/2}$ statistical rate under finite-dimensional linear
function approximation.

%% file: append4.tex
\section{Proof for the Convergence of CAT-CMDP}\label{sec:C}
The convergence result of our CAT-CMDP algorithm is shown in Theorem \ref{thm:7.5}. The proof of Theorem \ref{thm:7.5} is based on the following lemmas.
\begin{lemma}[Decomposition of suboptimality]\label{lemma:C.1} Define $\mathcal{E}$ as following:
\begin{align}
\begin{aligned}
\mathcal{E}&:=\biggl\{g_h^r\in C_h^r(\alpha_1),g_h^{s(i)}\in C_h^{s(i)}(\alpha_2),\mathbf{M}_h\in C_h^M(\alpha_3),\forall h\in[H],i\in[d_\mathcal{S}],\quad and\\ &\phi(s)\Lambda_h^{-1}\phi(s)
\leq c_1\left(\mathrm{rank}(\Sigma_h^{(2)})+\log\frac{c_2H}{\delta}\right)
\phi(s)^\top(K_2\Sigma_h^{(2)}+\lambda\mathbf{I})^{-1}\phi(s),\ \forall h\in[H],s\in\mathcal{S} \biggr\},
\end{aligned}\label{eq:C.1}
\end{align}
where $\Sigma_h^{(2)}=\E^{(2)}[\phi(S_h)\phi(S_h)^\top]$. In this event, the true model $\beta^*$ falls in the confidence set constructed by Algorithm \hyperref[alg:2]{2}, and the sample covariance matrix $\Lambda_h=\lambda\mathbf{I}+\sum_{k=1}^{K_2}\phi(S_h^{(2),k})\phi(S_h^{(2),k})^\top$ satisfies some regularity condition. Then under event $\mathcal{E}$, it holds that
\begin{align}
\begin{aligned}
    J&(\pi^*) - J(\widehat{\pi})\\
&\overset{\mathcal{E}}{\lesssim}\sum_{h=1}^H\left\{
\frac{H}{2\sigma_s}\sqrt{\frac{C_1^*C_w}{L^2} \sum_{i=1}^{d_\mathcal{S}}\mathcal{L}_h^{s(i)}(\widehat{g}_h^{s(i)})} +
\sqrt{C_1^*C_wC_z^2 \mathcal{L}_h^{r}(\widehat{g}_h^{r})} + \left(1+\frac{H\sqrt{d_\mathcal{S}}}{2\sigma_s}\right)\sqrt{C_2^*\mathcal{L}_h^F(F_h^{\widehat{\beta}})}\right\},
\end{aligned}
\label{eq:C.2}
\end{align}
where we define the RMSEs
\begin{align*}
\mathcal{L}_h^{s(i)}(\widehat{g}_h^{s(i)}) &= \E^{(1)}\left[\left(\E^{(1)}\left[S_{h+1,i} - \widehat{g}_h^{s(i)}(S_h,A_h,W_h)|S_h,Z_h\right]\right)^2\right],\\
\mathcal{L}_h^{r}(\widehat{g}_h^{r}) &= \E^{(1)}\left[\left(\E^{(1)}\left[R_{h} - \widehat{g}_h^{r}(S_h,A_h,W_h)|S_h,Z_h\right]\right)^2\right],\\
\mathcal{L}_h^F(F_h^{\widehat{\beta}}) &= \E^{(2)}\left[\Vert (F_h^*-F_h^{\widehat{\beta}})(\cdot|S_h)\Vert_1^2\right],
\end{align*}
and the pessimistic estimator of true model $\beta^*$ is $\widehat{\beta}=\{(\widehat{g}_h^r,\widehat{g}_h^s,\widehat{\mathbf{M}}_h)\}_{h=1}^H$, and $F^*(w_h|s_h),F^{\widehat{\beta}}(w_h|s_h)$ are the cumulative distribution functions induced by true $\beta^*$ and pessimistic $\widehat{\beta}$, respectively, which is specified in Assumption \ref{assumption:6.1} (i).
\end{lemma}
\begin{proof}
See $\S$\ref{sec:D.1} for a detailed proof.
\end{proof}

\begin{lemma}[Analysis of reward]\label{lemma:C.2} We set the parameter $\alpha_1$ in Algorithm \hyperref[alg:2]{2} as $$\alpha_1=\mathcal{O}\left(\frac{d}{K_1}\log \frac{HdK_1}{\delta}\right).$$ Then the event
\begin{equation}
\mathcal{E}_1:=\{g_h^r\in C_h^r(\alpha_1),\ \forall h\in[H]\}\label{eq:C.6}
\end{equation}
occurs with probability at least $1-\delta/3.$ Moreover, we have
\begin{equation}
\mathcal{L}_h^r(\widehat{g}_h^r)\overset{\mathcal{E}_1}{\lesssim}\frac{d}{K_1}\log \frac{HdK_1}{\delta}.\label{eq:C.7}
\end{equation}
\end{lemma}
\begin{proof}
See $\S$\hyperref[sec:D.3]{D.3} for a detailed proof.
\end{proof}

\begin{lemma}[Analysis of transition]\label{lemma:C.3} We set the parameter $\alpha_2$ in Algorithm \hyperref[alg:2]{2} as $$\alpha_2=\mathcal{O}\left(\frac{d}{K_1}\log\frac{Hdd_\mathcal{S}K_1}{\delta}\right).$$ Then the event
\begin{equation}
\mathcal{E}_2:=\{g_h^{s(i)}\in C_h^{s(i)}(\alpha_2),\ \forall h\in[H],i\in[d_\mathcal{S}]\}\label{eq:C.8}
\end{equation}
occurs with probability at least $1-\delta/3.$ Moreover, we have
\begin{equation}
\mathcal{L}_h^{s(i)}(\widehat{g}_h^{s(i)})\overset{\mathcal{E}_2}{\lesssim} \frac{d}{K_1}\log\frac{Hdd_\mathcal{S}K_1}{\delta}.\label{eq:C.9}
\end{equation}
\end{lemma}
\begin{proof}
See $\S$\ref{sec:D.3} for a detailed proof.
\end{proof}

\begin{lemma}[Analysis of negative outcome]\label{lemma:C.4} We set the parameter $\alpha_3$ in Algorithm \hyperref[alg:2]{2} as
\begin{equation}
\alpha_3 = \sup_{h\in[H]}\sqrt{m + \sigma_w^2d_\mathcal{W}+\sigma_w^2\mathrm{rank}(\Sigma_h^{(2)})\{\mathrm{rank}(\Sigma_h^{(2)}) + \log\left({c_2H}/{\delta}\right)\}\log(1+K_2)},\label{eq:C.10}
\end{equation}
and let $\lambda=\Omega(1)$ be a constant. Then the following event has probability at least $1-\delta/3$:
\begin{align}
\mathcal{E}_3&:=\biggl\{\mathbf{M}_h\in C_h^M(\alpha_3),\ \forall h\in[H],\quad and\quad\\
&\phi(s)\Lambda_h^{-1}\phi(s)\leq c_1\left(\mathrm{rank}(\Sigma_h^{(2)})+\log\frac{c_2H}{\delta}\right)
\phi(s)^\top(K_2\Sigma_h^{(2)}+\lambda\mathbf{I})^{-1}\phi(s),\ \forall s\in\mathcal{S},\forall h\in[H]\biggr\},\label{eq:C.11}
\end{align}
where $c_1,c_2>0$ are some universal constants. Moreover, we have
\begin{equation}
\mathcal{L}_h^F(F_h^{\widehat{\beta}}) \overset{\mathcal{E}_3}{\lesssim}\frac{m\overline{R}}{K_2\sigma_w^2}+\frac{d_{\mathcal{W}}\overline{R}}{K_2}+\frac{\overline{R}^2\log(1+K_2)}{K_2},\label{eq:C.12}
\end{equation}
where $\Sigma_h^{(2)}=\E^{(2)}[\phi(S_h)\phi(S_h)^\top],$ and $\overline{R}=
\max_{h\in[H]}\mathrm{rank}(\Sigma_h^{(2)})\{\mathrm{rank}(\Sigma_h^{(2)}) + \log(c_2H/\delta)\}.$
\end{lemma}
\begin{proof}
See \S\ref{sec:D.4} for a detailed proof.
\end{proof}
\renewcommand{\proofname}{\it Proof of Theorem \hyperref[thm:7.4]{7.4}}
\begin{proof}
Set the confidence parameters as the specifications given in Lemmas \ref{lemma:C.2}, \ref{lemma:C.3} and \ref{lemma:C.4}, then we have
\begin{equation}
\mathbb{P}(\mathcal{E}) = \mathbb{P}(\mathcal{E}_1\cap\mathcal{E}_2\cap\mathcal{E}_3) \geq 1-\delta.\label{eq:C.13}
\end{equation}
We plug in \eqref{eq:C.7}, \eqref{eq:C.9} and \eqref{eq:C.12} to \eqref{eq:C.2}:
\begin{align}
\begin{aligned}
J&(\pi^*) - J(\widehat{\pi})\\
&\overset{\mathcal{E}}{\lesssim}\frac{H^2}{2\sigma_s}\left(\sqrt{\frac{C_1^*C_wdd_\mathcal{S}}{L^2K_1}\log\frac{Hdd_\mathcal{S}K_1}{\delta}} + \sqrt{C_2^*\overline{R}d_\mathcal{S}\frac{m/\sigma_w^2+d_\mathcal{W}+\overline{R}\log(1+K_2)}{K_2}}\right).
\end{aligned}\label{eq:C.14}
\end{align}
Since $\sigma_s,\sigma_w$ are known constants, the above equation implies \eqref{eq:7.23}. Thus we finish the proof.
\end{proof}
\renewcommand{\proofname}{\it Proof}

%% file: append5.tex
\section{Proof of Lemmas in \S\ref{sec:C}}
\subsection{Proof of Lemma \ref{lemma:C.1}}\label{sec:D.1}
\begin{proof}
We denote by $\widehat{\beta}$ the pessimistic estimator of model $\beta$ given $\widehat{\pi}$, i.e. $\widehat{\beta} = \widehat{\beta}^{\widehat{\pi}}= \mathrm{argmin}_{\beta\in C_\beta}J(\widehat{\pi};\beta)$. Then on the event $\mathcal{E}$, the suboptimality of $\widehat{\pi}$ can be decomposed and bounded as follows:
\begin{align}
J(\pi^*) - J(\widehat{\pi}) &= J(\pi^*) - J(\pi^*;\widehat{\beta}) + J(\pi^*;\widehat{\beta}) - J(\widehat{\pi})\notag\\
&\leq J(\pi^*) - J(\pi^*;\widehat{\beta}) + J(\widehat{\pi};\widehat{\beta}) - J(\widehat{\pi})\notag\\
&\overset{\mathcal{E}}{\leq} J(\pi^*) - J(\pi^*;\widehat{\beta})\label{eq:D.1}
\end{align}
The first inequality follows from the optimization step $\widehat{\pi}=\mathrm{argmax}_{\pi\in\Pi}J(\pi;\widehat{\beta})$, and the second holds because the true model $\beta^*\in C_\beta$ under event $\mathcal{E}$. To continue, we have
\begin{align}
J(\pi^*) - J(\pi^*;\widehat{\beta}) &= \int_\mathcal{S} \left[V_1^{\pi^*}(s_1) - V_1^{\pi^*}(s_1;\widehat{\beta})\right]\mu_1(s_1)\mathrm{d}s_1\notag\\
&= \int_{\mathcal{S}\times\mathcal{A}} \left[V_1^{\pi^*}(s_1) - V_1^{\pi^*}(s_1;\widehat{\beta})\right]\mu_1(s_1)\pi^*(a_1|s_1)\mathrm{d}a_1\mathrm{d}s_1\notag\\
&= \E_{d_1^{\pi^*}}\left[V_1^{\pi^*}(S_1) - V_1^{\pi^*}(S_1;\widehat{\beta})\right],\label{eq:D.2}
\end{align}
where $d_1^{\pi^*}$ is the visitation measure induced by initial state distribution $\mu_1$ and policy $\pi^*$.

Now we transform an analogue of (\ref{eq:D.2}) for any $h\in[H]$ as follows:
\begin{align}
\begin{aligned}
&\E_{d_h^{\pi^*}}\left[V_h^{\pi^*}(S_h) - V_h^{\pi^*}(S_h;\widehat{\beta})\right]\\
&= \E_{d_h^{\pi^*}}\left[Q_h^{\pi^*}(S_h,A_h) - Q_h^{\pi^*}(S_h,A_h;\widehat{\beta})\right]\\
&= \E_{d_h^{\pi^*}}\left[r^*_h(S_h,A_h) - r_h^{\widehat{\beta}}(S_h,A_h) + \mathcal{P}_h^*V_{h+1}^{\pi^*}(S_h,A_h) - \mathcal{P}_h^{\widehat{\beta}}V_{h+1}^{\pi^*}(S_h,A_h;\widehat{\beta})\right]\\
&= \E_{d_h^{\pi^*}}\left[(r^*_h - r_h^{\widehat{\beta}})(S_h,A_h)\right] + \E_{d_h^{\pi^*}}\left[\mathcal{P}_h^*V_{h+1}^{\pi^*}(S_h,A_h) - \mathcal{P}_h^*V_{h+1}^{\pi^*}(S_h,A_h;\widehat{\beta})\right]\\
&\qquad +\E_{d_h^{\pi^*}}\left[\mathcal{P}_h^*V_{h+1}^{\pi^*}(S_h,A_h;\widehat{\beta}) - \mathcal{P}_h^{\widehat{\beta}}V_{h+1}^{\pi^*}(S_h,A_h;\widehat{\beta})\right]\\
&\leq \E_{d_h^{\pi^*}}\left\vert(r^*_h - r_h^{\widehat{\beta}})(S_h,A_h)\right\vert + \E_{d_h^{\pi^*}}\left[\int_\mathcal{S}\left[V_{h+1}^{\pi^*}(s_{h+1}) - V_{h+1}^{\pi^*}(s_{h+1};\widehat{\beta})\right]\mathcal{P}^*_h(s_{h+1}|S_h,A_h)\mathrm{d}s_{h+1}\right]\\
&\qquad +\E_{d_h^{\pi^*}}\left[ \left\Vert(\mathcal{P}_h^*-\mathcal{P}_h^{\widehat{\beta}})(\cdot|S_h,A_h)\right\Vert_1\left\Vert V_{h+1}^{\pi^*}(\cdot;\widehat{\beta})\right\Vert_\infty\right]\\
&\leq \E_{d_{h+1}^{\pi^*}}\left[V_{h+1}^{\pi^*}(S_{h+1}) - V_{h+1}^{\pi^*}(S_{h+1};\hat{\beta})\right] + \mathbb{E}_{d_h^{\pi^*}}\left[H\left\Vert(\mathcal{P}_h^*-\mathcal{P}_h^{\widehat{\beta}})(\cdot|S_h,A_h)\right\Vert_1 + \left\vert(r^*_h - r_h^{\widehat{\beta}})(S_h,A_h)\right\vert\right],\label{eq:D.3}
\end{aligned}
\end{align}
where the last inequality holds because of the boundedness of our estimated value functions.

Combining \eqref{eq:D.2} and \eqref{eq:D.3} yields
\begin{equation}
J(\pi^*) - J(\pi^*;\widehat{\beta}) = \sum_{h=1}^H \mathbb{E}_{d_h^{\pi^*}}\left[H\left\Vert(\mathcal{P}_h^*-\mathcal{P}_h^{\widehat{\beta}})(\cdot|S_h,A_h)\right\Vert_1 + \left\vert(r^*_h - r_h^{\widehat{\beta}})(S_h,A_h)\right\vert\right].\label{eq:D.4}
\end{equation}
Now we bound the two terms in (\ref{eq:D.4}). We note that both $\mathcal{P}_h^*(\cdot|S_h,A_h)$ and $\mathcal{P}_h^{\hat{\beta}}(\cdot|S_h,A_h)$ are Gaussian distributions. So we use Pinsker's inequality to bound their total variation distance:
\begin{align}
\begin{aligned}
&\left\Vert(\mathcal{P}_h^*-\mathcal{P}_h^{\widehat{\beta}})(\cdot|s_h,a_h)\right\Vert_1 \leq \left(\frac{1}{2}\int_\mathcal{S}\mathcal{P}_h^*(s_{h+1}|s_h,a_h)\log\frac{\mathcal{P}_h^*}{\mathcal{P}_h^{\widehat{\beta}}}(s_{h+1}|s_h,a_h)\mathrm{d}s_{h+1}\right)^{1/2}\\
&\quad= \left(\frac{1}{2}\int_\mathcal{S}\mathcal{P}_h^*(s_{h+1}|s_h,a_h)\frac{\Vert s_{h+1} - m_h^{\widehat{\beta}}(s_h,a_h)\Vert_2^2 -  \Vert s_{h+1} - m_h^*(s_h,a_h)\Vert_2^2}{2\sigma_s^2}\mathrm{d}s_{h+1}\right)^{1/2}\\
&\quad= \biggl(\frac{1}{4\sigma_s^2}\int_\mathcal{S}\mathcal{P}_h^*(s_{h+1}|s_h,a_h)\Vert (m_h^* - m_h^{\widehat{\beta}})(s_h,a_h)\Vert_2^2\mathrm{d}s_{h+1}\\
&\quad\quad + \frac{1}{4\sigma_s^2}\int_\mathcal{S}\mathcal{P}_h^*(s_{h+1}|s_h,a_h)(m_h^* - m_h^{\widehat{\beta}})^\top(s_h,a_h)(s_{h+1} - m_h^*(s_h,a_h))\mathrm{d}s_{h+1}\biggr)^{1/2}\\
&\quad= \left(\frac{1}{4\sigma_s^2}\int_\mathcal{S}\mathcal{P}_h^*(s_{h+1}|s_h,a_h)\Vert (m_h^*- m_h^{\widehat{\beta}})(s_h,a_h)\Vert_2^2\mathrm{d}s_{h+1}\right)^{1/2}\\
&\quad=\frac{\Vert (m_h^* - m_h^{\widehat{\beta}})(s_h,a_h)\Vert_2}{2\sigma_s},
\end{aligned}\label{eq:D.5}
\end{align}
where $m_h^*$ and $m_h^{\widehat{\beta}}$ are the mean vector of $\mathcal{P}_h^*(\cdot|s_h,a_h)$ and $\mathcal{P}_h^{\hat{\beta}}(\cdot|s_h,a_h)$, and the third equality holds because $\mathcal{P}_h^*(\cdot|s_h,a_h)$ is centered at $m_h^*(s_h,a_h)$ and the linear term of $(s_{h+1} - m_h^*(s_h,a_h))$ can be eliminated by expectation. Specifically, we have
\begin{align*}
	m_h^*(s_h,a_h) &= \int g_h^s(s_h,a_h,w_h)\mathrm{d}F_h^*(w_h|s_h),\\	
	m_h^{\widehat{\beta}}(s_h,a_h) &= \int\widehat{g}_h^s(s_h,a_h,w_h)\mathrm{d}F_h^{\widehat{\beta}}(w_h|s_h),
\end{align*}
where $F_h^*$ and $F_h^{\widehat{\beta}}$ are the measures of $W_h$ induced by the conditional Gaussian distributions under $\beta^*$ and $\widehat{\beta}$.
Then we can obtain that
\begin{align}
\begin{aligned}
&\E_{d_h^{\pi^*}}\left\Vert(\mathcal{P}_h^*-\mathcal{P}_h^{\widehat{\beta}})(\cdot|S_h,A_h)\right\Vert_1 \\
&\quad\leq \frac{1}{2\sigma_s}\E_{d_h^{\pi^*}}\left[\sqrt{\sum_{i=1}^{d_\mathcal{S}}\left\vert\int g_h^{s(i)}(S_h,A_h,w_h)\mathrm{d}F_h^*(w_h|S_h) - \int\widehat{g}_h^{s(i)}(S_h,A_h,w_h)\mathrm{d}F_h^{\widehat{\beta}}(w_h|S_h)\right\vert^2}\right]\\
&\quad\leq \frac{1}{2\sigma_s}\E_{d_h^{\pi^*}}\left[\sqrt{\sum_{i=1}^{d_\mathcal{S}}\left\vert\int(g_h^{s(i)}-\widehat{g}_h^{s(i)})(S_h,A_h,w_h)\mathrm{d}F_h^*(w_h|S_h)\right\vert^2}\right]\\
&\qquad + \frac{1}{2\sigma_s}\E_{d_h^{\pi^*}}\left[\sqrt{\sum_{i=1}^{d_\mathcal{S}}\left\vert\int \widehat{g}_h^{s(i)}(S_h,A_h,w_h)\mathrm{d}(F^*_h - F_h^{\widehat{\beta}})(w_h|S_h)\right\vert^2}\right]\\
&\quad\leq \frac{1}{2\sigma_s}\E_{d_h^{\pi^*}}\left[\sqrt{\sum_{i=1}^{d_\mathcal{S}}\left\vert\E^*[(g_h^{s(i)}-\widehat{g}_h^{s(i)})(S_h,A_h,W_h)|S_h,A_h]\right\vert^2}\right]\\
&\quad\quad + \frac{\sqrt{d_\mathcal{S}}}{2\sigma_s}\E_{d_h^{\pi^*}}\Vert(F^*_h - F_h^{\widehat{\beta}})(\cdot|S_h)\Vert_1,\\
&\quad\leq\frac{1}{2\sigma_s}\sqrt{C^*_1\sum_{i=1}^{d_\mathcal{S}}\E^{(1)}\left\vert\E^*[(g_h^{s(i)}-\widehat{g}_h^{s(i)})(S_h,A_h,W_h)|S_h,A_h]\right\vert^2}\\
&\quad\quad + \frac{1}{2\sigma_s}\sqrt{d_\mathcal{S}C_2^*\E^{(2)}\Vert(F^*_h - F_h^{\widehat{\beta}})(\cdot|S_h)\Vert_1^2},
\end{aligned} \label{eq:D.6}
\end{align}

where the third equality follows from the boundedness of $\mathcal{G}$ stated in Assumption \hyperref[assumption:6.2]{6.2} (i), which implies that $\Vert\widehat{g}^{s(i)}_h\Vert_\infty\leq 1,$ and the fourth from Cauchy-Schwarz inequality and Assumption \hyperref[assumption:7.4]{7.4} (i). 

According to Assumption \hyperref[assumption:7.4]{7.4} (ii), we have
\begin{align}
\E^*&[(g_h^{s(i)}-\widehat{g}_h^{s(i)})(S_h,A_h,W_h)|S_h,A_h] = \E^{(1)}\left[\left.(g_h^{s(i)}-\widehat{g}_h^{s(i)})(S_h,A_h,W_h)\frac{p^*(W_h|S_h)}{p^{(1)}(W_h|S_h)}\right|S_h,A_h\right]\notag\\
&\leq \sqrt{C_w\E^{(1)}\left[\left.\left(g_h^{s(i)}(S_h,A_h,W_h)-\widehat{g}_h^{s(i)}(S_h,A_h,W_h)\right)^2\right|S_h,A_h\right]},\label{eq:D.7}
\end{align}
Then we can bound the first term in (\ref{eq:D.6}),
\begin{align}
\begin{aligned}
&\sqrt{\sum_{i=1}^{d_\mathcal{S}}\E^{(1)}\left\vert\E^*[(g_h^{s(i)}-\widehat{g}_h^{s(i)})(S_h,A_h,W_h)|S_h,A_h]\right\vert^2}\\
&\quad=\sqrt{\sum_{i=1}^{d_\mathcal{S}}\E^{(1)}\left[C_w\E^{(1)}\left[\left.\left(g_h^{s(i)}(S_h,A_h,W_h)-\widehat{g}_h^{s(i)}(S_h,A_h,W_h)\right)^2\right|S_h,A_h\right]\right]}\\
&\quad= \sqrt{C_w\sum_{i=1}^{d_\mathcal{S}}\E^{(1)}\left[\left(g_h^{s(i)}(S_h,A_h,W_h)-\widehat{g}_h^{s(i)}(S_h,A_h,W_h)\right)^2\right]}\\
&\quad\leq\sqrt{\frac{C_w}{L_h^2}\sum_{i=1}^{d_\mathcal{S}}\E^{(1)}\left[\left(\E^{(1)}\left[g_h^{s(i)}(S_h,A_h,W_h)-\widehat{g}_h^{s(i)}(S_h,A_h,W_h)|S_h,Z_h\right]\right)^2\right]}\\
&\quad= \sqrt{\frac{C_w}{L_h^2} \sum_{i=1}^{d_\mathcal{S}}\mathcal{L}_h^{s(i)}(\widehat{g}_h^{s(i)})},
\end{aligned}\label{eq:D.8}
\end{align}
where the inequality follows from Lemma \hyperref[lemma:7.2]{7.2}. For the second term, note that $p^*(w_h|s_h)=p^{(2)}(s_h|w_h)$, which implies the equivalence of cumulative functions $F^*_h=F_h^{(2)}$, then we have
\begin{align}
\sqrt{\E^{(2)}\Vert(F^*_h - F_h^{\widehat{\beta}})(\cdot|S_h)\Vert_1^2} = \sqrt{\E^{(2)}\Vert(F^*_h - F_h^{\widehat{\beta}})(\cdot|S_h)\Vert_1^2} = \sqrt{\mathcal{L}_h^F(\widehat{F}_h^{\widehat{\beta}})},\label{eq:D.9}
\end{align}
and
\begin{equation}
\E_{d_h^{\pi^*}}\left\Vert(\mathcal{P}_h^*-\mathcal{P}_h^{\widehat{\beta}})(\cdot|S_h,A_h)\right\Vert_1 \leq \frac{1}{2\sigma_s}\sqrt{\frac{C_1^*C_w}{L^2}\sum_{i=1}^{d_\mathcal{S}}\mathcal{L}_h^{s(i)}(\widehat{g}_h^{s(i)})} + \frac{1}{2\sigma_s}\sqrt{d_\mathcal{S}C_2^*\mathcal{L}_h^F(\widehat{F}_h^{\widehat{\beta}})}.\label{eq:D.10}
\end{equation}
Now we bound the second term in (\ref{eq:D.4}).
\begin{align}
\begin{aligned}
&\E_{d_h^{\pi^*}}\left\vert(r^*_h - r_h^{\widehat{\beta}})(S_h,A_h)\right\vert = \E_{d_h^{\pi^*}}\left\vert\int g_h^r(S_h,A_h,w_h)\mathrm{d}F^*_h(w_h|S_h) - \int \widehat{g}_h^r(S_h,A_h,w_h)\mathrm{d}F^{\widehat{\beta}}_h(w_h|S_h)\right\vert\\
&\quad\leq \E_{d_h^{\pi^*}}\left\vert\int(g_h^{r}-\widehat{g}_h^{r})(S_h,A_h,w_h)\mathrm{d}F_h^*(w_h|S_h)\right\vert\\
&\quad\quad + \E_{d_h^{\pi^*}}\left\vert\int \widehat{g}_h^{r}(S_h,A_h,w_h)\mathrm{d}(F^*_h - F_h^{\widehat{\beta}})(w_h|S_h)\right\vert\\
&\quad\leq \E_{d_h^{\pi^*}}\left\vert\E^*\left[(g_h^{r}-\widehat{g}_h^{r})(S_h,A_h,W_h)|S_h,A_h\right]\right\vert + \E_{d_h^{\pi^*}}\Vert(F^*_h - F_h^{\widehat{\beta}})(\cdot|S_h)\Vert_1\\
&\quad\leq \sqrt{C_1^*\E^{(1)}\left\vert\E^*\left[(g_h^{r}-\widehat{g}_h^{r})(S_h,A_h,W_h)|S_h,A_h\right]\right\vert^2} + \sqrt{C_2^*\E^{(2)}\Vert(F^*_h - F_h^{\widehat{\beta}})(\cdot|S_h)\Vert_1^2}.
\end{aligned}\label{eq:D.11}
\end{align}
Akin to (\ref{eq:D.7}), (\ref{eq:D.8}) and (\ref{eq:D.9}), we can derive that
\begin{align}
\E_{d_h^{\pi^*}}\left\vert(r^*_h - r_h^{\widehat{\beta}})(S_h,A_h)\right\vert &\leq \sqrt{\frac{C_1^*C_w}{L^2} \mathcal{L}_h^{r}(\widehat{g}_h^{r})} + \sqrt{C_2^*\mathcal{L}_h^F(F_h^{\widehat{\beta}})}.\label{eq:D.12}
\end{align}
Combine (\ref{eq:D.1}), (\ref{eq:D.4}), (\ref{eq:D.10}) and (\ref{eq:D.12}), we have
\begin{align}
J(\pi^*) - J(\widehat{\pi})\overset{\mathcal{E}}{\lesssim}\sum_{h=1}^H\left\{
\frac{H}{2\sigma_s}\sqrt{\frac{C_1^*C_w}{L^2} \sum_{i=1}^{d_\mathcal{S}}\mathcal{L}_h^{s(i)}(\widehat{g}_h^{s(i)})} +
\sqrt{\frac{C_1^*C_w}{L^2} \mathcal{L}_h^{r}(\widehat{g}_h^{r})} + \left(1+\frac{H\sqrt{d_\mathcal{S}}}{2\sigma_s}\right)\sqrt{C_2^*\mathcal{L}_h^F(\widehat{F}_h^{\widehat{\beta}})}\right\},\label{eq:D.13}
\end{align}
which concludes the proof.
\end{proof}

\subsection{Fast Convergence Rate of Minimax Estimators}
\label{sec:D.2}
In this section, we formally introduce the definition of minimax estimators and provide a theoretical analysis of its convergence rate.

\begin{definition}[Minimax Estimator]\label{def:D.1} Assume that $\{(y_j,z_j)\}_{j=1}^n\subset\mathcal{Y}\times\mathcal{Z}$ is a dataset of i.i.d. samples drawn from a distribution $\mu$. Let $\mathcal{G}\subset L^2(\mathcal{Y},\mu)$ and $\Theta\subset L^2(\mathcal{Z},\mu)$ be two function classes, and for every $g\in\mathcal{G}$, let $h_g\in L^2(\mathcal{Y},\mu)$ be a function induced by $g$. Then we define the objective function $\Phi:\mathcal{G}\times\Theta\to\mathbb{R}$ by
\begin{align}
\Phi(g,\theta)=\E\left[\E[h_g(Y)|Z]\theta (Z)\right] = \E\left[h_g(Y)\theta(Z)\right],\quad\textit{where}\ (Y,Z)\sim\mu.\label{eq:D.14}
\end{align}
Given $\lambda>0$, we define the stabilized objective function by
\begin{align}
\Phi^\lambda(g,\theta)=\Phi(g,\theta)-\lambda\Vert \theta\Vert_{2,\mu}^2 = \E\left[h_g(Y)\theta(Z)\right] - \lambda\E\left[\theta(Z)^2\right],\quad\textit{where}\ (Y,Z)\sim\mu.\label{eq:D.15}
\end{align}
Correspondingly, we write the empirical version of (\ref{eq:D.15}):
\begin{align}
\Phi^\lambda_n(g,\theta)= \frac{1}{n}\sum_{j=1}^n\left[h_g(y_j)\theta(z_j) - \lambda\theta(z_j)^2\right].\label{eq:D.16}
\end{align}
We assume that there exists $g^*\in\mathcal{G}$ such that
\begin{align}
g^*=\argmin_{g\in\mathcal{G}}\max_{\theta\in\Theta}\Phi^\lambda(g,\theta),\quad\textit{and}\quad \max_{g\in\mathcal{G}}\Phi^\lambda(g,\theta)=0.\label{eq:D.17}
\end{align}
To estimate $g^*$, we define the minimax estimator $\widetilde{g}$ as
\begin{align}
\widetilde{g}=\argmin_{g\in\mathcal{G}}\max_{\theta\in\Theta}\Phi^\lambda_n(g,\theta).\label{eq:D.18}
\end{align}
For any $g\in\mathcal{G}$, we define the RMSE loss as $\mathcal{L}(g)=\E\left[(\E[h_g(Y)|Z])^2\right]$, where $(Y,Z)\sim\mu$.
\end{definition}
We have the following important result for the convergence rate of minimax estimators with confidence region:
\begin{theorem}[\citealp{uehara2022finite}]\label{thm:D.1} Based on notations of Definition \hyperref[def:D.1]{D.1}, we further assume that $\Theta$ is a symmetric and star-shaped class and that $\sup_{g\in\mathcal{G}}\Vert h_g\Vert_2\leq C_{\mathcal{G}}$. We denote by $\vartheta_{\mathcal{G}}$ and $\vartheta_{\Theta}$ the upper bounds of critical radii of function classes $\mathcal{G}$ and $\Theta$, respectively. Then there exists $\Gamma_n=C\left(\vartheta_{\mathcal{G}}^2+\vartheta_\Theta^2+\log(1/\delta)/n\right)$ for some absolute constant $C>0$, such that the following properties hold with probability at least $1-\delta$:
\begin{itemize}
    \item[(i)] $g^*\in\mathrm{CR}(\Gamma_n):=\{g\in\mathcal{G}:\sup_{\theta\in\Theta}\Phi_n^\lambda(g,\theta)-\sup_{\theta\in\Theta}\Phi_n^\lambda(\widetilde{g},\theta)\leq\mathcal{O}(\Gamma_n)\}$;
    \item[(ii)] $\sup_{g\in\mathrm{CR}(\Gamma_n)}\mathcal{L}(g)\leq\mathcal{O}(\Gamma_n)$ for all $h\in[H]$.
\end{itemize}
\end{theorem}

\subsection{Proof of Lemmas \ref{lemma:C.2} and \ref{lemma:C.3}}
\label{sec:D.3}
We use results in \S\hyperref[sec:D.2]{D.2} to prove Lemmas \ref{lemma:C.2} and \ref{lemma:C.3}.
\renewcommand{\proofname}{\it Proof of Lemma \hyperref[lemma:C.2]{C.2}}
\begin{proof}
We set $Y=(S_h,A_h,W_h,R_h)$, $Z=(S_h,Z_h)$, $h_g(\cdot,\cdot,\cdot,R_h)=R_h-g(\cdot,\cdot,\cdot)$, $C_{\mathcal{G}}\leq\vert R_h\vert + \Vert g\Vert_\infty\leq 2$, $g^*=g_h^r$, $\mu=p^{(1)}$ and $\lambda=\frac{1}{2}$ in Definition \hyperref[def:D.1]{D.1}. According to Assumption \hyperref[assumption:6.2]{6.2}, the conditions of Theorem \hyperref[thm:D.1]{D.1} hold. Furthermore, the function $\widetilde{g}_h^r$ in (\ref{eq:6.10}) is the minimax estimator of $g^*$. By Theorem \hyperref[thm:D.1]{D.1}, we set $$\alpha_1=\mathcal{O}\left(\vartheta_{\mathcal{G}}^2+\vartheta_\Theta^2 + \frac{\log(3H/\delta)}{K_1}\right)$$ in Algorithm \hyperref[alg:2]{2}. We adopt the union bound of events described in \hyperref[thm:D.1]{D.1} for all $h\in[H]$, each with probability at least $1-\delta/(3H)$. The following properties hold with probability at least $1-\delta/3$:
\begin{itemize}
    \item[(i)] The event $\mathcal{E}_1:=\{g_h^r\in C_h^r(\alpha_1),\ \forall h\in[H]\}$ holds;
    \item[(ii)] For all $h\in[H]$, $\mathcal{L}_h^r(\widehat{g}_h^r)\leq\sup_{g\in C_h^r(\alpha_1)}\mathcal{L}_h^r(g)\leq\mathcal{O}(\alpha_1)$.\vspace{0.1cm}
\end{itemize}

According to \cite{Dikkala2020}, the linear function classes $\mathcal{G}$ and $\Theta$ given in Assumption \hyperref[assumption:6.2]{6.2} have critical radii bounded by
\begin{align}
\vartheta_{\mathcal{G}}=\vartheta_{\Theta}=\sqrt{\frac{d\log K_1}{K_1}}.\label{eq:D.19}
\end{align}
Consequently, we can specify $\alpha_1$ as $\alpha_1=C_1d\log(HK_1/\delta)/K_1$ for some absolute constant $C_1>0$, which concludes the proof of Lemma \hyperref[lemma:C.2]{C.2}. 
\end{proof}

\renewcommand{\proofname}{\it Proof of Lemma \hyperref[lemma:C.3]{C.3}}
\begin{proof}
The proof is analogous to \hyperref[lemma:C.2]{C.2}, except that we need to apply Theorem \hyperref[thm:D.1]{D.1} to every component $i\in[d_\mathcal{S}]$ and construct the union bound. We set $Y=(S_h,A_h,W_h,S_{h+1}(i))$, $Z=(S_h,Z_h)$, $h_g(\cdot,\cdot,\cdot,S_{h+1}(i))=S_{h+1}(i)-g(\cdot,\cdot,\cdot)$, $C_{\mathcal{G}}\leq\vert S_{h+1}^{(i)}\vert + \Vert g\Vert_\infty\leq 2$, $g^*=g_h^{s(i)}$, $\mu=p^{(1)}$ and $\lambda=\frac{1}{2}$ to align with Definition \hyperref[def:D.1]{D.1}. Clearly, Assumption \hyperref[assumption:6.2]{6.2} of $\mathcal{G}$ and $\Theta$ is consistent with the conditions of Theorem \hyperref[thm:D.1]{D.1} hold, and the function $\widetilde{g}_h^{s(i)}$ in (\ref{eq:6.13}) is the minimax estimator of $g^*$. According to \hyperref[thm:D.1]{D.1}, we set $$\alpha_2=\mathcal{O}\left(\vartheta_{\mathcal{G}}^2+\vartheta_\Theta^2 + \frac{\log(3Hd_\mathcal{S}/\delta)}{K_1}\right)$$ in Algorithm \hyperref[alg:2]{2}. Then the following properties hold with probability at least $1-\delta/3$:
\begin{itemize}
    \item[(i)] The event $\mathcal{E}_2:=\{g_h^{s(i)}\in C_h^r(\alpha_2),\ \forall h\in[H],\forall i\in{\mathcal{S}}\}$ holds;
    \item[(ii)] $\mathcal{L}_h^{s(i)}(\widehat{g}_h^{s(i)})\leq\sup_{g\in C_h^{s(i)}(\alpha_2)}\mathcal{L}_h^{s(i)}(g)\leq\mathcal{O}(\alpha_2)$ for all $h\in[H]$ and all $i\in[d_{\mathcal{S}}]$.\vspace{0.1cm}
\end{itemize}

Plugging in (\ref{eq:D.19}), we can specify the parameter $\alpha_2$ as $\alpha_2=C_2d\log(Hd_{\mathcal{S}}K_1/\delta)/K_1$ for some absolute constant $C_2>0$. This finish the proof of Lemma \hyperref[lemma:C.3]{C.3}.
\end{proof}

\subsection{Proof of Lemma \ref{lemma:C.4}}\label{sec:D.4}
For simplicity, we assume all our operations in this proof are conducted in domain $\mathcal{M}^{(2)}$ and omit the superscript $(2)$. We fix $h\in[H]$, and assume the ridge regression estimator $\widetilde{\mathbf{M}}_h$ is obtained from the offline dataset $\mathcal{D}=\left\{(S_h^k,W_h^k)\right\}_{k=1}^K\overset{i.i.d.}{\sim} p^{(2)}$, where $W_h^k=\mathbf{M}_h^\top\phi(S_h^k)+\epsilon_h^k$, with $\{\epsilon_k\}\overset{i.i.d.}{\sim} N(0,\sigma_w^2\mathbf{I})$, by Assumption \hyperref[assumption:6.1]{6.1}. Luckily, the optimization problem \eqref{eq:6.13} has a solution of closed form:
\begin{align}
\widetilde{\mathbf{M}}_h = \sum_{k=1}^K W_h^k\phi(S_h^k)^\top\Lambda_h^{-1},\quad\textit{where}\quad\Lambda_h=\lambda\mathbf{I}+\sum_{k=1}^K\phi(S_h^k)\phi(S_h^k)^\top.\label{eq:D.20}
\end{align}
We introduce some technical lemmas.
\begin{lemma}[See Theorem 20 in \citealp{chang2021cov_shift}]\label{lemma:D.1} Let $\Sigma_h=\E[\phi(S_h)\phi(S_h)^\top]$, where the expectation is taken with respect to $S_h\sim p^{(2)}$, and let $\lambda$ be a constant of which the choice is independent of $K$. For any $0<\delta<1$, with probability at least $1-\delta$, the following event holds:
\begin{align}
\mathcal{E}_{3,h}^\prime = \left\{\phi(s)\Lambda_h^{-1}\phi(s)\leq c_1\left\{\mathrm{rank}(\Sigma_h)+\log(c_2/\delta)\right\}\phi(s)^\top(K\Sigma_h+\lambda\mathbf{I})^{-1}\phi(s),\ \ \forall s\in\mathcal{S}\right\},\label{eq:D.21}
\end{align}
where $c_1,c_2>0$ are some universal constants. Conditional on $\mathcal{E}_{3,h}^\prime$, one have
\begin{align}
\log\left(\frac{\det(\Lambda_h)}{\det(\lambda\mathbf{I})}\right)\leq c_1\mathrm{rank}(\Sigma_h)\left\{\mathrm{rank}(\Sigma_h)+\log(c_2/\delta)\right\}\log(1+K).\label{eq:D.22}
\end{align}
\end{lemma}
\begin{lemma}[See Theorem 21 in \citealp{chang2021cov_shift}]\label{lemma:D.2} Based on the notation of Lemma \ref{lemma:D.1}, we define $\Vert x\Vert_{\Sigma}=\Vert\Sigma^{-1/2}x\Vert$. Then for any $0<\delta<1$, conditional on $\mathcal{E}_{3,h}^\prime$, it holds
\begin{align}
\E\left[\left\Vert\phi(S_h)\right\Vert_{\Lambda_h}^2\right]\leq\frac{c_1\mathrm{rank}(\Sigma_h)\left\{\mathrm{rank}(\Sigma_h)+\log(c_2/\delta)\right\}}{K},\label{eq:D.23}
\end{align}
where the expectation is taken with respect to $S_h\sim p^{(2)}$.
\end{lemma}

\begin{lemma}[Self-normalized bound for matrix-valued martingales, see Lemma C.4 in \citealp{kakade2020information}]\label{lemma:D.3} Let $(\epsilon_t)_{t=1}^\infty$ be a $\mathbb{R}^d$-valued martingale difference sequence with respect to the filtration $(\mathscr{F}_t)_{t=1}^\infty$ such that $\epsilon_t$ is $\mathscr{F}_{t-1}$-conditionally $\sigma$-sub-Gaussian for some $\sigma>0$. Let $(X_t)_{t=1}^\infty$ be an adapted process taking values in $\mathbb{R}^d$, and let $V_0\in\mathbb{R}^{d\times d}$ be a positive definite matrix. For any $t\in\mathbb{N}$, define $V_t=V_0+\sum_{j=1}^tX_jX_j^\top$. Then with probability at least $1-\delta$, we have
\begin{align}
    \left\Vert\sum_{j=1}^t\epsilon_tX_j^\top V_t^{-1/2}\right\Vert^2_2\leq (8\log 5)\sigma^2 + 8\sigma^2\log\left(\frac{\det(V_t)^{1/2}\det(V_0)^{-1/2}}{\delta}\right),\quad\forall t\in\mathbb{N}.\label{eq:D.24}
\end{align}
\end{lemma}
\renewcommand{\proofname}{\it Proof of Lemma \hyperref[lemma:C.4]{C.4}}
\begin{proof}
We transform the estimation error as follows:
\begin{align}
\widetilde{\mathbf{M}}_h-\mathbf{M}_h&=\sum_{k=1}^K\left(\mathbf{M}_h^\top\phi(S_h^k)+\epsilon_h^k\right)\phi(S_h^k)^\top\Lambda_h^{-1} - \mathbf{M}_h\notag\\
&=\mathbf{M}_h^\top\left(\sum_{k=1}^K\phi(S_h^k)\phi(S_h^k)^\top\right)\Lambda_h^{-1} - \mathbf{M}_h^\top\Lambda_h\Lambda_h^{-1} + \sum_{k=1}^K\epsilon_h^k\phi(S_h^k)^\top\Lambda_h^{-1}\notag\\
&= \lambda\mathbf{M}_h^\top\Lambda_h^{-1} + \sum_{k=1}^K\epsilon_h^k\phi(S_h^k)^\top\Lambda_h^{-1}.\label{eq:D.25}
\end{align}
For any $0<\delta<1$, by Lemma \hyperref[lemma:D.3]{D.3}, with probability at least $1-\delta$, we have
\begin{align}
\left\Vert\left(\widetilde{\mathbf{M}}_h-\mathbf{M}_h\right)\Lambda_h^{1/2}\right\Vert_2&\leq\lambda\left\Vert\mathbf{M}_h^\top\Lambda_h^{-1/2}\right\Vert_2+\left\Vert\sum_{k=1}^K\epsilon_h^k\phi(S_h^k)^\top\Lambda_h^{-1/2}\right\Vert_2\notag\\
&\leq\sqrt{\lambda}\Vert\mathbf{M}_h\Vert_2 + \sqrt{(8\log 5)\sigma_w^2 d_{\mathcal{W}} + 8\sigma_w^2\log\left(\frac{\det(\Lambda_h)^{1/2}}{\delta\det(\lambda\mathbf{I})^{1/2}}\right)}.\label{eq:D.26}
\end{align}
To ensure that $\P(\mathcal{E}_3)\leq 1-\delta/3$ (See (\ref{eq:C.11}) and (\ref{eq:6.14}) for the definition of $\mathcal{E}_3$), we take
\begin{align}
\alpha_3=\sup_{h\in[H]}\sqrt{2\lambda\Vert\mathbf{M}_h\Vert_2^2 + (16\log 5)\sigma_w^2 d_{\mathcal{W}}+16\sigma^2_w\log\left(\frac{6H}{\delta}\right)+8\sigma^2_w\log\left(\frac{\det(\Lambda_h)}{\det(\lambda\mathbf{I})}\right)}.\label{eq:D.27}
\end{align}
Consequently, we have $\P(\mathcal{E}_3^{\prime\prime})\leq 1-\delta/6$, where
\begin{align}
    \mathcal{E}_3^{\prime\prime}=\left\{\mathbf{M}_h\in C_h^M(\alpha_3),\quad \forall h\in[H]\right\}.\label{eq:D.28}
\end{align}
Since $\widehat{\mathbf{M}}_h\in C_h^M(\alpha_3)$, we use the triangle inequality:
\begin{align}
    \left\Vert\left(\widehat{\mathbf{M}}_h-\mathbf{M}_h\right)\Lambda_h^{1/2}\right\Vert_2\leq\left\Vert\left(\widetilde{\mathbf{M}}_h-\mathbf{M}_h\right)\Lambda_h^{1/2}\right\Vert_2+\left\Vert\left(\widehat{\mathbf{M}}_h-\widetilde{\mathbf{M}}_h\right)\Lambda_h^{1/2}\right\Vert_2\overset{\mathcal{E}_3}{\leq} 2\alpha_3.\label{eq:D.29}
\end{align}
On the other hand, we let $\lambda=\Omega(1)$ to be a constant, and take the event
\begin{align*}
    \mathcal{E}_3^\prime = \left\{\phi(s)\Lambda_h^{-1}\phi(s)\leq c_1\left\{\mathrm{rank}(\Sigma_h)+\log\frac{6c_2H}{\delta}\right\}\phi(s)^\top(K\Sigma_h+\lambda\mathbf{I})^{-1}\phi(s),\ \ \forall s\in\mathcal{S},\forall h\in[H]\right\}.
\end{align*}
By Lemma \hyperref[lemma:D.1]{D.1}, we have the union bound $\P(\mathcal{E}_3^\prime)\geq 1-\delta/6$, and $\P(\mathcal{E}_3)=\P(\mathcal{E}_3^\prime\cap\mathcal{E}_3^{\prime\prime})\geq 1-\delta/3$. Furthermore, (\ref{eq:D.22}) and (\ref{eq:D.27}) together give
\begin{align}
\alpha_3=\mathcal{O}\left(\sup_{h\in[H]}\sqrt{\Vert\mathbf{M}_h\Vert_2^2 + \sigma_w^2d_{\mathcal{W}}+\sigma_w^2\mathrm{rank}(\Sigma_h)\left\{\mathrm{rank}(\Sigma_h)+\log(c_2H/\delta)\right\}\log(1+K)}\right).\label{eq:D.30}
\end{align}

Now we bound the loss function $\mathcal{L}_h^F(F^{\widehat{\beta}})$ on the event $\mathcal{E}_3$. Akin to (\ref{eq:D.5}), we use Pinsker's inequality to bound the total variation distance between two Gaussian distributions:
\begin{align}
    \left\Vert (F_h^*-F^{\widehat{\beta}}_h)(\cdot|S_h)\right\Vert_1^2 \leq \frac{1}{4\sigma_w^2}\left\Vert(\mathbf{M}_h-\widehat{\mathbf{M}}_h)\phi(S_h)\right\Vert_2^2\leq \frac{1}{4\sigma_w^2}\left\Vert\left(\widehat{\mathbf{M}}_h-\mathbf{M}_h\right)\Lambda_h^{1/2}\right\Vert_2\left\Vert\phi(S_h)\right\Vert_{\Lambda_h}^2.\label{eq:D.31}
\end{align}
Combining (\ref{eq:D.5}), (\ref{eq:D.30}), (\ref{eq:D.29}) and (\ref{eq:D.23}) and Assumption \hyperref[assumption:6.1]{6.1 (i)}, we have
\begin{align}
    \mathcal{L}_h^F(F^{\widehat{\beta}})=\E\left[\left\Vert (F_h^*-F^{\widehat{\beta}}_h)(\cdot|S_h)\right\Vert_1^2\right]\leq\frac{c_1\alpha_3^2}{K\sigma_w^2}\overline{R}=\mathcal{O}\left(\frac{m\overline{R}}{K\sigma_w^2}+\frac{d_{\mathcal{W}}\overline{R}}{K}+\frac{\overline{R}^2\log(1+K)}{K}\right),\label{eq:D.32}
\end{align}
where $\overline{R}=\max_{h\in[H]}\mathrm{rank}(\Sigma_h)\left\{\mathrm{rank}(\Sigma_h)+\log(c_2H/\delta)\right\}$. The expressions (\ref{eq:D.30}) and (\ref{eq:D.32}) complete the proof of Lemma \hyperref[lemma:C.4]{C.4}.
\end{proof}

\renewcommand{\proofname}{\it Proof}

%% file: appendix_scm.tex

\section{Discrete CCB-Neg Model and Evaluation Protocol}
\label{app:scm}


In \S\ref{sec:exp-ccbneg}, we consider a discrete CCB-Neg model with context
$X\in[K]=\{1,\ldots,K\}$, action $A\in\{0,\ldots,M-1\}$, negative outcome $W$, action instrument $Z_A\in\mathbb{R}^{M}$, negative-outcome instrument $Z_W\in\mathbb{R}$, and latent confounder $U$. For each domain $d\in\{\star,1,2\}$, the structural equations are
\begin{equation}
\label{eq:ccb-neg-dgp}
\begin{cases}
W
=
\omega_X^{(d)}
+
\kappa_W^{(d)}Z_W
+
\rho_W U
+
\sigma_W\varepsilon_W,
\\[2pt]
A
=
\displaystyle
\arg\max_{a\in\{0,\ldots,M-1\}}
\left\{
\alpha_{a,X}^{(d)}
+
\kappa_A Z_{A,a}
+
\rho_a U
+
\sigma_A\varepsilon_{A,a}
\right\},
\\[2pt]
Y
=
f^{(d)}(W,A,X)
+
\rho_Y U
+
\sigma_Y\varepsilon_Y .
\end{cases}
\end{equation}
Absorbing the context- and action-specific terms into cell coefficients, we write
\[
f^{(d)}(w,a,x)=c^{(d)}_{x,a}+s^{(d)}_{x,a}w,
\qquad
c^{(d)}_{x,a}:=\mu^{(d)}_x+\tau^{(d)}_{x,a},
\qquad
s^{(d)}_{x,a}:=b^{(d)}_x+\lambda^{(d)}_a.
\]
This reparameterization removes the redundant decomposition into $\mu_x$, $\tau_{x,a}$, $b_x$, and $\lambda_a$ while leaving the data-generating process unchanged.

The two source domains exhibit complementary discrepancies. Source $1$ preserves the target reward mechanism but shifts the $W$-mechanism and behavior policy, so $\Delta^{(1)}=\{A,W\}$. Source $2$ preserves the target $W$-mechanism but shifts the reward mechanism and behavior policy, so $\Delta^{(2)}=\{A,Y\}$. Consequently, neither source alone identifies the target causal mean. Combining the reward bridge identified in source $1$ with the conditional distribution of $W$ from source $2$ gives
\begin{equation}
\label{eq:ccb-neg-transport}
g^\star(a,x)
=
\mathbb{E}^{(2)}
\left[
h_1^\star(W,a,x)
\mid X=x\right].
\end{equation}
Under the centering conditions
$\mathbb{E}[Z_W\mid X]
=\mathbb{E}[U\mid X]
=\mathbb{E}[\varepsilon_W\mid X]
=0$, we have
\begin{equation}
\label{eq:true-causal-mean}
g^\star(a,x)
=
c_{x,a}^\star
+
s_{x,a}^\star\omega_x^\star,\qquad\text{where}\ \omega_x^\star:=\mathbb{E}^{(2)}[W\mid X=x].
\end{equation}
The target-optimal policy $\pi^\star(x)=\arg\max_a g^\star(a,x)$ is therefore available exactly when evaluating the simulation.

\subsection{Complete Estimation Procedure for the CCB-Neg Experiment}
\label{app:ccb-neg-estimation}
In this subsection, we provide the complete identification and estimation procedure for the tabular CCB-Neg model.
\paragraph{Parameterization.}
Let $e_x^K$ and $e_a^M$ denote the canonical basis vectors corresponding to the context $x\in[K]$ and action $a\in\{0,1,\cdots,M\}$, respectively. Define
\begin{equation}
\phi_h(w,a,x)
= e_x^K\otimes e_a^M\otimes \begin{pmatrix}
1\\ w \end{pmatrix}
\in\mathbb{R}^{2KM},
\qquad
\phi_g(a,x)
= e_x^K\otimes e_a^M
\in\mathbb{R}^{KM}.
\end{equation}
The bridge and target causal-mean models are
\begin{equation}
h_\beta(w,a,x)=\phi_h(w,a,x)^\top\beta,
\qquad
g_\gamma(a,x)=\phi_g(a,x)^\top\gamma,
\end{equation}
where
\begin{equation}
\beta
=
\operatorname{vec}
\left\{
(c_{x,a},s_{x,a}):x\in[K],\ a\in\{0,\ldots,M-1\}
\right\}
\in\mathbb{R}^{2KM}
\end{equation}
stacks the cell-specific bridge intercepts and slopes, and
\begin{equation}
\gamma
=
\operatorname{vec}
\left\{
g(a,x): x\in[K],\ a\in\{0,\ldots,M-1\}
\right\}
\in\mathbb{R}^{KM}.
\end{equation}
stacks the target causal means for all context-action cell. Thus,
\begin{equation}
\theta
=
\begin{pmatrix}
\beta\\ \gamma
\end{pmatrix}
\in\mathbb{R}^{3KM}
\end{equation}
specifies a full hypothesis in the bridge function class.

\paragraph{Integral-equation system.}
Because source $1$ preserves the reward mechanism, instrumental identification gives
\begin{equation}
\label{eq:source1-conditional-moment}
\mathbb{E}^{(1)}
\left[
h_\beta(W,A,X)-Y
\mid
X,Z_A,Z_W
\right] =0.
\end{equation}
The conditioning variables do not include the realized action. Conditioning on $A$ would condition on a descendant of $U$ and would generally reintroduce action-selection bias.

Because the conditional distribution of $W$ in source $2$ agrees with its target counterpart, the transportation equation is
\begin{equation}
\label{eq:source2-conditional-moment}
\mathbb{E}^{(2)}
\left[
g_\gamma(a,X)-h_\beta(W,a,X)
\mid X
\right]
= 0,
\qquad
a\in\{0,\ldots,M-1\}.
\end{equation}
Equations~\eqref{eq:source1-conditional-moment} and
\eqref{eq:source2-conditional-moment} form the complete
integral-equation system used in the experiment.

For a finite action space, we evaluate \eqref{eq:source2-conditional-moment} for every action and stack the resulting $M$ equations. This is equivalent at the population level to introducing an independent pseudo-action $A'\sim\operatorname{Unif}\{0,1,\cdots,M-1\}$, but it avoids unnecessary Monte Carlo variation from sampling $A'$.

\paragraph{Choice of test functions.}
For the source-$1$ equation, define
\begin{equation}
q_A(z_A)
=
\begin{pmatrix}
1\\
z_{A,1}-z_{A,0}\\
\vdots\\
z_{A,M-1}-z_{A,0}
\end{pmatrix}
\in\mathbb{R}^{M},
\qquad
q_W(z_W)
=
\begin{pmatrix}
1\\z_W
\end{pmatrix}
\in\mathbb{R}^{2},
\end{equation}
and
\begin{equation}
\psi_1(x,z_A,z_W)
=
e_x^{K}
\otimes
q_A(z_A)
\otimes
q_W(z_W)
\in\mathbb{R}^{2KM}.
\end{equation}
The corresponding linear test-function class is
\begin{equation}
\Theta_1
=\left\{
(x,z_A,z_W)
\mapsto
v_1^\top\psi_1(x,z_A,z_W)
:
v_1\in\mathbb{R}^{2KM}
\right\}.
\end{equation}
The constant coordinate identifies the overall level, the action-instrument contrasts identify differences among action cells, and their interactions with $Z_W$ identify the cell-specific coefficients multiplying $W$.

For each source-$2$ equation, we use the saturated context test class
\begin{equation}
\psi_2(x)=e_x^{K},
\qquad
\Theta_{2,a}
=
\left\{
x\mapsto v_{2,a}^\top e_x^{K}
:
v_{2,a}\in\mathbb{R}^{K}
\right\}.
\end{equation}
Hence all test functions are linear.
\paragraph{Empirical moment equations.}
Each conditional moment equation is associated with the probability law under which its expectation is taken. For the source-$1$ equation, let
\begin{equation}
\mu_1
=
p^{(1)}
(X,Z_A,Z_W,W,A,Y),\quad\and\quad\mathcal D_1
=
\left\{
X_i^{(1)},W_i^{(1)},A_i^{(1)},Y_i^{(1)},
Z_{A,i}^{(1)},Z_{W,i}^{(1)}
\right\}_{i=1}^{n_1}.
\end{equation}
For the source-$2$ equation associated with action $a$, let
\begin{equation}
\mu_{2,a}
=
p^{(2)}(X,W),\quad\text{and}\quad
\mathcal D_2
=
\left\{
X_i^{(2)},W_i^{(2)}
\right\}_{i=1}^{n_2}
\end{equation}
Here all source-$2$ equations use the same probability law, but we retain the subscript $a$ to distinguish the corresponding moment blocks.
\begin{itemize}
\item The empirical source-$1$ moment is
\begin{align}
\widehat m_1(\beta)
&=
\frac{1}{n_1}
\sum_{i=1}^{n_1}
\psi_{1i}
\left\{
h_\beta(W_i^{(1)},A_i^{(1)},X_i^{(1)})
-
Y_i^{(1)}
\right\}=\widehat J_{1\beta}\beta-\widehat q_1,
\label{eq:empirical-source1-moment}
\end{align}
where
\begin{equation}
\psi_{1i}
=
\psi_1
\left(
X_i^{(1)},Z_{A,i}^{(1)},Z_{W,i}^{(1)}
\right),
\end{equation}
\begin{equation}
\widehat J_{1\beta}
=
\frac{1}{n_1}
\sum_{i=1}^{n_1}
\psi_{1i}
\phi_h
\left(
W_i^{(1)},A_i^{(1)},X_i^{(1)}
\right)^\top
\in\mathbb{R}^{2KM\times 2KM},
\end{equation}
and
\begin{equation}
\widehat q_1
=
\frac{1}{n_1}
\sum_{i=1}^{n_1}
\psi_{1i}Y_i^{(1)}
\in\mathbb{R}^{2KM}.
\end{equation}

\item For each action $a$, the empirical source-$2$ moment is
\begin{align}
\widehat m_{2,a}(\beta,\gamma)
&=
\frac{1}{n_2}
\sum_{i=1}^{n_2}
e_{X_i^{(2)}}^{K}
\left\{
g_\gamma(a,X_i^{(2)})
-
h_\beta(W_i^{(2)},a,X_i^{(2)})
\right\}=
\widehat J_{2,a,\beta}\beta
+
\widehat J_{2,a,\gamma}\gamma,
\label{eq:empirical-source2-moment}
\end{align}
where
\begin{equation}
\widehat J_{2,a,\beta}
=
-
\frac{1}{n_2}
\sum_{i=1}^{n_2}
e_{X_i^{(2)}}^{K}
\phi_h(W_i^{(2)},a,X_i^{(2)})^\top
\in\mathbb{R}^{K\times 2KM}
\end{equation}
and
\begin{equation}
\widehat J_{2,a,\gamma}
=
\frac{1}{n_2}
\sum_{i=1}^{n_2}
e_{X_i^{(2)}}^{K}
\phi_g(a,X_i^{(2)})^\top
\in\mathbb{R}^{K\times KM}.
\end{equation}
\end{itemize}
Stacking all moment blocks gives
\begin{equation}
\widehat m(\theta)
=
\begin{pmatrix}
\widehat m_1(\beta)\\
\widehat m_{2,0}(\beta,\gamma)\\
\vdots\\
\widehat m_{2,M-1}(\beta,\gamma)
\end{pmatrix}
=
\widehat J\theta-\widehat q
\in\mathbb{R}^{3KM},
\label{eq:complete-stacked-moment}
\end{equation}
with the complete Jacobian and intercept
\begin{equation}
\widehat J
=
\begin{pmatrix}
\widehat J_{1\beta}
&
0_{2KM\times KM}
\\
\widehat J_{2,0,\beta}
&
\widehat J_{2,0,\gamma}
\\
\vdots
&
\vdots
\\
\widehat J_{2,M-1,\beta}
&
\widehat J_{2,M-1,\gamma}
\end{pmatrix}
\in\mathbb{R}^{3KM\times 3KM}
\quad
\text{and}
\quad
\widehat q
=
\begin{pmatrix}
\widehat q_1\\
0_K\\
\vdots\\
0_K
\end{pmatrix}
\in\mathbb{R}^{3KM}.
\label{eq:complete-jacobian-q}
\end{equation}
Now we derive the empirical squared norms of test functions. Using the linear parameterizations
\[
\vartheta_1(X,Z_A,Z_W)
=
v_1^\top\psi_1(X,Z_A,Z_W)
\quad\text{and}\quad
\vartheta_{2,a}(X)
=
v_{2,a}^\top e_X^K,
\]
the empirical squared norms become
\[
\|\vartheta_1\|_{\mathcal{D}_1,2}^2
=
\frac{1}{n_1}
\sum_{i=1}^{n_1}
\left(v_1^\top\psi_{1i}\right)^2
=
v_1^\top\widehat G_1v_1,\quad\text{where}
\ 
\widehat G_1
=\frac{1}{n_1}
\sum_{i=1}^{n_1}
\psi_{1i}\psi_{1i}^\top,
\]
and
\[
\|\vartheta_{2,a}\|_{\mathcal D_{2},2}^2
=\frac{1}{n_2}
\sum_{i=1}^{n_2}
\left(v_{2,a}^\top e_{X_i^{(2)}}^K\right)^2
=
v_{2,a}^\top\widehat G_2v_{2,a},\quad
\text{where}\ 
\widehat G_2
=
\frac{1}{n_2}
\sum_{i=1}^{n_2}
e_{X_i^{(2)}}^K
(e_{X_i^{(2)}}^K)^\top.
\]

\paragraph{Closed-form minimax loss.}
For fixed $(\beta,\gamma)$, the empirical minimax criterion \eqref{eq:4.3} is given by
\begin{equation}
\mathcal{L}_{\mathcal{D}}(\beta,\gamma)
=\sup_{v_1,\{v_{2,a}\}_{a=0}^{M-1}}
\Bigg\{
v_1^\top\widehat m_1(\beta)
+
\sum_{a=0}^{M-1}
v_{2,a}^\top
\widehat m_{2,a}(\beta,\gamma)
-
\frac{1}{2}
v_1^\top\widehat G_1v_1
-
\frac{1}{2}
\sum_{a=0}^{M-1}
v_{2,a}^\top\widehat G_2v_{2,a}
\Bigg\}.
\label{eq:expanded-minimax-loss}
\end{equation}
Let
\[
v=\begin{pmatrix}
v_1\\
v_{2,0}\\
\vdots\\
v_{2,M-1}
\end{pmatrix},
\qquad
\widehat m(\theta)
=\begin{pmatrix}
\widehat m_1(\beta)\\
\widehat m_{2,0}(\beta,\gamma)\\
\vdots\\
\widehat m_{2,M-1}(\beta,\gamma)
\end{pmatrix},
\]
and define
\begin{equation}
\widehat\Omega=\operatorname{diag}
\left(\widehat G_1,
I_M\otimes\widehat G_2\right).\label{eq:complete-test-gram}
\end{equation}
Then
\eqref{eq:expanded-minimax-loss} can be written compactly as
\[
\mathcal{L}_{\mathcal{D}}(\theta)
=\sup_{v\in\mathbb R^{3KM}}
\left\{
v^\top\widehat m(\theta)
-
\frac{1}{2}v^\top\widehat\Omega v
\right\}.
\]
Provided that $\widehat\Omega$ is positive definite, the maximizing test-function coefficient is
\[
\widehat v(\theta)
=
\widehat\Omega^{-1}\widehat m(\theta).
\]
Substituting it into the objective gives
\begin{equation}
    \mathcal{L}_{\mathcal{D}}(\theta)
=\frac{1}{2}
\widehat m(\theta)^\top
\widehat\Omega^{-1}
\widehat m(\theta)
=
\frac{1}{2}
(\widehat J\theta-\widehat q)^\top
\widehat\Omega^{-1}
(\widehat J\theta-\widehat q).\label{eq:compact-minimax-loss}
\end{equation}

\paragraph{Optional numerical regularization.}
No additional regularization parameter is required when $\widehat G_1$ and $\widehat G_2$ are nonsingular. If either empirical Gram matrix is nearly singular, we may stabilize the test-function optimization by replacing $\widehat G_1$ and $\widehat G_2$ with
\[
\widehat G_{1,\eta}
=\widehat G_1+\eta_1 I_{2KM},
\qquad
\widehat G_{2,\eta}
=\widehat G_2+\eta_2 I_K,
\]
where $\eta_1,\eta_2>0$ are small numerical ridge parameters. This corresponds to adding
\[
-\frac{\eta_1}{2}\|v_1\|_2^2
-
\frac{\eta_2}{2}
\sum_{a=0}^{M-1}\|v_{2,a}\|_2^2
\]
to the supremum objective. The stabilized matrix is then
\[
\widehat\Omega_\eta
=
\operatorname{diag}
\left(
\widehat G_1+\eta_1 I_{2KM},
I_M\otimes
\left(\widehat G_2+\eta_2 I_K\right)
\right).
\]
The parameters $\eta_1$ and $\eta_2$ are numerical stabilization parameters rather than confidence-set radii or statistical covariance parameters. 

\paragraph{Point estimator.} We define
\begin{equation}
\label{eq:loss-curvature}
\widehat Q
:=\widehat J^\top
\widehat\Omega^{-1}
\widehat J.
\end{equation}
which is the Hessian, or curvature matrix, of the quadratic minimax loss \eqref{eq:compact-minimax-loss} with respect to $\theta$. Under the full-rank condition on $\widehat J$, $\widehat Q$ is positive definite and the empirical minimizer is
\begin{equation}
\label{eq:closed-form-theta-estimator}
\widehat\theta
=
\argmin_{\theta}
\mathcal{L}_{\mathcal D}(\theta)
=
\widehat Q^{-1}
\widehat J^\top
\widehat\Omega^{-1}
\widehat q.
\end{equation}
Because $\widehat J$ is square in the present saturated
specification, this expression reduces to
$\widehat\theta=\widehat J^{-1}\widehat q$ whenever
$\widehat J$ is nonsingular. Write
\[
\widehat\theta
=
\begin{pmatrix}
\widehat\beta\\
\widehat\gamma
\end{pmatrix},
\qquad
\widehat g(a,x)
=
\phi_g(a,x)^\top\widehat\gamma.
\]
The plug-in greedy policy used as the primary CAT estimator is
\begin{equation}
\label{eq:cat-greedy-policy}
\widehat\pi_{\mathrm{greedy}}(x)
\in
\argmax_{a\in\{0,\ldots,M-1\}}
\widehat g(a,x).
\end{equation}
Equivalently, if $\Pi$ denotes a policy class and
$[d_\pi]_{x,a}=p^\star(x)\pi(a\mid x)$, then
\[
\widehat\pi_{\mathrm{greedy}}
\in
\argmax_{\pi\in\Pi}
d_\pi^\top\widehat\gamma.
\]

\paragraph{Loss ellipsoid and its projection.}
Completing the square around $\widehat\theta$ gives
\begin{equation}
\mathcal{L}_{\mathcal D}(\theta)
-
\mathcal{L}_{\mathcal D}(\widehat\theta)
=
\frac{1}{2}
(\theta-\widehat\theta)^\top
\widehat Q
(\theta-\widehat\theta).
\end{equation}
For a radius $r>0$, the corresponding loss-level uncertainty set is
\begin{equation}
\label{eq:joint-loss-ellipsoid}
\mathcal C_\theta(r)
= \left\{
\theta:
\frac{1}{2}
(\theta-\widehat\theta)^\top
\widehat Q
(\theta-\widehat\theta)
\le r
\right\}.
\end{equation}
Partition $\widehat Q$ according to $\theta=(\beta^\top,\gamma^\top)^\top$:
\begin{equation}
\widehat Q
=
\begin{pmatrix}
\widehat Q_{\beta\beta}
&
\widehat Q_{\beta\gamma}
\\
\widehat Q_{\gamma\beta}
&
\widehat Q_{\gamma\gamma}
\end{pmatrix}.
\end{equation}
The exact projection of $\mathcal C_\theta(r)$ onto the target causal-mean coefficients is
\begin{equation}
\label{eq:projected-gamma-ellipsoid}
\mathcal C_\gamma(r)
=
\left\{
\gamma:
\frac{1}{2}
(\gamma-\widehat\gamma)^\top
\widehat S_\gamma
(\gamma-\widehat\gamma)
\le r
\right\},
\end{equation}
where
\begin{equation}
\label{eq:gamma-schur-complement}
\widehat S_\gamma
=
\widehat Q_{\gamma\gamma}
-
\widehat Q_{\gamma\beta}
\widehat Q_{\beta\beta}^{-1}
\widehat Q_{\beta\gamma},
\qquad
\widehat S_\gamma^{-1}
=
[\widehat Q^{-1}]_{\gamma\gamma}.
\end{equation}
The Schur complement accounts for the propagation of uncertainty in the bridge coefficients $\beta$ into the target causal means $\gamma$. For any policy $\pi$, its value under the causal-mean parameter $\gamma$ is
\[
V(\pi;\gamma)
=
d_\pi^\top\gamma.
\]
The worst-case value over the projected ellipsoid has the closed-form expression
\begin{equation}
\label{eq:closed-form-pessimistic-value}
\inf_{\gamma\in\mathcal C_\gamma(r)}
d_\pi^\top\gamma
=
d_\pi^\top\widehat\gamma
-
\sqrt{
2r\,
d_\pi^\top
\widehat S_\gamma^{-1}
d_\pi
}.
\end{equation}
The single pessimistic CAT policy considered in the experiment is therefore
\begin{equation}
\label{eq:cat-pessimistic-policy}
\widehat\pi_{\mathrm{pess}}
\in
\arg\max_{\pi\in\Pi}
\left\{
d_\pi^\top\widehat\gamma
-
\sqrt{
2r\,
d_\pi^\top
\widehat S_\gamma^{-1}
d_\pi
}
\right\}.
\end{equation}
Thus, CAT-Greedy maximizes the transported point estimate, whereas CAT-Pess maximizes its ellipsoidal lower confidence value.

\paragraph{Interpretation and calibration of the radius.} The matrix $\widehat\Omega$ in \eqref{eq:compact-minimax-loss} is the empirical Gram matrix induced by the quadratic penalty on the test functions. It is not, in general, the sampling covariance matrix of $\widehat\theta$. Therefore, the loss ellipsoid $\mathcal C_\theta(r)$ should not automatically be described as a nominal $(1-\delta)$ Wald confidence region.

When $r$ is treated as a sensitivity parameter, we report the results as a radius sensitivity analysis. If nominal coverage is desired, $r$ can instead be calibrated by independently resampling the two source datasets. For bootstrap replicate $b$, let $\widehat\theta^{(b)}$ denote the estimator recomputed from the resampled datasets and define
\begin{equation}
T^{(b)}
=\frac{1}{2}
\left(
\widehat\theta^{(b)}-\widehat\theta
\right)^\top
\widehat Q
\left(
\widehat\theta^{(b)}-\widehat\theta
\right).
\end{equation}
Taking $r_{1-\delta}$ to be the empirical $(1-\delta)$ quantile of $\{T^{(b)}\}_{b=1}^{B}$ gives the bootstrap-calibrated loss ellipsoid
\begin{equation}
\mathcal C_\theta(r_{1-\delta}).
\end{equation}
This calibration preserves the geometry of the minimax loss while accounting for uncertainty from both source samples. This appendix specifies the discrete structural causal model used in Section~\ref{sec:experiments} in full, so that all reported numbers are reproducible.

\subsection{Evaluation Protocol}
\paragraph{Variables.}
Context $X \in \{1, \dots, K\}$ with $K = 6$; action $A \in \{0, \dots, M-1\}$ with $M = 5$; unobserved confounder $U \sim \mathcal{N}(0,1)$; negative-outcome
instrument $Z_W \sim \mathcal{N}(0,1)$; action instruments $Z_A = (Z_{A,0}, \dots, Z_{A,M-1})$ with $Z_{A,a} \sim \mathcal{N}(0,1)$
independent; and structural noises $\varepsilon_W, \varepsilon_{A,a}, \varepsilon_Y \sim \mathcal{N}(0,1)$, all mutually independent.

\paragraph{Structural equations.}
For domain $d \in \{\star, 1, 2\}$ (target, source~1, source~2),
\begin{align*}
W &= \omega^{(d)}_X + \kappa^{(d)}_W Z_W + \rho_W U + \sigma_W \varepsilon_W, \\
A &= \arg\max_{a \in \{0,\dots,M-1\}} \bigl( \alpha^{(d)}_{a, X} + \kappa_A Z_{A,a} + \rho_a U + \sigma_A \varepsilon_{A,a} \bigr), \\
Y &= f^{(d)}(W, A, X) + \rho_Y U + \sigma_Y \varepsilon_Y, \\
f^{(d)}(w, a, x) &= \mu^{(d)}_x + b^{(d)}_x w + \tau^{(d)}_{x,a} + \lambda^{(d)}_a w, \qquad \tau^{(d)}_{x,0} = 0,\ \lambda^{(d)}_0 = 0.
\end{align*}
The misspecification experiment adds the term $\lambda_{\mathrm{nl}} (w - c_x)^2$
to $f^{\star}$ and $f^{(1)}$, which the cell-linear-in-$w$ bridge cannot
represent; $\lambda_{\mathrm{nl}} = 0$ elsewhere.

\paragraph{Domain discrepancies.}
Source~1 shares the reward mechanism with the target ($f^{(1)} = f^\star$, $\Delta^{(1)} = \{A, W\}$) but shifts the $W$-mechanism and the behaviour policy.
Source~2 shares the $W$-mechanism with the target
($\omega^{(2)} = \omega^\star$, $\kappa^{(2)}_W = \kappa^\star_W$,
$\Delta^{(2)} = \{A, Y\}$) but shifts the reward mechanism and the behaviour
policy. The two behaviour utility tables $\alpha^{(1)}, \alpha^{(2)}$ are drawn
independently.

\paragraph{Coefficients.}
Shared noise and confounding scales are $\rho_W = 0.7$, $\sigma_W = 0.5$,
$\kappa_A = 2.5$, $\sigma_A = 1.0$, $\rho_Y = 1.0$, $\sigma_Y = 0.3$, and
$$\rho_a = (0.896, 1.441, 0.701, 1.488, 1.258).$$ The negative-outcome loadings
are $\kappa^\star_W = \kappa^{(2)}_W = 1.0$ and $\kappa^{(1)}_W = 0.7$. The
target reward coefficients are
$$
\begin{aligned}
    \mu^\star &= (0.012, 0.450, -0.356, 0.449, -0.188, -0.077),\\
b^\star &= (0.714, 0.505, 0.575, 0.314, 0.677, 0.569),\\
\lambda^\star &= (0, -0.144, 0.204, 0.006, 0.007),
\end{aligned}
$$ and $\tau^\star$ is the
$6 \times 5$ table
\[
\tau^\star =
\begin{pmatrix}
0 & 0.231 & -0.157 & -0.037 & -0.293 \\
0 & -0.237 & -0.190 & 0.200 & -0.176 \\
0 & 0.385 & 0.369 & 0.180 & 0.033 \\
0 & -0.271 & 0.376 & 0.013 & -0.307 \\
0 & 0.221 & 0.090 & 0.334 & -0.368 \\
0 & -0.033 & -0.350 & 0.113 & 0.282
\end{pmatrix}.
\]
The target $W$-means are
$$\omega^\star = (-0.475, -0.158, -0.788, 0.266, -0.239, 0.451)$$ and the source-1
$W$-means are $$\omega^{(1)} = (-0.229, -0.268, -0.200, 0.477, 0.257, 0.198).$$
The source-2 reward table is obtained from the target table by an independent perturbation of each coefficient. All coefficients are generated once from a fixed seed and held constant across replications; only the data draws vary with the replication seed.

\paragraph{Controlled reward regimes.}
For the pessimism phase diagram the target advantage on the hard contexts is set
to a controlled value by overriding $\tau^\star_{x,a}$ so that
$\Delta^\star(x, a) = \tau^\star_{x,a} + \lambda^\star_a \omega^\star_x$ equals a
prescribed target. In the harmful regime $\Delta^\star(x, a)$ is drawn in
$[-0.40, -0.15]$ for every non-reference action on the hard contexts, so the
reference action is optimal there. In the beneficial regime one non-reference
action has $\Delta^\star(x, a)$ in $[0.10, 0.30]$ and the rest in
$[-0.40, -0.15]$.

\paragraph{Coverage hole.}
Partial coverage is induced by subtracting a constant $\delta$ from the
behaviour utility of every non-reference action on the two hard contexts
$\{4, 5\}$ in source~1. The reported horizontal axis is the realized minimum
action propensity $p_{\min} = \min_{x,a} p^{(1)}(A = a \mid X = x)$, computed by
Monte-Carlo over the probit noise and $U$.

\paragraph{Features and estimator.}
The bridge features are cell-saturated: $\phi_h(w, a, x)$ carries $(1, w)$ in the
block for cell $(x, a)$ and zeros elsewhere, so $d_h = 2KM$. The CATE features
$\phi_g(a, x)$ are the $(x, a)$ cell indicators, so $d_g = KM$. The source-1 test
functions condition on the instruments $(X, Z_A, Z_W)$: per context, the terms
$1$, $Z_W$, and for each action $Z_{A,a}$ and $Z_{A,a} Z_W$; the source-2
transport test functions are the context indicators. The test-covariance is
ridged with $10^{-3}$ and the Hessian with $10^{-8}$ for numerical stability.

\paragraph{Evaluation.}
Ground-truth quantities are computed analytically: $$g^\star(a, x) = \mu^\star_x +
b^\star_x \omega^\star_x + \tau^\star_{x,a} + \lambda^\star_a \omega^\star_x,$$ 
with the additional term $$\lambda_{\mathrm{nl}}((\omega^\star_x - c_x)^2 +
\mathrm{Var}(W)),\quad\text{where}\ \mathrm{Var}(W) = (\kappa^\star_W)^2 + \rho_W^2 + \sigma_W^2$$ under misspecification.
The optimal policy is $$\pi^\star(x) = \arg\max_a
g^\star(a, x),$$ the value is $v^\pi = \sum_x p^\star(x)\, g^\star(\pi(x), x)$, and suboptimality is $v^{\pi^\star} - v^{\hat\pi}$. The CATE RMSE averages $$\left(\hat g(a,x) - g^\star(a,x)\right)^2$$ over all cells with $X$ drawn from the target law. The bound quantities are computed as follows: $\sigma_{\min}$ is the
smallest singular value of the whitened source-1 moment operator at $n = 10^5$; $\sum_k \|b_k\|$ is the target-policy reweighting norm over the two source distributions; and $\inf_h \|Th\|$ is $\sqrt{2\,\widehat L(\hat\theta)}$ at large $n$. Slopes are fit by ordinary least squares on log-log axes with a nonparametric bootstrap over seeds.

\paragraph{Confidence-set coverage.}
Empirical coverage of $\gamma^\star$ under the sampling covariance
$(\sigma^2 / n)\,[Q^{-1}]_{\gamma\gamma}$ is reported in Table~\ref{tab:coverage}.
Per-coordinate coverage tracks the nominal level, while the simultaneous region
over all CATE coordinates is over-covered, which reflects the geometry of
simultaneous confidence sets rather than a miscalibration.

\begin{table}[h]
  \centering
  \begin{tabular}{ccc}
    \toprule
    nominal & per-coordinate & simultaneous \\
    \midrule
    $0.80$ & $0.824$ & $0.958$ \\
    $0.90$ & $0.918$ & $0.992$ \\
    $0.95$ & $0.964$ & $0.998$ \\
    \bottomrule
  \end{tabular}
  \caption{Empirical coverage of the CATE confidence set, benign regime,
    $n_1 = 2000$, $n_2 = 8000$, over $400$ seeds.}
  \label{tab:coverage}
\end{table}

\paragraph{Sample sizes and sweeps.}
The identification experiment uses $n_1 = n_2 = n$ over
$n \in \{250, 500, 1000, 2000, 4000, 8000, 16000\}$ with $60$ seeds. The sensitivity experiments use $n_1 = 400, n_2 = 4000$ (distribution shift), $n_1 = 2000, n_2 = 8000$ (instrument strength), and $n_1 = 4000, n_2 = 8000$
(misspecification), each with $80$ seeds. The phase diagram and calibration experiments use $n_1 = 600, n_2 = 12000$ in the harmful regime and $n_1 = 2000, n_2 = 8000$ in the benign regime, with $100$ to $120$ seeds and nominal level $1 - \alpha = 0.9$. The second-source sample is kept large so that the target-compatible $W$-law is close to known and the first source isolates the instrumental and coverage effects.


\section{Discrete Confounded MDP Model and Evaluation Protocol}
\label{app:cmdp}

This appendix specifies the discrete confounded MDP used in
Section~\ref{sec:exp-cmdp}. It is the sequential analogue of the CCB-Neg model of Appendix~\ref{app:scm} and keeps the same closed-form, auditable structure.

\paragraph{Variables.}
State $S_h \in \{0, \dots, |\mathcal{S}|-1\}$ with $|\mathcal{S}| = 4$; action
$A_h \in \{0, \dots, |\mathcal{A}|-1\}$ with $|\mathcal{A}| = 4$; horizon
$H = 5$. At each stage $h$ there is a fresh unobserved confounder
$U_h \sim \mathcal{N}(0,1)$, a negative-outcome instrument
$Z_W \sim \mathcal{N}(0,1)$, action instruments
$Z_h = (Z_{h,0}, \dots, Z_{h,|\mathcal{A}|-1})$ with
$Z_{h,a} \sim \mathcal{N}(0,1)$ independent, and mutually independent structural
noises $\varepsilon_W, \varepsilon_{A,a}, \varepsilon_R \sim \mathcal{N}(0,1)$.

\paragraph{Structural equations.}
For domain $d \in \{\star, 1, 2\}$ and each stage $h$,
\begin{align*}
W_h &= \omega^{(d)}_{S_h} + \kappa^{(d)}_W Z_W + \rho_W U_h + \sigma_W \varepsilon_W, \\
A_h &= \arg\max_{a} \bigl( \alpha^{(d)}_{S_h, a} + \kappa_A Z_{h,a} + \rho_a U_h + \sigma_A \varepsilon_{A,a} \bigr), \\
R_h &= r^{(d)}(S_h, A_h, W_h) + \rho_R U_h + \sigma_R \varepsilon_R, \\
S_{h+1} &\sim \operatorname{Categorical}\bigl( \operatorname{softmax}\{\, t^{(d)}(S_h, A_h, W_h) + \rho_T U_h \,\} \bigr),
\end{align*}
where the action does not enter $W_h$, so $W_h$ is a valid negative outcome, and
the instruments are independent of $U_h$ and of the reward and transition noise.
The reward and transition means are cell-saturated and linear in $W_h$:
\[
r^{(d)}(s, a, w) = \mu^{(d)}_s + b^{(d)}_s w + \tau^{(d)}_{s,a} + \lambda^{(d)}_a w, \qquad
t^{(d)}(s, a, w)_{s'} = \mu^{(d)}_{t;s,s'} + \tau^{(d)}_{t;s,a,s'} + b^{(d)}_{t;s} w,
\]
with $\tau^{(d)}_{s,0} = 0$ and $\lambda^{(d)}_0 = 0$.

\paragraph{Domain discrepancies.}
Source~$1$ shares the reward and transition mechanisms with the target
($r^{(1)} = r^\star$, $t^{(1)} = t^\star$) but shifts the $W$-mechanism
($\omega^{(1)} \ne \omega^\star$, $\kappa^{(1)}_W \ne \kappa^\star_W$) and the
behaviour utilities $\alpha^{(1)}$, so $\Delta^{(1)} = \{A, W\}$. Source~$2$
shares the $W$-mechanism ($\omega^{(2)} = \omega^\star$,
$\kappa^{(2)}_W = \kappa^\star_W$) but shifts the reward and transition
mechanisms and the behaviour utilities $\alpha^{(2)}$, so
$\Delta^{(2)} = \{A, S, R\}$.

\paragraph{Coefficients.}
The shared scales are $\rho_W = 0.7$, $\sigma_W = 0.5$, $\kappa_A = 2.5$,
$\sigma_A = 1.0$, $\sigma_R = 0.3$, with negative-outcome loadings
$\kappa^\star_W = \kappa^{(2)}_W = 1.0$ and $\kappa^{(1)}_W = 0.7$. The
confounding loadings used in the experiments are $\rho_R = 1.8$ into the reward
and $\rho_T = 2.2$ into the transition logits, which is strong enough that the
naive and transport-no-IV baselines carry a clearly visible bias floor. All
reward intercepts, $W$-slopes, action effects, and transition logits, together
with the source-2 perturbations and the two behaviour-utility tables, are drawn
once from a fixed seed and held constant across replications; only the data draws
vary with the replication seed. The exact numerical values are produced by
\texttt{default\_cmdp\_params} in the released code.

\paragraph{Target quantities and evaluation.}
The target interventional reward mean is
$\mathbb{E}^\star[R \mid s, \operatorname{do}(a)] = \mu^\star_s + b^\star_s
\omega^\star_s + \tau^\star_{s,a} + \lambda^\star_a \omega^\star_s$, using
$\mathbb{E}^{(2)}[W \mid S = s] = \omega^\star_s$. The target transition kernel
is obtained by averaging the softmax of the $W$-dependent logits over the target
$W$-law and the confounder by Monte-Carlo integration at high precision, which is
treated as ground truth. The optimal policy and its value are computed by
backward induction on these exact target means, and suboptimality is
$v^{\pi^\star} - v^{\hat\pi}$ with $v^{\hat\pi}$ the exact value of the estimated
policy under the true target dynamics.

\paragraph{Estimator.}
The stage-wise reward and transition bridges use the cell-saturated bases of the
bandit model. The source-1 instrument test functions condition on
$(S, Z_h, Z_W)$ rather than on the realized action, which is what removes the
action-selection confounding. Because the mechanisms are stationary across
stages, the moment equations are pooled over the $H$ stages when forming the
bridge, which reduces the per-stage variance by a factor of order $H$; this
pooled estimator is \textsc{CAT-CMDP}. The transport step averages the source-1
bridge over the source-2 $W$-law to obtain the target reward means $\hat R_h$ and
transition kernels $\hat P_h$, which are then combined by backward induction.

\paragraph{Coverage hole.}
Partial coverage is induced by subtracting a constant $\delta$ from the behaviour
utility of every non-reference action on two states $\{2, 3\}$ of source~$1$. The
reported horizontal axis is the realized minimum action propensity
$p_{\min} = \min_{s,a} p^{(1)}(A = a \mid S = s)$, computed by Monte-Carlo over the
probit noise and the confounder.

\paragraph{Sample sizes and sweeps.}
The identification experiment uses $n_1 = n_2 = n$ trajectories per source over
$n \in \{1000, 2000, 4000, 8000, 16000, 32000\}$ with $40$ seeds. The horizon
experiment fixes $n = 8000$ per source and sweeps $H \in \{2, 4, 6, 8, 10\}$ with
$40$ seeds. The pessimism experiment uses $n_1 = 2000$, $n_2 = 16000$ with the
coverage hole $\delta \in \{0, 3, 6, 9\}$ over $80$ seeds, at nominal level
$1 - \alpha = 0.9$. The second source is kept large so that the target-compatible
$W$-law is close to known and the first source isolates the instrumental and
coverage effects.